%% file: main.tex
\documentclass{article}
\usepackage{iclr2026_conference,times}

\usepackage[hyphens]{url}
\usepackage{graphicx}
\usepackage{caption}

\input{packages}

\title{A Stochastic--Geometric Theory of Scaling Laws in Grokking}

\preprintcopy

\author{
R\'ois\'in Luo\thanks{Correspondence to: \texttt{roisincrtai@gmail.com}} \\
Research Ireland -- Centre for Research Training in AI\\
University of Galway \\
\And
Christian Gagn\'e \\
Universit\'e Laval \\
Canada--CIFAR AI Chair \\
Mila -- Qu\'ebec AI Institute
\And
Jonas Ngnaw\'e \\
Universit\'e Laval \\
Mila -- Qu\'ebec AI Institute
\And
Ihsan Ullah \\
Visual Intelligence Lab\\
University of Galway
\And
Karyn Morrissey \\
University of Galway
}

\begin{document}

\maketitle

\begin{abstract}
Delayed generalization (\ie~grokking) refers to the phenomenon in which a neural network fits its training data early in training but only begins to generalize after a prolonged delay, often through an abrupt transition. Despite extensive empirical study, its underlying mechanism remains poorly understood. In this work, we first theoretically characterize a shell--core topological configuration of the reachable solution space induced by Adam's optimization dynamics with weight-shrinkage regularization, supported by empirical evidence. This optimization-induced topological configuration gives rise to grokking. In model's parameter space, random initialization solutions concentrate on a thin outer spherical shell, enclosing another spherical shell of memorization solutions, which in turn contains a core corresponding to the generalization solutions. Leveraging stopping-time theory, we then analyze the geometry of this topological configuration and the solution transition time at which optimization trajectories escape the memorization manifold and first reach the boundary of the generalization manifold. Our theoretical analysis derives grokking scaling laws for the learning rate, batch size, and $\ell_2$ regularization coefficient, which are further validated through experiments and shown to recover results from prior literature.


\end{abstract}

\input{text}

\ificlrfinal

\section*{Acknowledgments}

This research was financially supported in part by \textbf{Taighde \'Eireann} -- Research Ireland under Grant No. 18/CRT/6223, and by the J.E. Cairnes School of Business \& Economics, University of Galway, Ireland. It was also supported in part by computational resources and services provided by \textbf{Calcul Qu\'ebec} (\url{calculquebec.ca}) and the \textbf{Digital Research Alliance of Canada} (\url{alliancecan.ca}). Partial computational support was also provided by \textbf{Taighde \'Eireann} -- Research Ireland under Grant No. SFI/12/RC/2289\_P2 and the Insight Research Ireland Centre for Data Analytics. The authors gratefully acknowledge Prof. Karyn Morrissey of the J.E. Cairnes School of Business \& Economics, University of Galway, Ireland, for the support. For the purpose of open access, the author has applied a CC BY public copyright licence to any Author Accepted Manuscript version arising from this submission.
\fi

\bibliographystyle{iclr2026_conference}
\bibliography{references}

\appendix

\input{appendix}

\end{document}

%% file: packages.tex
\usepackage{titletoc}

\usepackage{mathrsfs}
\usepackage{physics}
\usepackage{subcaption}
\usepackage{amsmath}
\usepackage{amsthm}
\usepackage{amssymb}
\usepackage{booktabs}

\usepackage{amssymb}
\usepackage{textcomp}

\usepackage{xcolor}

\usepackage[
    colorlinks=true,
    linkcolor=blue,
    citecolor=blue,
    urlcolor=blue
]{hyperref}

\newcommand{\ie}{\emph{i.e.}}
\newcommand{\eg}{\emph{e.g.}}
\newcommand{\etc}{etc}

\theoremstyle{plain}

\newtheorem{theorem}{Theorem}
\newtheorem{lemma}{Lemma}

\theoremstyle{definition}
\newtheorem{observation}{Observation}
\theoremstyle{remark}
\newtheorem{remark}{Remark}

\usepackage{tikz}
\usetikzlibrary{positioning,arrows.meta,calc}
\usetikzlibrary{arrows.meta, decorations.pathmorphing, positioning, calc}
\usetikzlibrary{arrows.meta,calc,decorations.pathreplacing}
\usetikzlibrary{tikzmark,calc,arrows.meta}

%% file: text.tex
\section{Introduction}












Neural networks trained on noise-free, highly structured learning tasks have been observed to exhibit an \emph{epiphany} phenomenon known as delayed generalization, or grokking~\citep{power2022grokking}. In these settings, models rapidly achieve near-zero training loss, often early in training, yet fail to generalize for an extended period before abruptly transitioning to strong test performance. This behavior has been reported across a range of tasks with exact underlying structure \citep{HwangPark2026}, such as modular arithmetic \citep{power2022grokking,zimingliuNIPS2022grokking,zhong2023clock}, algorithmic reasoning \citep{nanda2023progress}, and group-theoretic learning \citep{stander2024grokking,NotsawoDumasRabusseau2026}. Moreover, prior work has shown that the onset of grokking is highly sensitive to hyperparameter choices \citep{zhong2023clock}, including the fraction of training data, learning rate, batch size, and regularization coefficient~\citep{power2022grokking,zimingliuNIPS2022grokking}. Despite extensive empirical investigation, however, its underlying mechanism remains poorly understood.

\begin{figure*}[t!]
    \centering

    \begin{subfigure}[t]{0.32\linewidth}
        \centering
        \includegraphics[width=\linewidth]{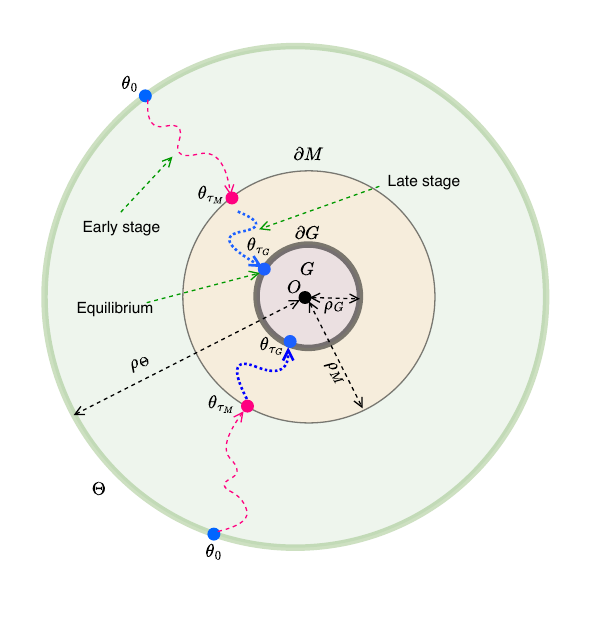}
        \caption{Weight-Shrinkage Grokking}
        \label{fig:topological_prior_illustration}
    \end{subfigure}
    \begin{subfigure}[t]{0.32\linewidth}
        \centering
        \includegraphics[width=\linewidth]{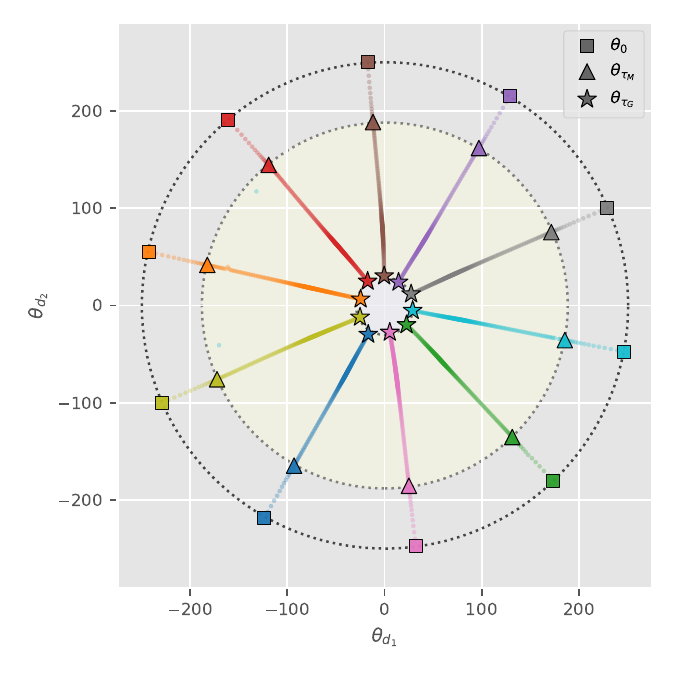}
        \caption{Trajectories w/ $S_5$}
        \label{fig:trajectories_evidence_s5}
    \end{subfigure}
    \begin{subfigure}[t]{0.32\linewidth}
        \centering
        \includegraphics[width=\linewidth]{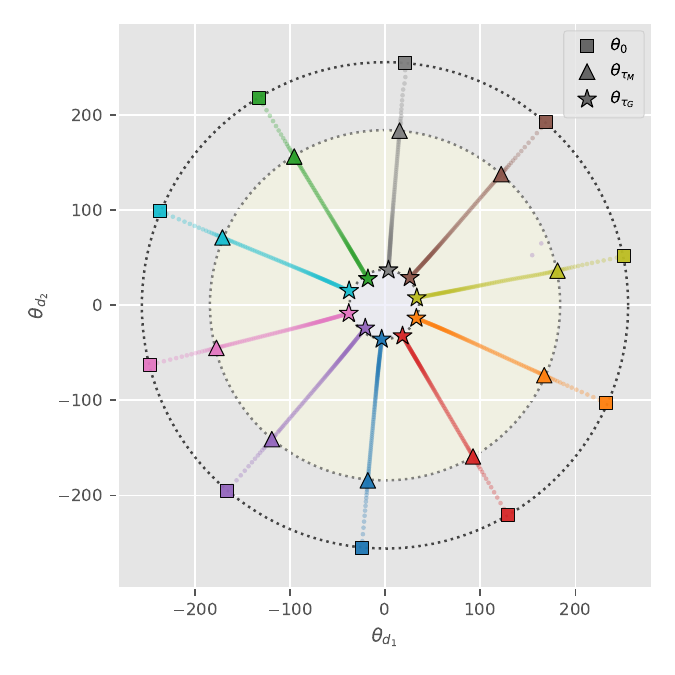}
        \caption{Trajectories w/ $\mathbb{Z}_{127}$}
        \label{fig:trajectories_evidence_z127}
    \end{subfigure}

   \caption{\textbf{Optimization-Induced Shell--Core Topology and Empirical Evidence.} Left (\subref{fig:topological_prior_illustration}): grokking dynamics under weight-shrinkage regularization induce a shell--core topology in parameter space, where initialization points $\theta_0$ concentrate on a hyperspherical shell $\Theta$ with radius $\rho_\Theta = O(k)$ and thickness $O(k/\Tilde{p})$, where $k$ denotes the number of parameter matrices and $\Tilde{p}$ denotes effective parameter dimension~(see Section~\ref{sec:initialization_concentration}), enclosing the memorization shell $M \setminus G$ and the generalization core $G$; the radii $\rho_G$, $\rho_M$, and $\rho_\Theta$ are determined by the task and optimization configuration. Middle--right (\subref{fig:trajectories_evidence_s5}--\subref{fig:trajectories_evidence_z127}): empirical evidence on $S_5$ and $\mathbb Z_{127}$ using isometric manifold learning (\eg, MDS) to visualize ten optimization parameter trajectories, where $\theta_0$, $\theta_{\tau_M}$, and $\theta_{\tau_G}$ denote initialization, first memorization, and first generalization states. The projections show that Adam-reachable $M$ and $G$ are approximately symmetric embedded in ambient space $\mathbb R^p$.} 

    \label{fig:shell_core_prior_and_evidence}

\end{figure*}


Empirical findings consistently suggest that training dynamics in the grokking regime exhibit three-stage behavior~\citep{power2022grokking,nanda2023progress,kumar2024grokking}. In the initial stage, the model rapidly interpolates the training set and reaches the memorization regime. In the second stage, the trajectory remains near the memorization solutions for an extended period before escaping and eventually reaching the generalization regime. The length of this extended period scales with the task, learning rate, batch size, and regularization coefficient. In the third stage, drift and diffusion equilibrate in the optimization-induced dynamics, and the trajectory settles within the generalization solution region centered at its minimizer. 

Under the regime with $\ell_2$ weight-shrinkage regularization, these observations imply the existence of two distinct classes of solutions: \emph{memorization solutions}, which only interpolate the training data and does not generalize, and \emph{generalization solutions}, which additionally achieve low loss with respect to the underlying data distribution. Their topological configuration is thus induced by optimization dynamics, and characterized into three stages: the \emph{early stage}, the \emph{late stage}, and the \emph{equilibrium}, as shown in Figure~\ref{fig:topological_prior_illustration}. The \emph{early-stage} dynamics are driven by deterministic drift from large gradients, which dominates the optimization trajectory and rapidly drives it toward the memorization manifold. At this stage, the gradient fluctuations are anisotropic. The \emph{late-stage} dynamics are driven more slowly by decaying but nonzero gradients, together with isotropic gradient fluctuations, which guide the optimization trajectory across the trivial memorization solutions. Once the trajectory reaches the generalization set, the drift and diffusion in the optimization dynamics equilibrate and confine the trajectory within the generalization manifold, referred to as the \emph{equilibrium}. Figure~\ref{fig:trajectories_evidence_s5}--\ref{fig:trajectories_evidence_z127} provide empirical evidence consistent with this optimization-induced topological configuration by visualizing optimization trajectories via isometric manifold learning. This optimizer-induced topological configuration, as illustrated in Figure~\ref{fig:topological_prior_illustration}, gives rise to grokking and forms a shell–core configuration: initialization solutions concentrate on a thin hyperspherical shell, which encloses a shell of memorization solutions containing a compact core of generalization solutions. Detailed empirical results of Adam--induced shell--core radii are provided in Appendix~\ref{app:adam_induced_radius_concentration}.

We theoretically analyze the geometry of this topological configuration and derive scaling laws for the manifold radii and the solution transition time from memorization to generalization through a \emph{stopping-time analysis} \citep{dynkin1965markov} of stochastic differential equations (SDEs) \citep{oksendal2003sde}. The remainder of the paper is organized as follows. Section~\ref{sec:initialization_concentration} presents initialization concentration, and Section~\ref{sec:joint_state_sde} establishes the joint-state SDE for Adam \citep{kingma2015adam} designed to analyze grokking dynamics. Building on this formulation, Section~\ref{sec:law_memorization} derives the scaling law for the memorization-manifold radius. Under the late-stage reduction observations of grokking, Section~\ref{sec:late_stage_radius_dynamics} derives a closed-form preconditioned radius SDE for the evolution of manifold radii. Sections~\ref{sec:law_generalization}--\ref{sec:law_transition_time} derive the scaling laws for the generalization-manifold radius and the solution transition time from memorization to generalization. Finally, Section~\ref{sec:validation} validates the theory through experiments on modular arithmetic and group-theoretic learning tasks, and by comparison with the grokking literature. The contributions are summarized below:
\begin{enumerate}

    \item \textbf{Optimization-Induced Shell--Core Topology.} We introduce an optimizer-induced shell--core prior for grokking, where initialization concentrates on a thin hyperspherical shell enclosing memorization solutions, which in turn contain a compact core of generalization solutions.

    \item \textbf{Stochastic--Geometric Characterization.} Leveraging the topological prior, we formulate grokking as a stochastic transition problem and characterize the geometry by the first stopping time at which stochastic gradient flows escape the memorization manifold and reach the generalization manifold.

    \item \textbf{Empirical, Literature, and Symbolic-Algebra Validation.} We validate our theory through experiments on group-theoretic learning and modular arithmetic tasks, and by comparison with the literature, recovering known scaling laws for the learning rate, batch size, and regularization coefficient. We additionally machine-check all closed-form and asymptotic results appearing in the proofs by a symbolic algebra system (SAS) with Julia.
    
\end{enumerate}

\section{Notations}
\label{sec:notations}

\paragraph{Data and Network.} Let $\{(x_i,y_i)\}_{i=1}^n$ denote a finite training set of size $n$, where each pair $(x_i,y_i)\in\mathcal X\times\mathcal Y$ is drawn i.i.d. from the empirical distribution $\widehat P_{\mathcal{XY}}$, with underlying true distribution $P_{\mathcal{XY}}$. Let $f_\theta:\mathcal X\to\mathcal Y$ be a model parameterized by $\theta\in\mathbb R^p$. Throughout, we use $X_t$ to denote the time-indexed value of a variable $X$ at time $t$; for example, $\theta_t$ denotes the parameter at time $t$. 

\paragraph{Instance, Batch, Empirical, and True Loss.} Let $\ell_f(x,y;\theta)$ denote the instance loss evaluated at a point $(x,y)$. Let $\mathscr{L}_{f}(\xi;\theta)$ denote the \emph{batch loss} evaluated on a mini-batch $\xi=\{(x_{\xi_i},y_{\xi_i})\}_{i=1}^b$ of size $b$. Let $\widehat{\mathscr{L}}_f(\theta):=\mathbb{E}_{\xi\sim \widehat P_{\mathcal{XY}}}\!\left[\mathscr{L}_{f}(\xi;\theta)\right]$ denote the \emph{empirical loss} evaluated over the training set. Let $\mathscr{L}_{f}(\theta)$ denote the \emph{true loss} with respect to the underlying data distribution $P_{\mathcal{XY}}$.


\paragraph{Initialization, Memorization, and Generalization Manifolds.} We study the solution manifolds reachable under the dynamics induced by Adam~\citep{kingma2015adam}. The ambient parameter space is $\mathbb R^p$ endowed with the Euclidean metric. Let $\Theta$ denote the set of \emph{initialization solutions} under some initialization distribution $P_\Theta$. Let $M:=\{\theta\in\mathbb R^p \mid \widehat{\mathscr L}_f(\theta)<\epsilon \;\land\; p_{\mathrm{Adam}}(\theta\mid \theta_0\in\Theta)>0\}$ denote the set of \emph{memorization solutions} under Adam's dynamics with a loss tolerance $\epsilon$. Similarly, let $G:=\{\theta\in\mathbb R^p \mid \mathscr L_f(\theta)<\epsilon \;\land\; p_{\mathrm{Adam}}(\theta\mid \theta_0\in\Theta)>0\}$ denote the set of \emph{generalization solutions}. Throughout, $\partial S$ and $S^\circ$ denote the boundary and interior of a set $S$, respectively; $\partial M$ is the outer boundary of $M\setminus G$, and $\partial G$ is its inner boundary. Let
$\rho_M^2:=\mathbb{E}_{\theta_0 \in \Theta,\theta\in\partial M}\!\left[\|\theta\|_2^2\mid \theta_0\right]$
and
$\rho_G^2:=\mathbb{E}_{\theta_0\in\Theta,\theta\in\partial G}\!\left[\|\theta\|_2^2\mid \theta_0\right]$
denote the mean squared radii of the memorization and generalization manifolds, respectively.

\section{Initialization Concentration}
\label{sec:initialization_concentration}


We first demonstrate the topological concentration of network-parameter initialization under normal and uniform distributions. Consider that $\theta_0=(\theta_0^{(1)},\cdots,\theta_0^{(k)}) \in \mathbb{R}^p$ consists of $k$ sub-vectors such that $\theta_0^{(j)} \in \mathbb{R}^{p_j}$ where $\sum_{j=1}^{k} p_j=p$. Define effective dimension $\Tilde{p}$ such that $\frac{1}{\Tilde{p}}:=\frac{1}{k} \sum_{j=1}^k \frac{1}{p_j}$. For coordinate-wise initialization $[\theta_0^{(j)}]_i \sim \mathcal{N}(0,\sigma_j^2)$, the initialization radius square
\begin{align}
\rho_{\Theta}^2(\theta_0) = \|\theta_0\|_2^2 
=
\sum_{j=1}^{k} 
\Biggl(\sum_{i=1}^{p_j} \; \Bigl|[\theta_0^{(j)}]_i\Bigr|^2\Biggr)
\sim \sum_{j=1}^{k}\sigma_j^2\chi_{p_j}^2
,
\end{align}
follows a sum of scaled chi-square distributions, where each $\sigma_j^2\chi_{p_j}^2$ has $p_j$ degrees of freedom. Hence, the radius of $\Theta$ admits the closed-form expression
\begin{align}
\tilde{\sigma}^2
:=
\frac{1}{p}
\sum_{j=1}^k p_j\sigma_j^2,
\quad
\mathbb{E}_{\theta_0}\bigl[ \rho_{\Theta}^2(\theta_0) \bigr]
\;=\;  \Tilde{\sigma}^2p = O(k),
\quad
\mathrm{Var}_{\theta_0}\bigl[ \rho_{\Theta}^2(\theta_0) \bigr] =2\sum_{j=1}^k p_j\sigma_j^4 = O(\frac{k}{\Tilde{p}})
\notag
,
\end{align}
under a coordinate scaling $\sigma_j=O(1/\sqrt{p_j})$~\citep{lecun1998efficient,glorot2010understanding}. This suggests that initialization solutions concentrate on a thin annulus of radius $O(k)$ with thickness $O(\sqrt{k/\Tilde{p}})$ in parameter space. Similar concentration also holds for coordinate-wise uniform initialization distribution $[\theta_0^{(j)}]_i\sim \mathcal{U}(-\varepsilon_j,\varepsilon_j)$, 
\begin{align}
\tilde{\varepsilon}^2
:=
\frac{1}{p}
\sum_{j=1}^k p_j\varepsilon_j^2,
\quad
\mathbb{E}_{\theta_0}\bigl[ \rho_{\Theta}^2(\theta_0) \bigr]
\;=\;  \frac{p\Tilde{\varepsilon}^2}{3} = O(k),
\quad
\mathrm{Var}_{\theta_0}\bigl[ \rho_{\Theta}^2(\theta_0) \bigr] = \frac{4}{45}\sum_{j=1}^k p_j\varepsilon_j^4 = O(\frac{k}{\Tilde{p}})
\notag
,    
\end{align}
under a coordinate scaling $\varepsilon_j=O(1/\sqrt{p_j})$. Proofs for the normal and uniform concentration results are provided in Appendices~\ref{app:initialization_concentration_normal_appendix} and~\ref{app:initialization_concentration_uniform_appendix}, respectively. 

\section{Manifold Radius and Solution Transition Analysis}
\label{sec:radius_transition_analysis}

We characterize Adam's dynamics through its continuous-time limit, formulated as an SDE on the joint state of the parameters and first- and second-moment estimates. Based on this joint-state SDE, we derive the scaling law for the \emph{memorization radius} $\rho_M$, defined as the parameter norm at the first-hitting time of the empirical-loss boundary $\partial M$, via perturbation method \citep{benderorszag1999}. Further leveraging two reduction properties of the grokking dynamics, we reduce the intractable joint-state SDE for Adam (Section~\ref{sec:joint_state_sde}) to a tractable radial SDE. This reduction enables an analytical characterization of the \emph{generalization radius} $\rho_G$, defined as the parameter norm at the first-hitting time of the generalization boundary $\partial G$, and the \emph{solution transition time} $\mathbb{E}[\tau_{M\to G}]$, defined as the mean time required to transition from memorization boundary $\partial M$ to generalization boundary $\partial G$. A high-level sketch for theoretical analysis framework is illustrated as in Figure~\ref{fig:high_level_sketch_flowchart} of the appendix.

\subsection{Adam's Joint-State Continuous-Time SDE Limit}
\label{sec:joint_state_sde}

We consider Adam dynamics in the small learning rate and large batch size regime $\eta \to 0,\, b \to \infty$, with exponential moving average (EMA) coefficients $(\beta_1,\beta_2)\in(0,1)^2$ and an $\ell_2$ regularizer $R(\theta)=\tfrac{1}{2}\|\theta\|_2^2$ with coefficient $\lambda>0$. Let $g_t := \nabla[\mathscr{L}^*_f(\xi_t;\theta_t)+\lambda R(\theta_t)]$ denote the regularized mini-sample gradient flow, with per-sample mean $\bar g_t:= \mathbb{E}_{s}[\nabla\ell_f(s;\theta_t)]$ and covariance $\Sigma_t:=\mathrm{Cov}_s[\nabla\ell_f(s;\theta_t)]$. For sufficiently large batch size $b$, the mini-batch gradient admits the distributional limit $g_t \sim \mathcal{N}\bigl(\bar g_t, \Sigma_t/b\bigr)$. We refer to this setting as \emph{coupled weight decay}, since the $\ell_2$ regularizer contributes to the gradient used to update Adam's first- and second-moment estimates. In contrast, the \emph{decoupled weight decay} scheme applies weight decay as a separate parameter-shrinkage step and is commonly referred to as \emph{AdamW} \citep{loshchilov2019decoupled}. 

\paragraph{Adam's Parameter Update Rules.} Let $\theta \in \mathbb{R}^p$ be model's parameters. Let $m,v \in \mathbb{R}^p$ be the first- and second-moment estimates, respectively. Let $\hat{m},\hat{v}$ be the $m,v$ with bias corrections. Adam's discrete update rules at iteration $k$ are given as
\begin{gather}
m_{k+1} = \beta_1\,m_k + (1-\beta_1)\,g_k, 
\qquad
v_{k+1} = \beta_2\,v_k + (1-\beta_2)\,(g_k \odot g_k), \notag\\
\hat m_{k+1} = m_{k+1}\,/\,(1-\beta_1^{k+1}), 
\qquad
\hat v_{k+1} = v_{k+1}\,/\,(1-\beta_2^{k+1}), \notag\\
\theta_{k+1} = \theta_k - \eta\, \hat m_{k+1} \oslash \Bigl(\sqrt{\hat v_{k+1}} + \varepsilon \,\mathbf{1}_p \Bigr) \notag
, 
\end{gather}
with initial values $\theta_0 \in \Theta, m_0 = 0$, $v_0 = 0$, where $\odot$ is element-wise product and $\oslash$ is element-wise division. $\hat{m}$ and $\hat{v}$ are Adam's bias corrections, which compensate the warm-up underestimation of $(m_k,v_k)$ caused by the zero initialization $m_0 = v_0 = 0$. We drop Adam's numerical-stability constant $\varepsilon$ throughout the analysis as it does not affect the analysis. Defining the Adam joint state as $S_t := (\theta_t,m_t,v_t) \in \mathbb{R}^{3p}$ and using the continuous-time interpolation $t=\eta k$, we state the resulting joint-state SDE for Adam in Lemma~\ref{lem:adam_joint_sde}. The proof sketch is illustrated in Figure~\ref{fig:adam_closed_form_sde_proof_flowchart}. The proof is provided in Appendix~\ref{app:adam_ct_sde_appendix}. The technical correctness is verified through an induced radius SDE with It\^o's lemma in Appendix~\ref{app:induced_radius_sde_appendix}. Lemma~\ref{lem:adam_joint_sde} provides a stochastic--theoretical framework for studying grokking dynamics induced by Adam optimizer.

\begin{lemma}[Adam's Joint-State Continuous-Time SDE Limit]
\label{lem:adam_joint_sde}
Let $S_t := (\theta_t^\top,m_t^\top,v_t^\top)^\top \in\mathbb R^{3p}$ denote the joint state of Adam, where $\theta_t$, $m_t$, and $v_t$ are the parameters, first-moment estimates, and second-moment estimates, respectively. Taking the continuous-time interpolation $t=\eta k$, the discrete Adam updates admit the following It\^o SDE limit on the joint state
\begin{align}
\dd S_t \;=\; \mu(S_t)\,\dd t \;+\; \sqrt{\frac{\eta}{b}}\,\sigma(S_t)\,\dd W_t,
\label{eq:adam_joint_sde}
\end{align}
where $\dd W_t\in\mathbb R^p$ denotes the infinitesimal increment of a Wiener process adapted to the filtration generated by the mini-batch sampling process $\{\xi_t\}$, with drift factor $\mu(S_t)\in\mathbb R^{3p}$ and diffusion factor $\sigma(S_t)\in\mathbb R^{3p\times p}$
\begin{align}
\mu(S_t) \;=\;
\begin{pmatrix}
-\,B(t)\,m_t \oslash \sqrt{v_t} \\[2pt]
-\,\alpha_1\,\bigl(m_t - \bar g_t\bigr) \\[2pt]
-\,\alpha_2\,\bigl(v_t - \overline{g\odot g}_t\bigr)
\end{pmatrix}, 
\quad
\sigma(S_t) \;=\;
\begin{pmatrix}
0 \\[2pt]
\alpha_1\, \Sigma_t^{1/2} \\[2pt]
\alpha_2\, D_t
\end{pmatrix},
\quad
B(t) \;:=\; \frac{\sqrt{\,1 - e^{-\alpha_2 t}\,}}{\,1 - e^{-\alpha_1 t}\,},
\label{eq:adam_joint_drift_diffusion}
\end{align}
where $\alpha_i := (1-\beta_i)/\eta$ and $B(t)$ is the bias-correction factor in continuous time 
and 
\begin{align}
\overline{g\odot g}_t \;:=\; \bar g_t \odot \bar g_t \;+\; \frac{1}{b}\,\mathrm{diag}(\Sigma_t) \in \mathbb{R}^p, \quad 
D_t \;:=\; 2\,\mathrm{diag}\left(\bar g_t\right)\,\Sigma_t^{1/2} \in \mathbb{R}^{p \times p}
.
\label{eq:diffusion_cross_term_def}
\end{align}
In particular, to decompose this SDE into a tractable form, we define $\pi(\theta_t)
:=
\mathrm{diag}\!\left(\overline{g\odot g}_t\right)^{-1/2} \in \mathbb{R}^{p \times p}$ as Adam's preconditioner. We write $\theta_t(\theta_0),\theta_t(S_0)$ to denote the Adam's $\theta$-evolution starting from initial states $\theta_0 \in \Theta,S_0 \in \Theta_S$, respectively, where $\Theta_S$ is the initialization distribution of $S$.

\end{lemma}





\begin{remark}
While the continuous-time SDE limit of ordinary stochastic gradient descent (SGD) is now well established \citep{li2017stochastic,mandt2017sgd,luo2025optimizationinduceddynamicslipschitzcontinuity}, a rigorous continuous-time treatment of Adam that explicitly accounts for stochasticity remains, to the best of our knowledge, comparatively less developed. For example, the joint $(m,v)$ deterministic limit of Adam appears in \citep{dasilva2020general}, and approximations for adaptive methods such as RMSprop and Adam were established in \citep{malladi2022adamsde,compagnoni2025adaptive}. Our joint-state SDE limit only partially overlaps with this line of work and yields an SDE limit that accounts for stochasticity without relying on a first-order approximation. 
\end{remark}

\subsection{Memorization Radius}
\label{sec:law_memorization}

Starting from the initial state $S_0=(\theta_0^\top,0_p^\top,0_p^\top)^\top$, Adam's joint-state SDE~\eqref{eq:adam_joint_sde} generates a stochastic trajectory $\{S_t\}$ adapted to the filtration generated by mini-batch sampling. The first time at which $\theta_t$ reaches the memorization boundary $\partial M$ defines the \emph{stopping time} \citep{karatzas1991brownian} and mean squared memorization radius as
\begin{align}
\tau_M(S_0) := \inf\{t \ge 0 : \theta_t(S_0) \in \partial M\},
\qquad \text{and} \qquad 
\rho_M^2 := \mathbb{E}_{S_0 \in \Theta_S}\bigl[\|\theta_{\tau_M(S_0)}\|_2^2  \bigr]
,
\end{align}
respectively. This defines a \emph{Dirichlet exit problem} on the state $S_t$, with absorbing boundary $\partial M$ and the \emph{infinitesimal generator} (\ie, a differential operator) \citep{dynkin1965markov} induced by Adam's SDE in Lemma~\ref{lem:adam_joint_sde} as:
\begin{align}
\mathcal{L}_S[\bullet] \;:=\; \mu(S)^{\!\top}\nabla_S[\bullet] \;+\; \frac{1}{2}\,\frac{\eta}{b}\,\mathrm{tr}\!\bigl(\Sigma_S(S)\,\nabla_S^2\,[\bullet] \,\bigr),
\end{align}
which contains an operator modulated by a coefficient $\eta/b$ with $\Sigma_S(S) := \sigma(S)\sigma(S)^{\!\top}$. The resulting Dirichlet partial differential equation (PDE) admits no closed-form solution, we apply a low-order regular perturbation expansion~\citep{benderorszag1999} in $\varepsilon=\eta/b$ to obtain Theorem~\ref{thm:memorization_radius_scaling}. The proof is provided in Appendix~\ref{app:proof_memorization_radius_scaling_appendix}. Experimental verification is provided for $S_5$ in Figure~\ref{fig:radius_scaling_rho_M_s5} and for $\mathbb{Z}_{127}$ in Appendix~\ref{app:additional_results_rho_M_z127_appendix}.

\begin{figure*}[!t]
    \centering

    \includegraphics[width=\linewidth]{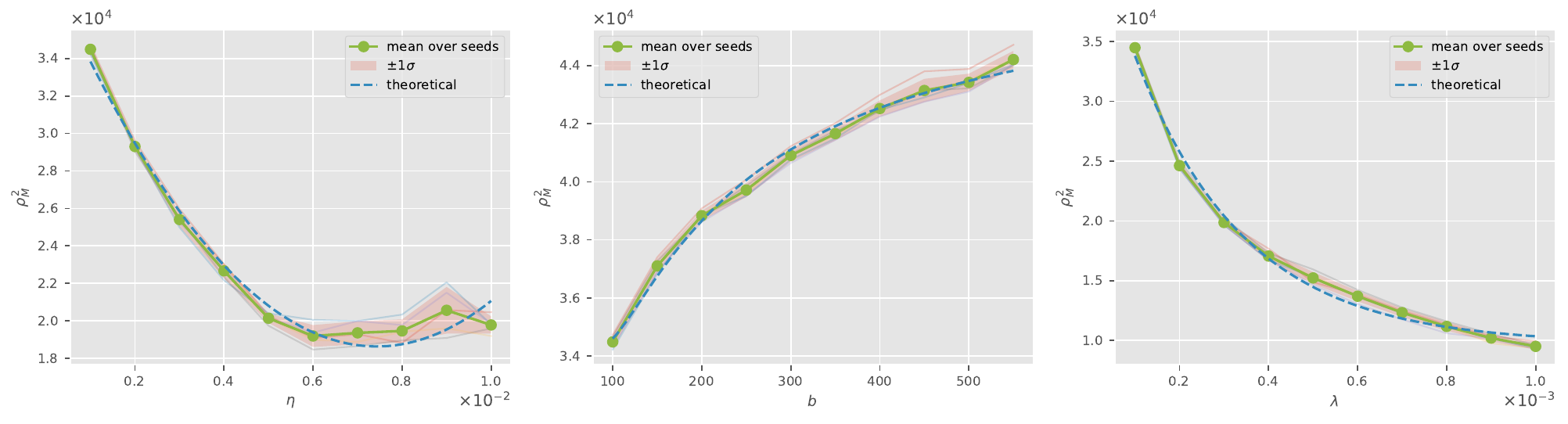}

    \caption{\textbf{Scaling Laws of Manifold Radius $\rho_M^2$ on $S_5$.} We show the scaling law of $\rho_M^2$ with respect to the learning rate $\eta$, batch size $b$, and $\ell_2$ regularization coefficient $\lambda$ on the $S_5$ task. For each hyperparameter configuration, we train for ten runs. The results show that larger $\eta/b$ modifies the dynamics with a stronger diffusion variance, whereas $\lambda$ does not affect the diffusion variance. We also overlay the theoretical fits. \textbf{As predicted by theory, the scaling law of $\rho_M^2$ with respect to the learning rate $\eta$ exhibits a U-shaped curve.} An additional experimental results for $\mathbb{Z}_{127}$ are provided in Appendix~\ref{app:additional_results_rho_M_z127_appendix}.}

    \label{fig:radius_scaling_rho_M_s5}

\end{figure*}

\begin{theorem}[Scaling Law of Memorization Radius (Perturbation Solution)]
\label{thm:memorization_radius_scaling}
The memorization radius approximately admits the scaling law with respect to $\eta/b$ as
\begin{align}
\rho_M^2 \; = \; (\rho_M^{(0)})^2 \;+\; \frac{\eta}{b}\,c_M \;+\; \left(\frac{\eta}{b}\right)^{\!2}\!c_M^{(2)} \;+\; O\!\left(\!\left(\frac{\eta}{b}\right)^{\!3}\right),
\label{eq:law_rho_M}
\end{align}
where $(\rho_M^{(0)})^2$, $c_M$, $c_M^{(2)}$ are task-determined constants independent of $\eta$ and $b$, as defined in Appendix~\ref{app:proof_memorization_radius_scaling_appendix}. The quadratic regularizer $\frac{\lambda}{2}\|\theta_t\|_2^2$ induces deterministic contraction. The leading constant $(\rho_M^{(0)})^2$ depends on the deterministic gradient flow $\nabla\mathscr{L}_f(\theta_t)$, when the task gradient $\nabla\mathscr{L}_f^*(\theta_t)$ is small, $\lambda$ enters $(\rho_M^{(0)})^2$ through $\dd\|\theta_t\|_2^2 = 2\theta_t^\top \dd \theta_t \approx -2\lambda \|\theta_t\|_2^2 \, \dd t \Rightarrow \frac{1}{(\rho_M^{(0)})^2}\dd (\rho_M^{(0)})^2 = -2 \lambda \, \dd t$. Thus, $(\rho_M^{(0)})^2  \propto \exp(-O(\lambda))$, and hence $\rho_M^2\propto \exp(-O(\lambda))$.
\end{theorem}

\begin{remark}
In the radius-shrinkage regime, diffusion first shifts the hitting point inward, giving $c_M<0$, while boundary curvature produces a stabilizing second-order correction, giving $c_M^{(2)}>0$. Our theory predicts that $\rho_M^2$ is U-shaped in $\eta/b$ with two regimes: a sub-linear decrease $\rho_M^2 \approx (\rho_M^{(0)})^2 + (\eta/b)c_M$ for $\eta/b \ll |c_M|/(2c_M^{(2)})$, and a sub-quadratic increase $\rho_M^2 \approx (\rho_M^{(0)})^2 + (\eta/b)^2 c_M^{(2)}$ for $\eta/b \gtrsim |c_M|/(2c_M^{(2)})$, with minimum at $(\eta/b)^\star = |c_M| /(2c_M^{(2)})$. This U-shaped effect is observed in Figure~\ref{fig:radius_scaling_rho_M_s5}, where the scaling law with respect to $\eta$ exhibits a U-shaped curve under large $b$.
 
\end{remark}

\subsection{Late-Stage Radius Dynamics}
\label{sec:late_stage_radius_dynamics}

To characterize the manifold geometry, we derive the radius SDE induced by Adam's SDE~\eqref{eq:adam_joint_sde} via It\^o's lemma under the late-stage reduction observations. Empirically, as shown in Figure~\ref{fig:characterizing_dynamics}, once the trajectory $\{S_t\}$ reaches the memorization manifold $M$ (\ie, memorization regime), the late-stage grokking dynamics exhibit two reduction properties. First, refer to equation~\eqref{eq:adam_joint_drift_diffusion}, the first- and second-moment estimates $m_t$ and $v_t$ rapidly converge to $\bar g_t\bigl(1-e^{-\alpha_1 t}\bigr)$ and $\overline{g\odot g}_t$, respectively, at an exponential rate $\exp(-O(t))$. We refer to this property as \emph{slow-manifold reduction}, formalized in Observation~\ref{obs:slow_manifold}. Second, refer to equation~\eqref{eq:diffusion_cross_term_def}, the gradient fluctuations, characterized by $\Sigma_t$, become small and isotropic at an exponential rate $\exp(-O(t))$. We refer to this property as \emph{small-isotropic gradient covariance}, formalized in Observation~\ref{obs:small_isotropic_noise}.

\begin{figure*}[!t]
    \centering

    \begin{subfigure}[t]{0.32\linewidth}
        \centering
        \includegraphics[width=\linewidth]{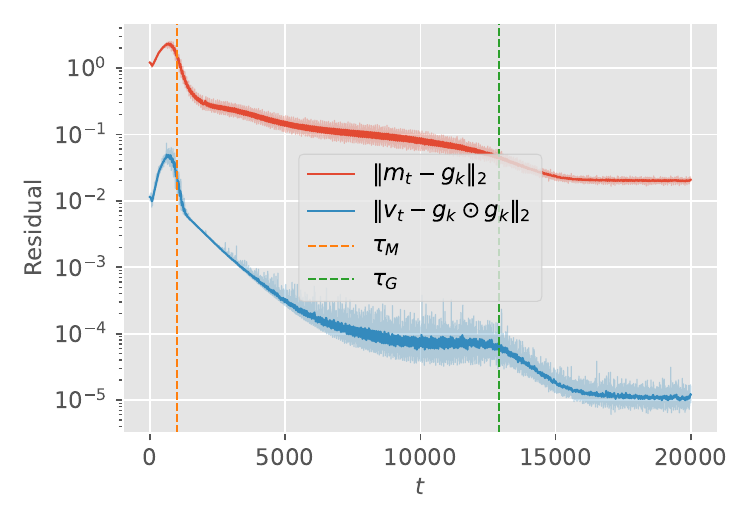}
        \caption{Reduction on $(m_t,v_t)$}
        \label{fig:characterizing_dynamics_mv}
    \end{subfigure}
    \begin{subfigure}[t]{0.32\linewidth}
        \centering
        \includegraphics[width=\linewidth]{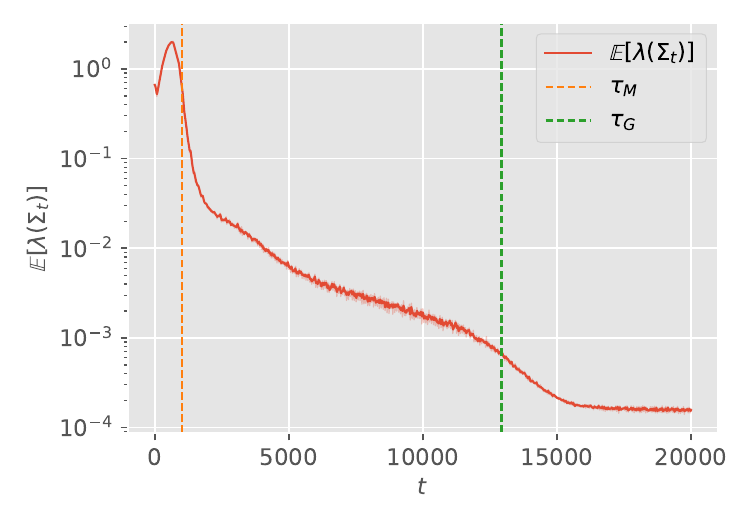}
        \caption{Gradient noise magnitude}
        \label{fig:characterizing_dynamics_noise_magnitude}
    \end{subfigure}
    \begin{subfigure}[t]{0.32\linewidth}
        \centering
        \includegraphics[width=\linewidth]{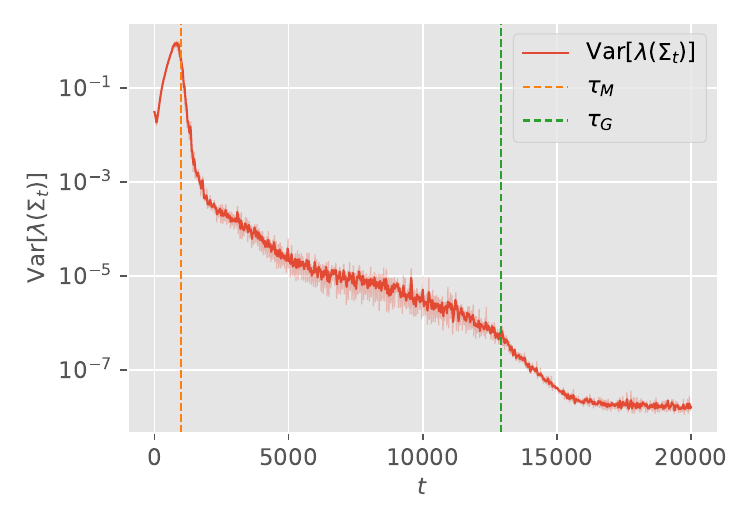}
        \caption{Gradient noise isotropy}
        \label{fig:characterizing_dynamics_noise_isotropy}
    \end{subfigure}

    \caption{\textbf{Observations in Grokking Dynamics.} The markers $\tau_M$ and $\tau_G$ denote the first-hitting times of $\partial M$ and $\partial G$, respectively, and the $y$-axis is shown on a logarithmic scale. Left (\subref{fig:characterizing_dynamics_mv}) shows the exponential reduction of moment estimates; $(m_t,v_t)$ rapidly converges to $\bigl(\bar g_t(1-e^{-\alpha_1 t}), \overline{g\odot g}_t(1-e^{-\alpha_2 t})\bigr)$, with exponential decay $\exp(-O(t))$. Middle--right (\subref{fig:characterizing_dynamics_noise_magnitude}--\subref{fig:characterizing_dynamics_noise_isotropy}) shows the exponential reduction of the gradient noise; the mean and variance of the singular values of the gradient covariance $\Sigma_t$ decay as $\exp(-O(t))$, indicating small and isotropic gradient fluctuations in late-stage dynamics.}

    \label{fig:characterizing_dynamics}
\end{figure*}

\begin{observation}[Slow-Manifold Reduction on $(m_t,v_t)$]
\label{obs:slow_manifold}
If $(m_t,v_t)$ varies slowly as $t \to \infty$, the first- and second-moment estimates $m_t,v_t$ rapidly converge to its mean-field limit
\begin{align}
m_t \;\to\; \bar g_t\,\bigl(1 - e^{-\alpha_1 t}\bigr) \to \bar g_t \quad \text{and} \quad v_t \;\to\; \overline{g\odot g}_t (1 - e^{-\alpha_2t}) \to \overline{g\odot g}_t
,
\label{eq:slow_manifold_reduction}
\end{align}
respectively, at an exponential rate $\exp(-O(t))$. The proof for this mean-field limit is provided in Lemma~\ref{lem:meanfield_mt_and_vt_limit} of Appendix~\ref{app:mean_field_mt_and_vt_limit}.
\end{observation}

\begin{observation}[Small-Isotropic Gradient Covariance]
\label{obs:small_isotropic_noise}
In Adam's late-stage dynamics, the gradient covariance $\Sigma_t$ rapidly becomes small and approximately isotropic as $t \to \infty$. With this isotropy and by equation~\eqref{eq:diffusion_cross_term_def}, the Adam's preconditioner is therefore approximated by $\pi(\theta)\approx s(\theta)^{-1}I_p$, as $t \to \infty$, where the state-dependent scalar $s(\theta)>0$ is defined by
\begin{align}
\frac{1}{s(\theta)}
:=
\frac{1}{p}\mathrm{tr}\bigl(\pi(\theta)\bigr),
\qquad
\pi(\theta)
:=
\mathrm{diag}\!\left(\overline{g\odot g}(\theta)\right)^{-1/2}
,
\end{align}
where $p$ is the parameter dimension. In particular, let $G(\theta) := \frac{1}{\sqrt{b}}B(t)\pi(\theta)\Sigma(\theta)^{1/2}$ be the preconditioned diffusion factor, which admits $G(\theta)G(\theta)^\top\approx I_p$. See Appendix~\ref{app:memorization_regime_identities}.
\end{observation}

Under Observations~\ref{obs:slow_manifold}--\ref{obs:small_isotropic_noise}, Lemma~\ref{lem:reduced_late_stage_radius_sde} gives the approximate closed-form dynamics of $r_t^2=\|\theta_t\|_2^2$. The proof is provided in Appendix~\ref{app:reduced_late_stage_radius_sde_appendix}, where we also experimentally verify that the lemma accurately characterizes the late-stage dynamics of manifold radii in Figure~\ref{fig:late_stage_radius_sde_appendix}.

\begin{lemma}[Reduced Late-Stage Radius SDE]
\label{lem:reduced_late_stage_radius_sde}
Let $r_t^2:=\|\theta_t\|_2^2$. In Adam's late-stage regime underlying grokking, assume that the diffusion covariance becomes isotropic $G(\theta_t)G(\theta_t)^\top\approx I_p$ as $t\to\infty$, then the reduced late-stage squared-radius SDE admits
\begin{align}
\dd r_t^2
\approx
\left[
-\frac{2\lambda}{s(\theta_t)}r_t^2
-
\frac{2}{s(\theta_t)}
\theta_t^\top\bigl(\bar g_t-\lambda\theta_t\bigr)
+
\eta\,p
+
\mathcal R_{\mathrm{SM}}(t)
+
\mathcal R_{\pi}(t)
\right]\dd t
+
2\sqrt{\eta}\,r_t\,\dd W_t^{(r)},
\label{eq:reduced_late_stage_radius_sde}
\end{align}
where slow-manifold residual $\mathcal R_{\mathrm{SM}}(t)$ and preconditioner residual $\mathcal R_{\pi}(t)$ are
\begin{align}
\mathcal R_{\mathrm{SM}}(t)
&:=
-2B(t)\theta_t^\top
\left[
\mathbb{E}\left[ m_t\oslash\sqrt{v_t}\right]
-
\bar g_t\bigl(1-e^{-\alpha_1 t}\bigr)
\oslash
\sqrt{\overline{g\odot g}_t}
\right],
\label{eq:slow_manifold_residual_def}
\\
\mathcal R_{\pi}(t)
&:=
-2\theta_t^\top
\left[
\pi(\theta_t)-s(\theta_t)^{-1}I_p
\right]
\bar g_t
,
\end{align}
respectively, and a one-dimensional Wiener process $W_t^{(r)}:=\int_0^t e_r(\theta_u)^\top G(\theta_u)\dd W_u$ with $e_r(\theta):=\frac{\theta}{\|\theta\|_2}$ by L\'evy's characterization: $\langle W_t^{(r)} \rangle_t=t$. Empirically, in the memorization regime, the residual sum $\mathcal R_{\mathrm{SM}}(t)+\mathcal R_{\pi}(t)$ becomes negligible. Hence, we set $\mathcal R_{\mathrm{SM}}(t)+\mathcal R_{\pi}(t)=0$ in the subsequent analysis, see Appendix~\ref{app:reduced_late_stage_radius_sde_appendix}.
\end{lemma}

\subsection{Generalization Radius}
\label{sec:law_generalization}

\begin{figure*}[!t]
    \centering

    \includegraphics[width=\linewidth]{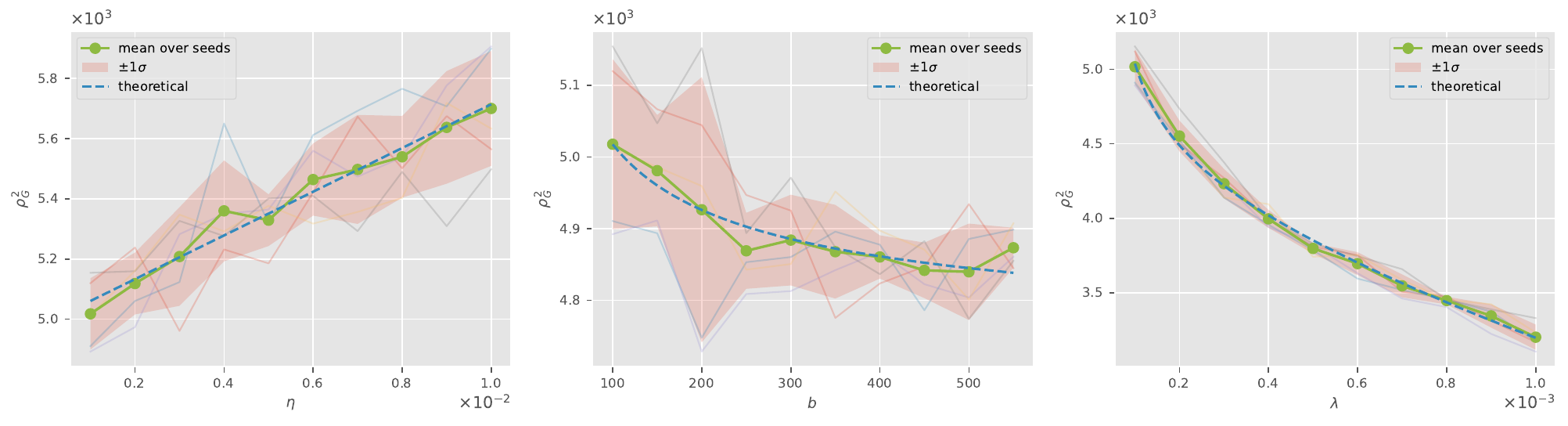}

    \caption{\textbf{Scaling Laws of Manifold Radius $\rho_G^2$ on $S_5$.} We show the scaling laws of $\rho_G^2$ with respect to the learning rate $\eta$, batch size $b$, and $\ell_2$ regularization coefficient $\lambda$ on the $S_5$ task. For each hyperparameter configuration, we train for ten runs. The results show that larger $\eta$ induces stronger diffusion variance, whereas $\lambda$ does not affect the diffusion variance. We also overlay the theoretical fits. An additional experimental results for $\mathbb{Z}_{127}$ are provided in Appendix~\ref{app:additional_results_rho_G_z127_appendix}.}

    \label{fig:radius_scaling_rho_G_s5}

\end{figure*}

Let $S_{\tau_M}$ denote the state at which the optimization trajectory first reaches $\partial M$. We treat $S_{\tau_M}$ as the initial state for the subsequent \emph{first-hitting problem} \citep{karatzas1991brownian} on $\partial G$. Define the \emph{stopping time} for first hitting $\partial G$ and the corresponding mean squared generalization radius as
\begin{align}
\tau_G(S_0)
:=
\inf\{t \ge \tau_M : \theta_t(S_0) \in \partial G\},
\qquad
\rho_G^2
:=
\mathbb{E}_{S_0 \in \Theta_S}\bigl[\|\theta_{\tau_G(S_0)}\|_2^2\bigr]
.
\end{align}
Once the trajectory enters $\partial G$, the drift and diffusion in equation~\eqref{eq:adam_joint_sde} equilibrate, and Adam's SDE approaches its equilibrium. The center of this equilibrium starting from $S_0$ is the regularized local minimizer
\begin{align}
\theta^\star(S_0) :=\arg\min_{\theta(S_0)}
\mathscr{L}^{*}_f(\theta(S_0))+ \frac{\lambda}{2}\|\theta(S_0)\|_2^2
,
\end{align}
where $\mathscr{L}^{*}_f(\theta(S_0))$ is task loss. This yields an equilibrium-distribution problem for the reduced late-stage dynamics around $\theta^\star(S_0)$, initialized from the post-memorization state $S_{\tau_M}(S_0)$. Linearizing the reduced dynamics around $\theta^\star(S_0)$ solves Theorem~\ref{thm:generalization_radius_scaling}. The proof is provided in Appendix~\ref{app:proof_generalization_radius_scaling_appendix}. Experimental verification is provided for $S_5$ in Figure~\ref{fig:radius_scaling_rho_G_s5} and for $\mathbb{Z}_{127}$ in Appendix~\ref{app:additional_results_rho_G_z127_appendix}.

\begin{theorem}[Scaling Law of Generalization Radius]
\label{thm:generalization_radius_scaling}
The mean squared generalization manifold radius admits the asymptotic expansion
\begin{align}
\rho_G^2 \;\approx \; \bigl(\rho_G^{(0)}\bigr)^2 \;+\; \frac{\eta}{\lambda}\,c_G \;+\; O\!\left(\tfrac{\eta^2}{\lambda^2}\right),
\label{eq:law_rho_G}
\end{align}
with $(\rho_G^{(0)})^2 := \mathbb{E}_{S_0\in\Theta_S}\left[\|\theta^\star(S_0)\|_2^2\right]$ the equilibrium radius modulated by $\lambda$ and $c_G>0$ a task-determined constant independent of $\eta$ and $\lambda$, as defined in Appendix~\ref{app:proof_generalization_radius_scaling_appendix}. A weak dependence on $\lambda$ enters $(\rho_G^{(0)})^2$ through $(\rho_G^{(0)})^2 \propto \exp(-O(\lambda))$; and a weak batch-size dependence enters through $c_G \propto O(\frac{1}{\sqrt b})$.
\end{theorem}

\subsection{Solution Transition Time}
\label{sec:law_transition_time}

Starting from $S_0 \in \Theta_S$, the transition time on the trivial annulus $M \setminus G$ is given as
\begin{align}
\tau_{M \to G}(S_0) \;:=\; \tau_G(S_0) - \tau_M(S_0),
\end{align}
where $\tau_M(S_0)$ and $\tau_G(S_0)$ are the stopping times that the trajectory first hits $\partial{M}$ and $\partial{G}$, respectively, defined in Sections~\ref{sec:law_memorization}--\ref{sec:law_generalization}. Taking $\theta_{\tau_M}(S_0)\in\partial M$ as the initial condition (Theorem~\ref{thm:memorization_radius_scaling}), we formulate the transition from $\partial M$ to $\partial G$ as a \emph{mean first-passage problem} \citep{karatzas1991brownian}
 for the radial process $r_t:=\|\theta_t\|_2$ on the annulus $M\setminus G$. Its equilibrium distribution characterizes the equilibrium around the generalization region, while the mean first-passage time from the memorization radius to the generalization radius gives the solution transition time. Solving this first-passage problem yields Theorem~\ref{thm:solution_transition_time_scaling}. The proof is provided in Appendix~\ref{app:proof_solution_transition_time_appendix}. Experimental verification is provided for $S_5$ in Figure~\ref{fig:transition_time_scaling_s5} and for $\mathbb{Z}_{127}$ in Appendix~\ref{app:additional_results_transition_time_z127_appendix}.

\begin{theorem}[Scaling Law of Solution Transition Time]
\label{thm:solution_transition_time_scaling}
The mean number of optimizer iterations between the first hit at $\partial M$ and the first hit at $\partial G$ admits the asymptotic expansion
\begin{align}
\tau_{M\to G}:=\mathbb{E}_{S_0 \in \Theta_S}\bigl[\tau_{M\to G}(S_0)\bigr] \;\approx\; \frac{\bar s}{\eta\,\lambda}\,\log\!\frac{\rho_M^{(0)}}{\rho_G^{(0)}} \;+\; \frac{c_\tau}{b\,\lambda} \;+\; \frac{c_\tau^{(2)}}{\lambda^2} \;+\; O\!\left(\frac{\eta}{\lambda^3}\right),
\label{eq:law_transition_time}
\end{align}
where $\bar s$ is determined by the harmonic mean of the effective deterministic gradient flow, while $c_\tau$ and $c_\tau^{(2)}$ collect the trajectory corrections, as defined in Appendix~\ref{app:proof_solution_transition_time_appendix}. All three are task-determined constants independent of $\eta$ and $\lambda$; a weak batch-size $b$-dependence enters through the scalar preconditioning scale $\bar s \propto O(\frac{1}{\sqrt{b}})$, and hence $c_\tau \propto O(\frac{1}{\sqrt{b}})$ and $c_\tau^{(2)} \propto O(\frac{1}{b})$.
\end{theorem}

\begin{figure*}[!t]
    \centering

    \includegraphics[width=\linewidth]{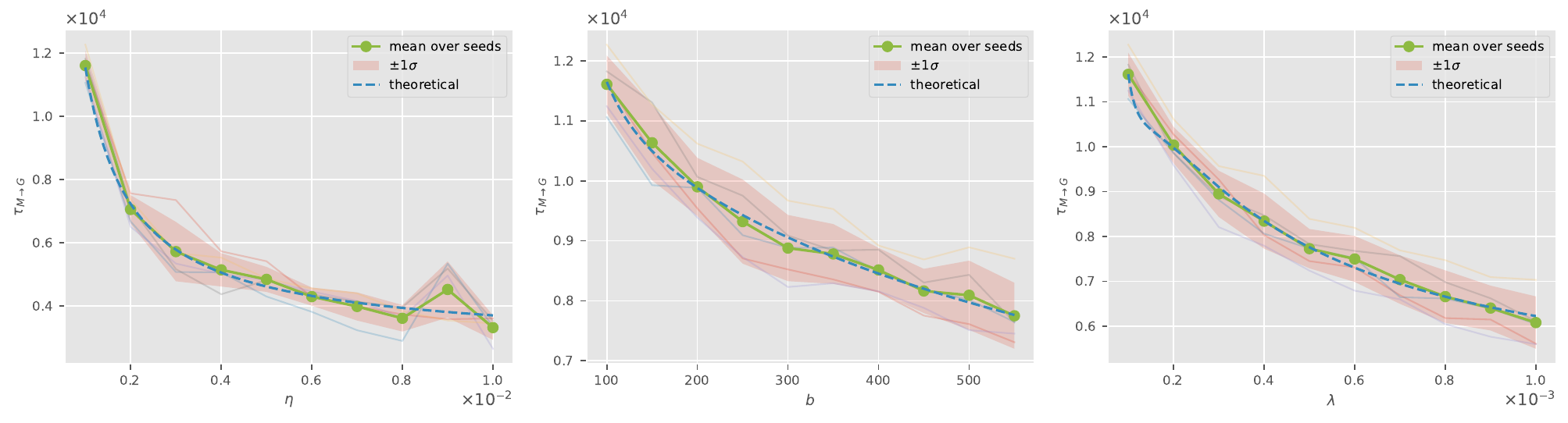}

    \caption{\textbf{Scaling laws of solution transition time on $S_5$.} We show that the solution transition time $\tau_{M \to G}$ from the memorization manifold $M$ to the generalization manifold $G$ scales with the learning rate $\eta$, batch size $b$, and $\ell_2$ regularization coefficient $\lambda$. For each hyperparameter configuration, we train for ten runs. We also overlay the theoretical fits. An additional experimental results for $\mathbb{Z}_{127}$ are provided in Figure~\ref{fig:transition_time_scaling_z127_appendix}.}

    \label{fig:transition_time_scaling_s5}

\end{figure*}

\section{Validation}
\label{sec:validation}

We validate the three scaling laws in Section~\ref{sec:experimental_validation} by directly measuring $\rho_M^2$, $\rho_G^2$, and $\tau_{M \to G}$ on the group-theoretic learning task $S_5$ and modular arithmetic tasks over $\mathbb{Z}_{127}$, and further compare them with published results from the grokking literature in Section~\ref{sec:literature_validation}.

\subsection{Experimental Validation}
\label{sec:experimental_validation}

\paragraph{Experimental Setting.}
We use a two-layer MLP: each input token is embedded into $256$ dimensions, the two embeddings are concatenated, passed through a width-$512$ ReLU hidden layer, and projected to $|\mathcal{Y}|$ output logits, where $|\mathcal{Y}|=120$ for $S_5$ and $127$ for $\mathbb{Z}_{127}$. The embedding layers are initialized from a standard normal distribution, while the linear layers use \emph{Kaiming} initialization. Training is performed in $32$-bit floating-point precision. We find that $16$-bit precision can introduce numerical instability in grokking, suggesting that grokking dynamics are sensitive to numerical precision errors. We use base hyperparameters $\eta=10^{-3}$, $(\beta_1,\beta_2)=(0.9,0.999)$, $\ell_2$ regularization coefficient $\lambda=10^{-4}$, and batch size $100$. These base values are varied as needed to study scaling laws.

\paragraph{Learning Tasks.} We adopt two structured learning tasks: a group-theoretic learning task on $S_n$ and a modular arithmetic learning task on $\mathbb{Z}_p$. For $S_n$, each input is an ordered pair of permutations $(f,g)\in S_n^2$, and the model is trained to predict their group product $f\circ g$ under permutation composition. For $\mathbb{Z}_p$, each input is a pair $(a,b)\in\mathbb{Z}_p^2$, and the model is trained to predict the modular sum $a+b \bmod p$. In both tasks, the dataset consists of all possible input--output pairs, with a randomly sampled subset used for training and the remaining pairs used for evaluation. More details are provided in Appendix~\ref{app:learning_tasks_appendix}.

\paragraph{Results.}
Figures~\ref{fig:radius_scaling_rho_M_s5} and~\ref{fig:radius_scaling_rho_G_s5} validate the predicted scaling laws of the manifold radii $\rho_M^2$ and $\rho_G^2$ with respect to the learning rate $\eta$, batch size $b$, and regularization coefficient $\lambda$, as stated in Theorems~\ref{thm:memorization_radius_scaling} and~\ref{thm:generalization_radius_scaling}. Figure~\ref{fig:transition_time_scaling_s5} validates the predicted scaling law of the solution transition time $\tau_{M\to G}$ with respect to the same hyperparameters, as stated in Theorem~\ref{thm:solution_transition_time_scaling}. Additional results on $\mathbb{Z}_{127}$ are provided in Appendices~\ref{app:additional_results_rho_M_z127_appendix},~\ref{app:additional_results_rho_G_z127_appendix}, and~\ref{app:additional_results_transition_time_z127_appendix}. Overall, the experiments support the theoretical predictions; in particular, $\rho_M^2$ exhibits the predicted U-shaped dependence on $\eta/b$. 

\paragraph{Discussions and Limitations.} The scaling laws originate from how $\eta$, $b$, and $\lambda$ modulate Adam's SDE: $\eta$ and $b$ module the stochastic diffusion by scaling the gradient covariance with a factor $\sqrt{\eta/b}$, whereas $\lambda$ modulates the deterministic gradient drift flow by regularization. Our results inherit the conditions required by Adam's SDE limit, including an adapted Wiener-process approximation. These conditions require a sufficiently small learning rate and a sufficiently large batch size; thus, our results do not directly apply to the large-$\eta/b$ regime.


\subsection{Literature Validation}
\label{sec:literature_validation}













Theorem~\ref{thm:memorization_radius_scaling} characterizes the scaling of the memorization radius $\rho_M$ with respect to $\eta/b$, which, to the best of our knowledge, has not been derived in prior work. Theorem~\ref{thm:generalization_radius_scaling} derives the scaling law of the generalization radius $\rho_G$, showing that it decomposes into the landscape-dependent term $\bigl(\rho_G^{(0)}\bigr)^2=\mathbb{E}_{\theta_0\in\Theta}\left[\|\theta^\star(\theta_0)\|_2^2\right]$ and an $O(\eta/(\sqrt{b}\lambda))$ fluctuation correction; this is consistent with the weight-norm ``Goldilocks zone'' reported in prior studies \citep{zimingliuNIPS2022grokking,varma2023circuit}. Theorem~\ref{thm:solution_transition_time_scaling} recovers the inverse dependence of the grokking delay on the regularization coefficient, including the scaling $\tau\propto 1/(\eta\lambda)$ observed or implied in prior empirical and mechanistic studies \citep{zimingliuNIPS2022grokking}. Together, these results show that the proposed stopping-time analysis recovers known scaling behavior while refining it into separate contributions from learning rate, batch size, and regularization coefficient.

\section{Conclusions}

Optimization dynamics induce a topological configuration of solution manifolds underlying grokking. In this work, we characterize this configuration through a shell--core prior of the reachable solution space, in which training first reaches memorization solutions and then transitions, after a prolonged delay, toward generalization solutions. We derive scaling laws for the manifold radii and the solution transition time with respect to the learning rate, batch size, and regularization coefficient through stopping-time and equilibrium-distribution analyses of Adam's joint-state SDE. These results connect the geometry of solution manifolds with the stochastic dynamics of Adam and provide a theoretical framework for understanding grokking through optimization dynamics.

%% file: appendix.tex
\newpage
\section{Appendix}

This section provides proofs, experimental verification of the technical derivations, and additional experimental results. A high-level sketch of theoretical analysis framework is illustrated as in Figure~\ref{fig:high_level_sketch_flowchart}.

\begin{figure}[h]
    \centering

    \resizebox{0.85\linewidth}{!}{\input{tikz/tikz_high_level_sketch}}

    \caption{\textbf{High-Level Sketch of Theoretical Analysis Framework.}}
    \label{fig:high_level_sketch_flowchart}
\end{figure}
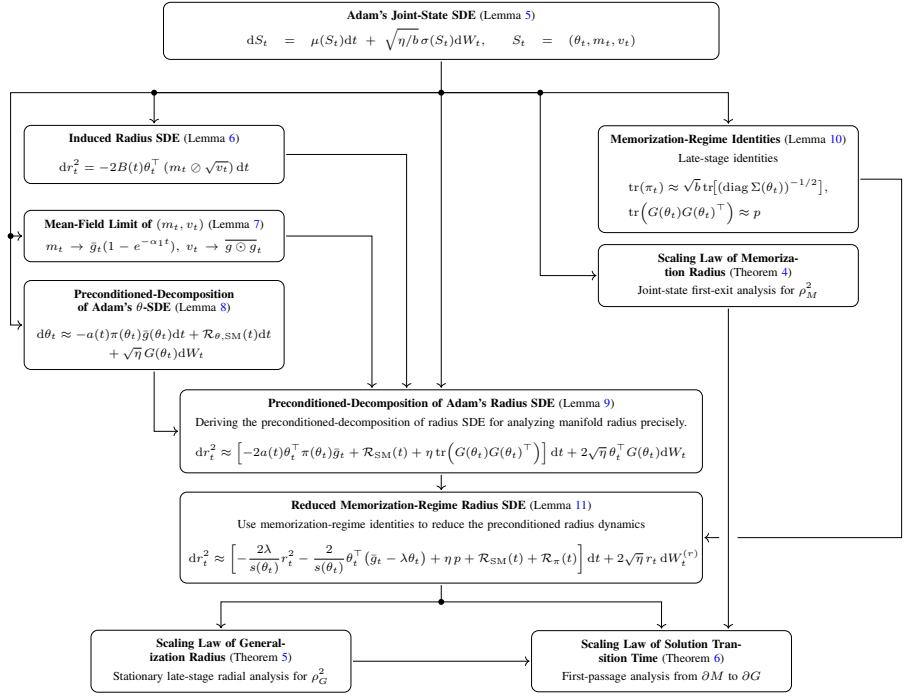

\startcontents[appendices]
\printcontents[appendices]{}{1}{Table of Appendix Contents:} 

\subsection{Experimental Settings: Learning Tasks}
\label{app:learning_tasks_appendix}

\paragraph{Group-theoretic learning task on $S_n$.} Let $S_n$ denote the symmetric group on $n$ elements. Each group element $f \in S_n$ is a permutation sending its index set $[n]:=\{1,2,\dots,n\}$ to a fixed enumeration of all $n!$ permutations. Let $f[i]$ denote the indexed permutation $i \mapsto j$ from an index $i \in [n]$ to a target $j \in [n]$. For example, given two permutations $f,g \in S_5$:
\begin{align}
    f=
    \begin{pmatrix}
    1 & 2 & 3 & 4 & 5 \\
    \downarrow & \downarrow & \downarrow & \downarrow & \downarrow \\
    2 & 5 & 1 & 4 & 3
    \end{pmatrix}
    \quad
    g
    =
    \begin{pmatrix}
    1 & 2 & 3 & 4 & 5 \\
    \downarrow & \downarrow & \downarrow & \downarrow & \downarrow \\
    3 & 1 & 5 & 2 & 4
    \end{pmatrix}\notag
    ,
    \end{align}
    the permutations are given by $f[1]: 1 \mapsto 2$, $f[2]: 2 \mapsto 5$, \etc. The product of $f$ and $g$ is defined by permutation composition
    $(f \circ g)[i]=g\bigl[f[i]\bigr]$ given by
    \begin{align}
    f \circ g=
    \begin{pmatrix}
    1 & 2 & 3 & 4 & 5 \\
    \downarrow & \downarrow & \downarrow & \downarrow & \downarrow \\
    1 & 4 & 3 & 2 & 5
    \end{pmatrix}\notag
    .
    \end{align}
    The learning task is to predict the group product $f \circ g$ given the pair $(f, g)$. The dataset consists of all ordered pairs $(f, g)\in S_n^2$, with a randomly sampled subset used for training and the remaining pairs used for evaluation.

\paragraph{Modular arithmetic learning task on $\mathbb{Z}_p$.} Let $\mathbb{Z}_p$ denote the cyclic group of integers modulo $p$. Each input consists of a pair $(a,b)\in\mathbb{Z}_p^2$, and the task is to predict their modular sum $a+b \bmod p$. As in the group-theoretic task, training is performed on a randomly selected subset of all $p^2$ input–output pairs, while evaluation is conducted on the full set.

\newpage
\subsection{Additional Results: Adam--Induced Shell--Core Radius \& Stopping-Time Concentration}
\label{app:adam_induced_radius_concentration}

\input{tables/tab_radius_concentration}

\subsection{Proof: Concentration of Normal Initialization}
\label{app:initialization_concentration_normal_appendix}

\input{proofs/proof_concentration_normal}
\subsection{Proof: Concentration of Uniform Initialization}
\label{app:initialization_concentration_uniform_appendix}

\input{proofs/proof_concentration_uniform}

\newpage
\subsection{Proof: Adam's Closed-Form Continuous-Time SDE Limit}
\label{app:adam_ct_sde_appendix}

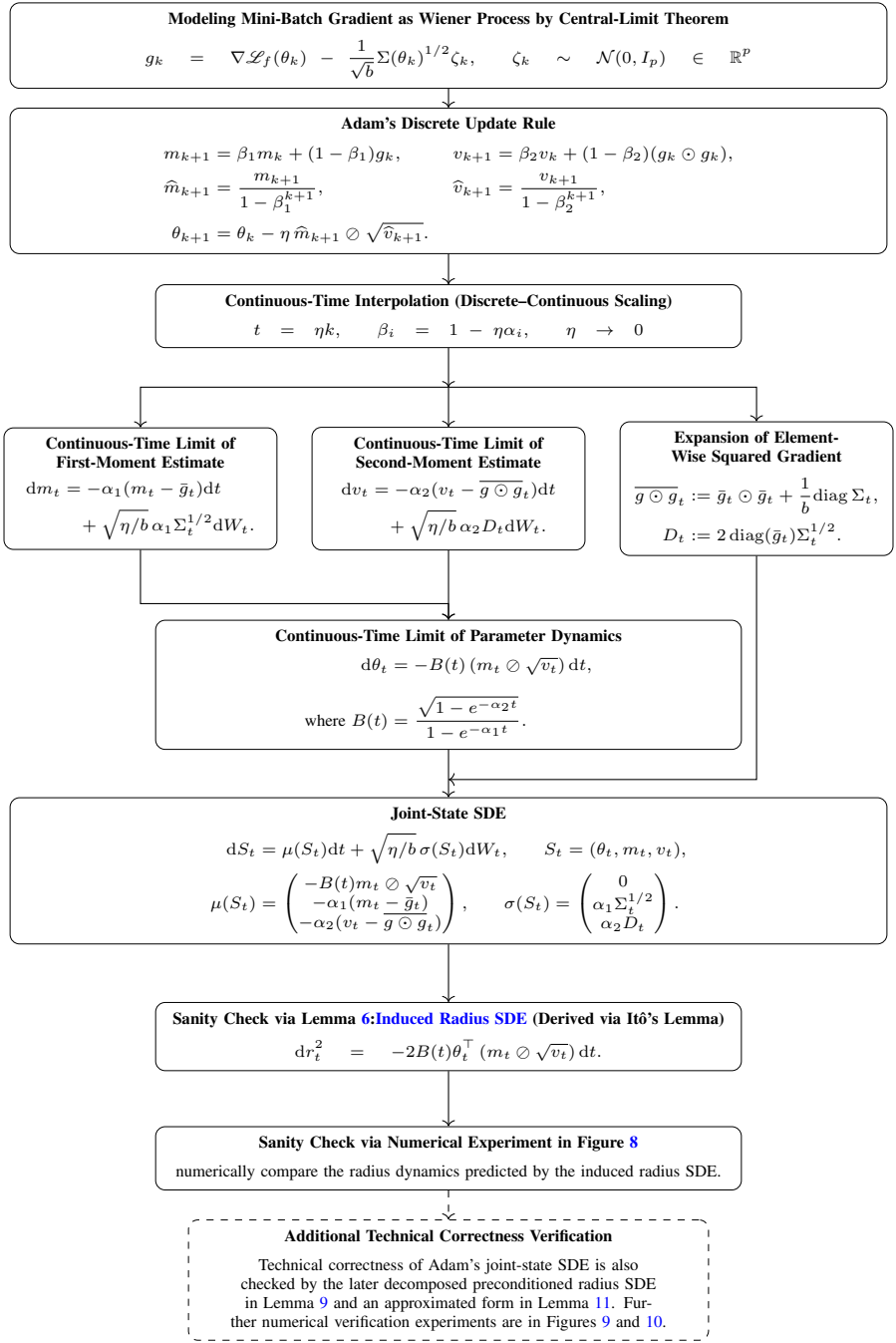
\begin{figure}[h]
    \centering

    \resizebox{0.85\linewidth}{!}{\input{tikz/tikz_proof_sketch_adam_sde_limit}}

    \caption{\textbf{Proof Sketch for Adam's Closed-Form Continuous-Time SDE Limit.} This diagram illustrates the proof sketch and the corresponding correctness checks for Adam's continuous-time SDE limit. The mini-batch gradient is modeled as a Wiener process by \textbf{Central-Limit Theorem}. Combining this stochastic gradient model with Adam's discrete update rules and the continuous-time interpolation yields the continuous-time limits of the first- and second-moment estimates, as well as the continuous-time parameter dynamics. Together with the expansion of the element-wise squared gradient, these components give the joint-state SDE for $S_t:=(m_t,v_t,\theta_t)$.}
    \label{fig:adam_closed_form_sde_proof_flowchart}
\end{figure}

\input{proofs/proof_adam_joint_state_SDE}

\newpage
\subsection{Proof: Induced Radius SDE}
\label{app:induced_radius_sde_appendix}

\begin{figure*}[h]
    \centering

    \begin{subfigure}[t]{0.45\linewidth}
        \centering
        \includegraphics[width=\linewidth]{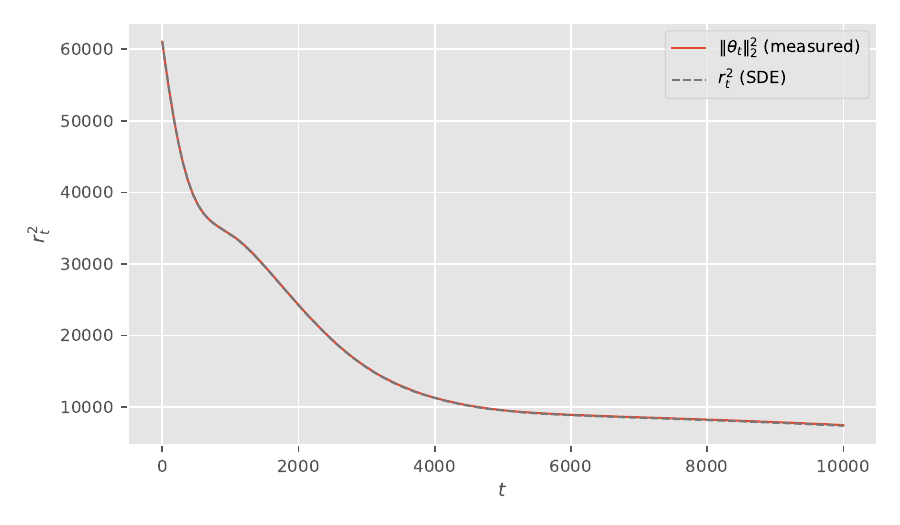}
        \caption{Problem $S_5$}
        \label{fig:radius_sde_sanity_s5_appendix}
    \end{subfigure}
    \begin{subfigure}[t]{0.45\linewidth}
        \centering
        \includegraphics[width=\linewidth]{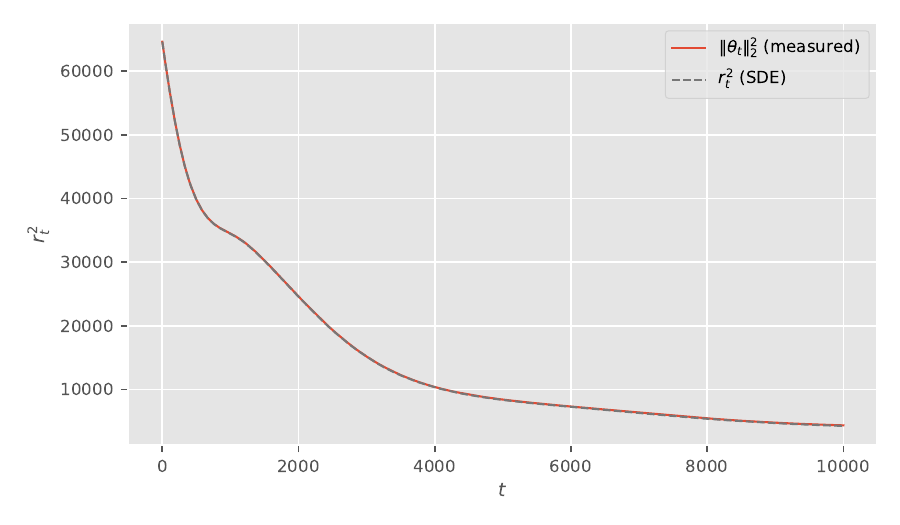}
        \caption{Problem $\mathbb{Z}_{127}$}
        \label{fig:radius_sde_sanity_z127_appendix}
    \end{subfigure}

    \caption{\textbf{Sanity Check with Radius SDE.} We use the induced radius SDE (Lemma~\ref{lem:induced_radius_sde_restated}), derived via It\^o's lemma, as a sanity check for Adam's continuous-time SDE limit.}

    \label{fig:radius_sde_sanity_appendix}
\end{figure*}

We use the induced radius SDE as a sanity check on the joint-state SDE of Lemma~\ref{lem:adam_joint_sde}. Applying It\^o's lemma~\citep{oksendal2003sde} to the quadratic form $r_t^2 = S_t^{\!\top} E\, S_t$ with the $\theta$-block projector $E \in \mathbb{R}^{3p\times 3p}$ defined in equation~\eqref{eq:theta_block_projector_def_appendix} below must reproduce the It\^o decomposition stated in Lemma~\ref{lem:induced_radius_sde_restated}; empirical verification on Adam runs is reported in Figure~\ref{fig:radius_sde_sanity_appendix}.

\input{proofs/proof_induced_radius_SDE}

\subsection{Proof: Mean-Field Limit of First- and Second-Moment Estimates}
\label{app:mean_field_mt_and_vt_limit}
\input{proofs/proof_mean_field_mt_and_vt_limit}

\subsection{Proof: Preconditioned-Decomposition of Adam's $\theta$-SDE}
\label{app:preconditioned_theta_sde_appendix}

We decompose Adam's exact $\theta$-dynamics into a preconditioned drift and diffusion with a residual. This decomposition helps us to formulate the late-stage evolution of training dynamics with Adam in grokking. The technical correctness is verified via numerical experiments in Appendix~\ref{app:preconditioned_radius_sde_appendix}.

\input{proofs/proof_preconditioned_decomposition_of_theta_SDE}

\newpage
\subsection{Proof: Preconditioned-Decomposition of Adam's Radius SDE}
\label{app:preconditioned_radius_sde_appendix}

We decompose Adam's exact radius SDE into preconditioned and residual terms in Lemma~\ref{lem:preconditioned_radius_sde_appendix}. This decomposition is used to derive the reduced radius SDE under the slow-manifold reduction and small-isotropic gradient covariance observations. Its technical correctness is verified in Figure~\ref{fig:preconditioned_radius_sde_appendix}, in which the predicted radius dynamics exactly match the theoretical values computed from Lemma~\ref{lem:preconditioned_radius_sde_appendix}. 

\begin{figure*}[h]
    \centering

    \includegraphics[width=\linewidth]{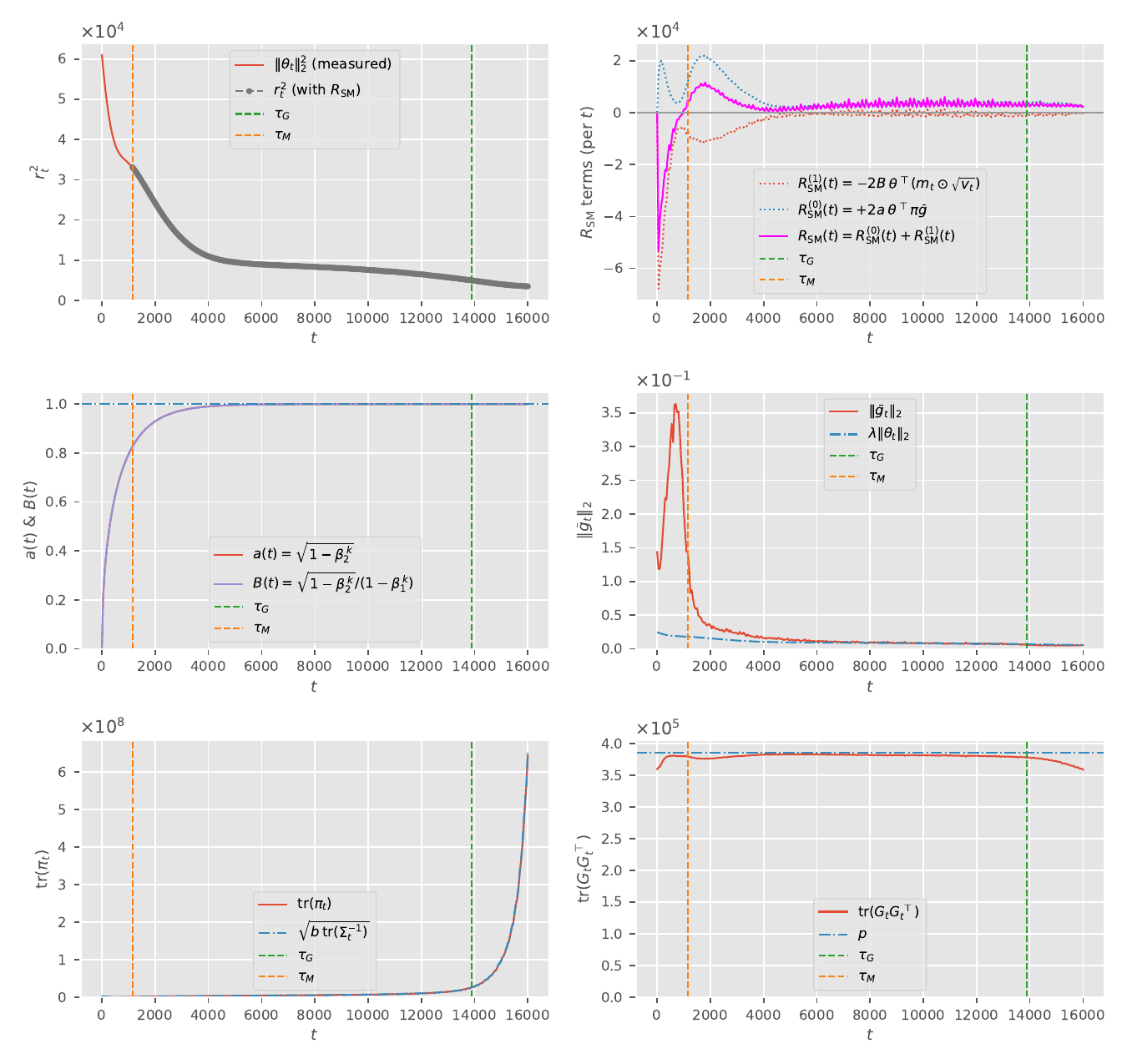}

    \caption{\textbf{Preconditioned Radius SDE.} This experiment shows the dynamics of the preconditioned radius SDE in Lemma~\ref{lem:preconditioned_radius_sde_appendix}, and two memorization-regime identities in Lemma~\ref{lem:late_stage_preconditioned_diffusion_isotropy} where $\mathrm{tr}(\pi(\theta_t))\approx \sqrt{b}\, \mathrm{tr}\!\bigl[\left(\mathrm{diag}\,\Sigma_t\right)^{-1/2}\bigr]$ and $\mathrm{tr}(G(\theta_t)G(\theta_t)^\top) \approx p$. The predicted radius dynamics match the theoretical values computed from Lemma~\ref{lem:preconditioned_radius_sde_appendix} and Lemma~\ref{lem:late_stage_preconditioned_diffusion_isotropy}, validating the correctness of the decomposition.}

    \label{fig:preconditioned_radius_sde_appendix}
\end{figure*}

\input{proofs/proof_preconditioned_decomposition_of_radius_SDE}

\subsection{Proof: Memorization-Regime Preconditioner and Effective Diffusion Identities}
\label{app:memorization_regime_identities}

The validation for Lemma~\ref{lem:late_stage_preconditioned_diffusion_isotropy} is provided in Figure~\ref{fig:preconditioned_radius_sde_appendix}. The experiment shows that, in the memorization regime, the approximated identities
\begin{align}
\left.
        \begin{array}{c}
            \pi(\theta_t) :=
            \mathrm{diag}\!\left(\overline{g\odot g}_t\right)^{-1/2} \\
            \overline{g\odot g}_t
            \approx
            (1/b)\,\mathrm{diag}\,\Sigma_t
        \end{array}
        \right\}
\Rightarrow
\mathrm{tr}(\pi(\theta_t))
\approx
\sqrt{b}\,
\mathrm{tr}\!\bigl[\left(
\mathrm{diag}\,\Sigma_t
\right)^{-1/2}\bigr]
\end{align}
and
\begin{align}
G(\theta_t)
=
\frac{1}{\sqrt{b}}B(t)\pi(\theta_t)\Sigma_t^{1/2}
\Rightarrow
\mathrm{tr}\!\left(G(\theta_t)G(\theta_t)^\top\right)
\approx
p
\end{align}
hold in the memorization regime.

\input{proofs/proof_memorization_regime_identities}

\newpage
\subsection{Proof: Reduced Memorization-Regime Radius SDE}
\label{app:reduced_late_stage_radius_sde_appendix}

We derive the exact closed-form reduced late-stage radius SDE with residual terms in Lemma~\ref{lem:reduced_late_stage_radius_sde_restated}. Its technical correctness is verified in Figure~\ref{fig:late_stage_radius_sde_appendix}, where the predicted radius dynamics match the theoretical values computed from Lemma~\ref{lem:reduced_late_stage_radius_sde_restated}.

\begin{figure*}[h]
    \centering

    \includegraphics[width=\linewidth]{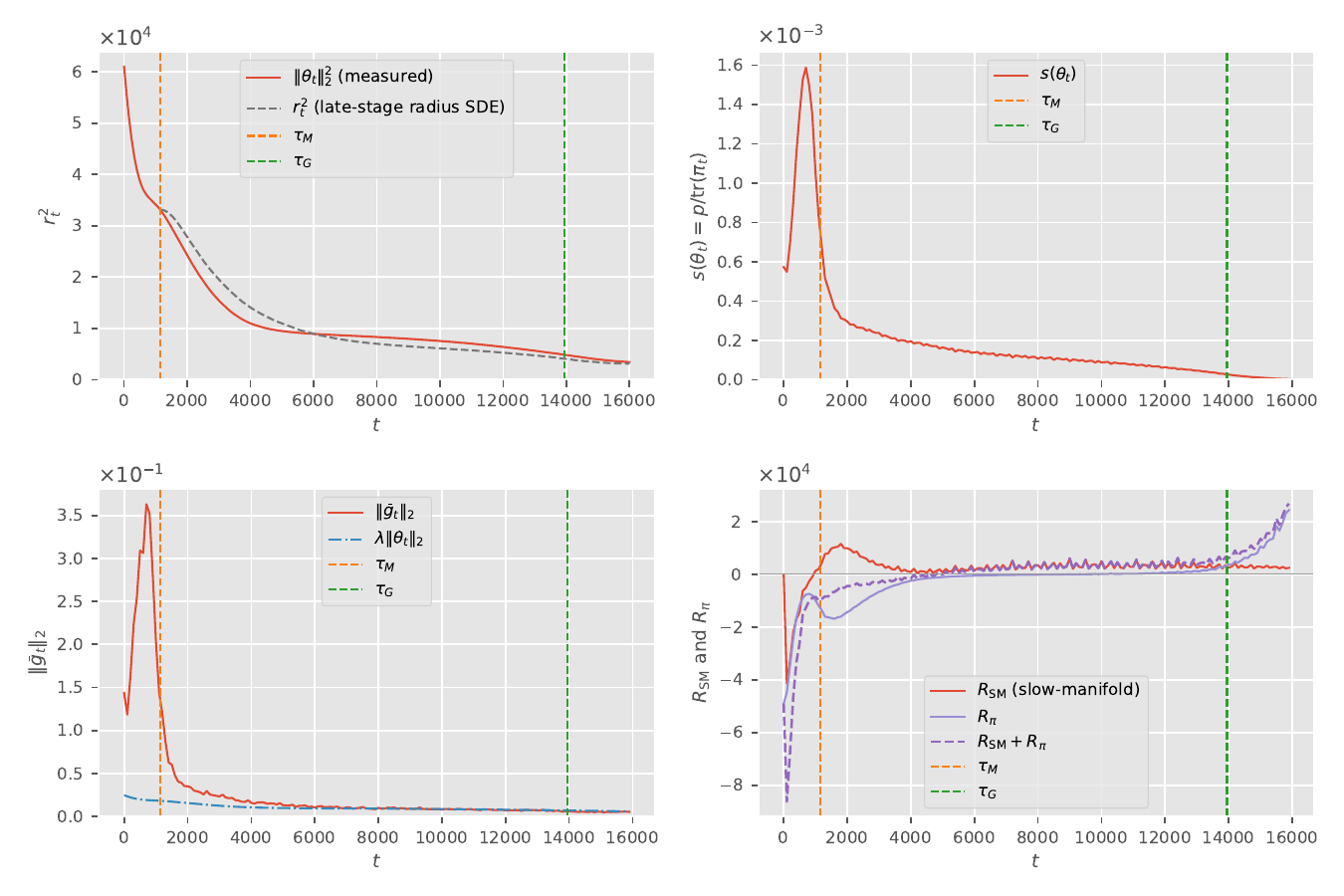}

    \caption{\textbf{Late-Stage Radius SDE.} This experiment shows the dynamics of the late-stage radius SDE in Lemma~\ref{lem:reduced_late_stage_radius_sde_restated}. The predicted radius dynamics closely match the theoretical values computed from Lemma~\ref{lem:reduced_late_stage_radius_sde_restated} without counting residual terms, validating the correctness of the decomposition. In particular, the late-stage residual sum $\mathcal{R}_{\rm SM}+\mathcal{R}_{\pi}$ is negligible, as hypothesized.}

    \label{fig:late_stage_radius_sde_appendix}
\end{figure*}

\input{proofs/proof_preconditioned_latestage_radius_SDE}

\subsection{Proof: Scaling Law of Memorization Radius}
\label{app:proof_memorization_radius_scaling_appendix}

\input{proofs/proof_rho_M}

\newpage
\subsection{Proof: Scaling Law of Generalization Radius}
\label{app:proof_generalization_radius_scaling_appendix}

\input{proofs/proof_rho_G}

\newpage
\subsection{Proof: Scaling Law of Solution Transition Time}
\label{app:proof_solution_transition_time_appendix}
\input{proofs/proof_tau}

\newpage
\subsection{Additional Results: Scaling Laws of Manifold Radius $\rho_M^2$ on $\mathbb{Z}_{127}$}
\label{app:additional_results_rho_M_z127_appendix}

\begin{figure*}[h]
    \centering

    \includegraphics[width=\linewidth]{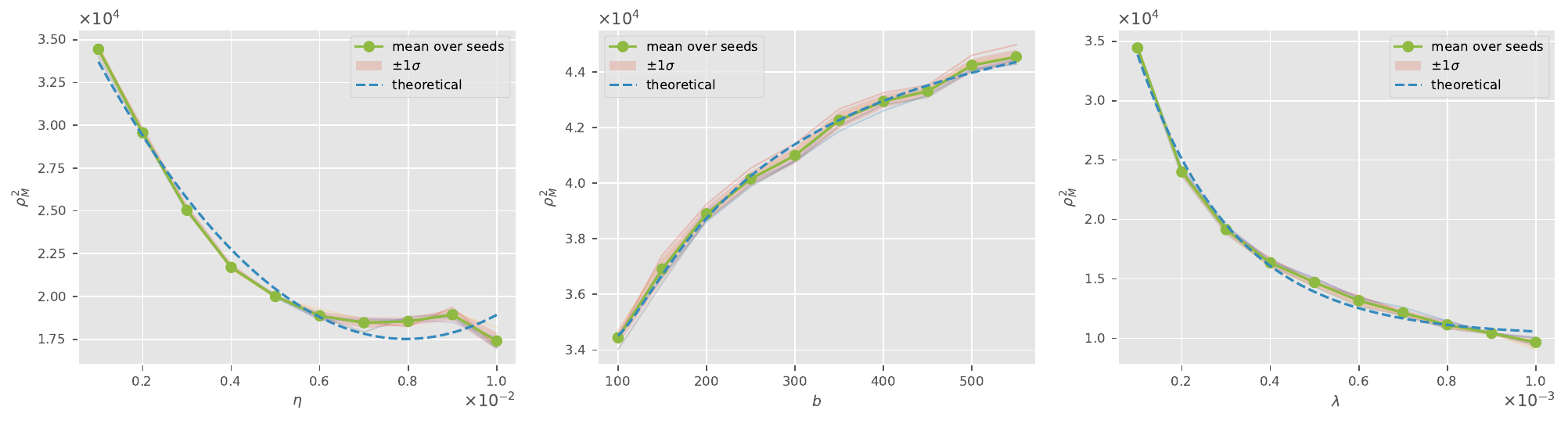}

    \caption{\textbf{Scaling Law of Manifold Radius $\rho_M^2$ on $\mathbb{Z}_{127}$.} We show the scaling law of $\rho_M^2$ with respect to the learning rate $\eta$, batch size $b$, and $\ell_2$ regularization coefficient $\lambda$ on the $\mathbb{Z}_{127}$ task. For each hyperparameter configuration, we train for ten runs. The results show that larger $\eta/b$ induces stronger diffusion variance, whereas $\lambda$ does not affect the diffusion variance. We also overlay the theoretical fits.}

    \label{fig:radius_scaling_rho_M_z127_appendix}
\end{figure*}

\newpage
\subsection{Additional Results: Scaling Laws of Manifold Radius $\rho_G^2$ on $\mathbb{Z}_{127}$}
\label{app:additional_results_rho_G_z127_appendix}

\begin{figure*}[h]
    \centering

    \includegraphics[width=\linewidth]{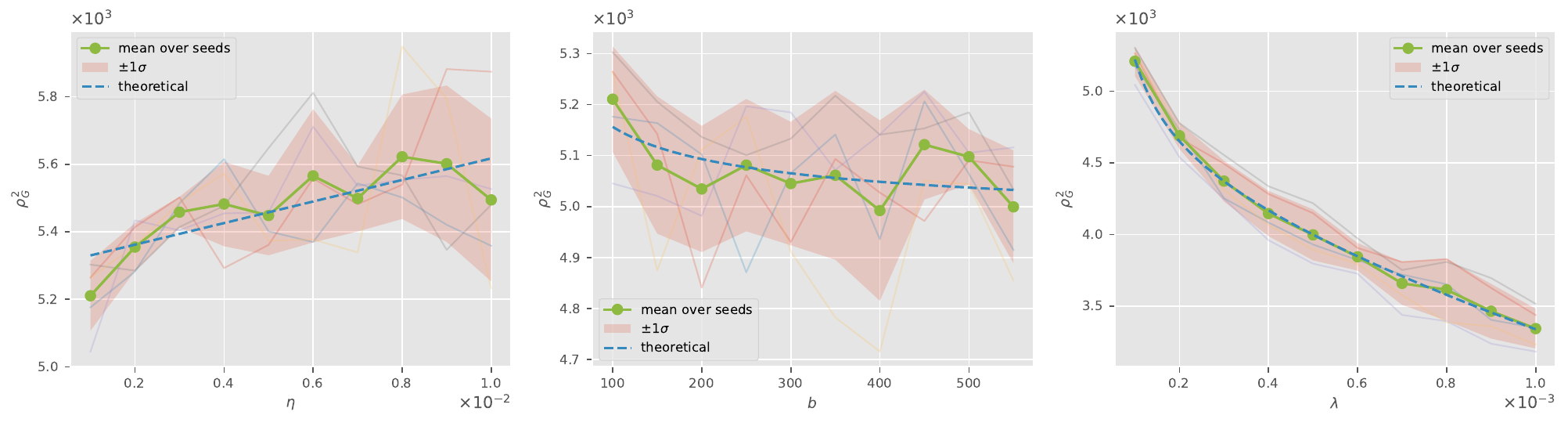}

    \caption{\textbf{Scaling Law of Manifold Radius $\rho_G^2$ on $\mathbb{Z}_{127}$.} We show the scaling law of $\rho_G^2$ with respect to the learning rate $\eta$, batch size $b$, and $\ell_2$ regularization coefficient $\lambda$ on the $\mathbb{Z}_{127}$ task. For each hyperparameter configuration, we train for ten runs. The results show that larger $\eta$ induces stronger diffusion variance, whereas $\lambda$ does not affect the diffusion variance. We also overlay the theoretical fits.}

    \label{fig:radius_scaling_rho_G_z127_appendix}
\end{figure*}

\newpage
\subsection{Additional Results: Scaling Laws of Solution Transition Time on $\mathbb{Z}_{127}$}
\label{app:additional_results_transition_time_z127_appendix}

\begin{figure*}[th]
    \centering

    \includegraphics[width=\linewidth]{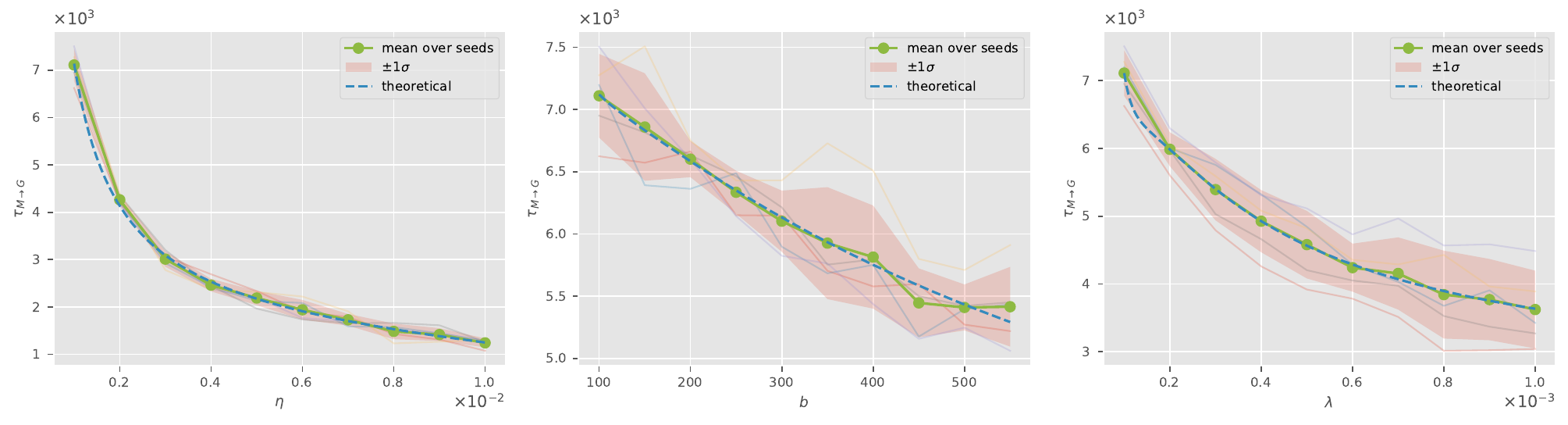}

    \caption{\textbf{Scaling laws of solution transition time on $\mathbb{Z}_{127}$.} We show that the solution transition time $\tau_{M \to G}$ from the memorization manifold $M$ to the generalization manifold $G$ scales with the learning rate $\eta$, batch size $b$, and $\ell_2$ regularization coefficient $\lambda$. For each hyperparameter configuration, we train for ten runs. We also overlay the theoretical fits.}

    \label{fig:transition_time_scaling_z127_appendix}
\end{figure*}



%% file: tikz/tikz_high_level_sketch.tex
\begin{tikzpicture}[
    x=1cm,
    y=1cm,
    every node/.style={font=\scriptsize},
    theoryflowroot/.style={
        rectangle,
        rounded corners=4pt,
        draw=black,
        line width=0.35pt,
        align=center,
        inner xsep=6pt,
        inner ysep=5pt,
        text width=10.8cm
    },
    theoryflowbox/.style={
        rectangle,
        rounded corners=4pt,
        draw=black,
        line width=0.35pt,
        align=center,
        inner xsep=5pt,
        inner ysep=5pt,
        text width=4.9cm
    },
    theoryflowwide/.style={
        rectangle,
        rounded corners=4pt,
        draw=black,
        line width=0.35pt,
        align=center,
        inner xsep=5pt,
        inner ysep=5pt,
        text width=10.2cm
    },
    theoryflownote/.style={
        rectangle,
        rounded corners=4pt,
        draw=black,
        dashed,
        line width=0.35pt,
        align=center,
        inner xsep=5pt,
        inner ysep=5pt,
        text width=5.6cm
    },
    theoryflowarrow/.style={
        ->,
        line width=0.45pt,
        shorten >=1pt,
        shorten <=1pt
    },
    theoryflowdasharrow/.style={
        ->,
        dashed,
        line width=0.45pt,
        shorten >=1pt,
        shorten <=1pt
    }
]

\node[theoryflowroot] (theoryflow_joint) at (0,0) {
\textbf{Adam's Joint-State SDE} (Lemma~\ref{lem:adam_joint_sde_restated})\\[1.0mm]
\(\displaystyle
\dd S_t
=
\mu(S_t)\dd t
+
\sqrt{\eta/b}\,\sigma(S_t)\dd W_t,
\qquad
S_t=(\theta_t,m_t,v_t)
\)
};

\node[theoryflowbox] (theoryflow_radius_induced) at (-5.8,-2.5) {
\textbf{Induced Radius SDE} (Lemma~\ref{lem:induced_radius_sde_restated})
\begin{align}
\dd r_t^2
=
-2B(t)\theta_t^\top
\left(m_t\oslash\sqrt{v_t}\right)\dd t \notag
\end{align}
};

\node[theoryflowbox] (theoryflow_mem_id) at (5.8,-3) {
\textbf{Memorization-Regime Identities} (Lemma~\ref{lem:late_stage_preconditioned_diffusion_isotropy})\\[1.0mm]
Late-stage identities
\begin{align}
\displaystyle 
&\mathrm{tr}(\pi_t) \approx \sqrt{b}\, \mathrm{tr}\!\bigl[\left(\mathrm{diag}\,\Sigma(\theta_t)\right)^{-1/2}\bigr],\nonumber\\
&\mathrm{tr}\!\left(G(\theta_t)G(\theta_t)^\top\right)
\approx\nonumber
p
\end{align}
};

\node[theoryflowbox] (theoryflow_theta_prec) at (-5.8,-5.95) {
\textbf{Preconditioned-Decomposition of Adam's $\theta$-SDE} (Lemma~\ref{lem:preconditioned_theta_sde_appendix})
\begin{align}
\dd\theta_t
&\approx
-a(t)\pi(\theta_t)\bar g(\theta_t)\dd t
+
\mathcal R_{\theta,\mathrm{SM}}(t)\dd t \notag\\
&\qquad\qquad +\sqrt{\eta}\,G(\theta_t)\dd W_t\notag
\end{align}
};

\node[theoryflowbox] (theoryflow_mem_scale) at (5.8,-4.95) {
\textbf{Scaling Law of Memorization Radius} (Theorem~\ref{thm:memorization_radius_scaling_restated})\\[1.0mm]
Joint-state first-exit analysis for
\(\displaystyle \rho_M^2\)
};

\node[theoryflowbox] (theoryflow_meanfield) at (-5.8,-4.15) {
\textbf{Mean-Field Limit of $(m_t,v_t)$} (Lemma~\ref{lem:meanfield_mt_and_vt_limit})\\[1.0mm]
\(\displaystyle m_t \to \bar g_t(1-e^{-\alpha_1 t}),\ v_t \to \overline{g\odot g}_t\)
};
\node[theoryflowwide] (theoryflow_radius_prec) at (0,-8.1) {
\textbf{Preconditioned-Decomposition of Adam's Radius SDE} (Lemma~\ref{lem:preconditioned_radius_sde_appendix})\\[1.0mm]
Deriving the preconditioned-decomposition of radius SDE for analyzing manifold radius precisely.
\begin{align}
\dd r_t^2
\approx
\left[
-2a(t)\theta_t^\top\pi(\theta_t)\bar g_t
+
\mathcal R_{\mathrm{SM}}(t)
+
\eta\,
\mathrm{tr}\!\left(G(\theta_t)G(\theta_t)^\top\right)
\right]\dd t
+
2\sqrt{\eta}\,\theta_t^\top G(\theta_t)\dd W_t \nonumber
\end{align}
};

\node[theoryflowwide] (theoryflow_reduced_radius) at (0,-10.25) {
\textbf{Reduced Memorization-Regime Radius SDE} (Lemma~\ref{lem:reduced_late_stage_radius_sde_restated})\\[1.0mm]
Use memorization-regime identities to reduce the preconditioned radius dynamics
\begin{align}
\dd r_t^2
\approx
\left[
-\frac{2\lambda}{s(\theta_t)}r_t^2
-
\frac{2}{s(\theta_t)}
\theta_t^\top\bigl(\bar g_t-\lambda\theta_t\bigr)
+
\eta\,p
+
\mathcal R_{\mathrm{SM}}(t)
+
\mathcal R_{\pi}(t)
\right]\dd t
+
2\sqrt{\eta}\,r_t\,\dd W_t^{(r)}
\nonumber
\end{align}
};

\node[theoryflowbox] (theoryflow_gen_scale) at (-4.45,-12.75) {
\textbf{Scaling Law of Generalization Radius} (Theorem~\ref{thm:generalization_radius_scaling_restated})\\[1.0mm]
Stationary late-stage radial analysis for
\(\displaystyle \rho_G^2\)
};

\node[theoryflowbox] (theoryflow_transition) at (4.45,-12.75) {
\textbf{Scaling Law of Solution Transition Time} (Theorem~\ref{thm:solution_transition_time_scaling_restated})\\[1.0mm]
First-passage analysis from
\(\displaystyle \partial M\)
to
\(\displaystyle \partial G\)
};


\draw[theoryflowarrow]
    (theoryflow_joint.south)
    -- (0,-1.25)
    -- (theoryflow_radius_prec.north);                       
\draw[theoryflowarrow]
    (-8.7,-4.15)
    -- (theoryflow_meanfield.west);                          
\draw[theoryflowarrow]
    (theoryflow_meanfield.east)
    -| ([xshift=-1.4cm]theoryflow_radius_prec.north);        

\draw[theoryflowarrow]
    (0,-1.25)
    -- (-5.8,-1.25)
    -- (theoryflow_radius_induced.north);                    

\draw[theoryflowarrow]
    (0,-1.25)
    -- (-8.7,-1.25)
    -- (-8.7,-5.95)
    -- (theoryflow_theta_prec.west);                         

\draw[theoryflowarrow]
    (0,-1.25)
    -- (5.8,-1.25)
    -- (theoryflow_mem_id.north);                            

\draw[theoryflowarrow]
    (0,-1.25)
    -- (2.0,-1.25)
    -- (2.0,-4.95)
    -- (theoryflow_mem_scale.west);                          

\draw[theoryflowarrow]
    (theoryflow_radius_induced.east)
    -| ([xshift=-0.7cm]theoryflow_radius_prec.north);        

\draw[theoryflowarrow]
    (theoryflow_theta_prec.south)
    -- (-5.8,-8.1)
    -- (theoryflow_radius_prec.west);                        

\draw[theoryflowarrow,-]
    (theoryflow_mem_id.east)
    -- (9.3,-3)
    -- (9.3,-10.25)
    -- (5.95,-10.25);                                        
\draw[theoryflowarrow]
    (5.65,-10.25)
    -- (theoryflow_reduced_radius.east);                     

\draw[theoryflowarrow]
    (theoryflow_radius_prec.south)
    -- (theoryflow_reduced_radius.north);                    

\draw[theoryflowarrow]
    (theoryflow_mem_scale.south)
    -- ([xshift=13.5mm]theoryflow_transition.north);         

\draw[theoryflowarrow]
    (theoryflow_reduced_radius.south)
    -- (0,-11.55)
    -- (-4.45,-11.55)
    -- (theoryflow_gen_scale.north);                         

\draw[theoryflowarrow]
    (0,-11.55)
    -- (4.45,-11.55)
    -- (theoryflow_transition.north);                        

\draw[theoryflowarrow]
    (theoryflow_gen_scale.east)
    -- (theoryflow_transition.west);                         

\fill (0,-1.25) circle (1.7pt);                              
\fill (-5.8,-1.25) circle (1.7pt);                           
\fill (2.0,-1.25) circle (1.7pt);                            
\fill (-8.7,-4.15) circle (1.7pt);                           
\fill (0,-11.55) circle (1.7pt);                             

\end{tikzpicture}

%% file: tables/tab_radius_concentration.tex
 
\begin{table}[h]
\centering
\small
\setlength{\tabcolsep}{4pt}
\caption{\textbf{Adam--Induced Shell--Core Radius \& Stopping-Time Concentration on Learning Task $S_5$.} The experiment shows the concentration of Adam--induced shell--core radii and stopping times on the learning task $S_5$, across 10 seeds, with learning rate $\eta=10^{-3},(\beta_1,\beta_2)=(0.9,0.999)$, $\ell_2$ regularization coefficient $10^{-4}$, and batch size $100$.}
\label{tab:radius_shell_core_s5}
\begin{tabular}{rrrrrrr}
\toprule
Seed ($S_0$) & $\rho_0(S_0)$ & $\rho_M(S_0)$ & $\tau_M(S_0)$ & $\rho_G(S_0)$ & $\tau_G(S_0)$ & $\tau_{M\to G}(S_0)$ \\
\midrule
1 & 249.27 & 186.07 & 1006 & 71.39 & 14000 & 12994 \\
2 & 247.09 & 184.59 & 1009 & 69.22 & 14850 & 13841 \\
3 & 247.87 & 185.19 & 1007 & 68.84 & 12150 & 11143 \\
4 & 248.13 & 185.16 & 1009 & 69.67 & 13950 & 12941 \\
5 & 248.85 & 185.64 & 1008 & 70.72 & 12500 & 11492 \\
6 & 248.44 & 186.61 & 939 & 69.34 & 12100 & 11161 \\
7 & 247.98 & 185.05 & 1009 & 69.50 & 15450 & 14441 \\
8 & 248.03 & 185.19 & 1009 & 68.33 & 13850 & 12841 \\
9 & 247.95 & 186.23 & 937 & 67.80 & 13050 & 12113 \\
10 & 248.88 & 185.59 & 1004 & 70.15 & 12750 & 11746 \\
\midrule
$\mathbb{E}[\cdot]$ & 248.25 & 185.53 & 993.7 & 69.50 & 13465.0 & 12471.3 \\
$\mathrm{Std}[\cdot]$ & 0.60 & 0.59 & 27.9 & 1.02 & 1083.3 & 1071.5 \\
\bottomrule
\end{tabular}

 
\bigskip

\setlength{\tabcolsep}{4pt}
\caption{\textbf{Adam--Induced Shell--Core Radius \& Stopping-Time Concentration on Learning Task $\mathbb{Z}_{127}$.} The experiment shows the concentration of Adam--induced shell--core radii and stopping times on the learning task $S_5$, across 10 seeds, with learning rate $\eta=10^{-3},(\beta_1,\beta_2)=(0.9,0.999)$, $\ell_2$ regularization coefficient $10^{-4}$, and batch size $100$.}
\label{tab:radius_shell_core_z127}
\begin{tabular}{rrrrrrr}
\toprule
Seed ($S_0$) & $\rho_0(S_0)$ & $\rho_M(S_0)$ & $\tau_M(S_0)$ & $\rho_G(S_0)$ & $\tau_G(S_0)$ & $\tau_{M\to G}(S_0)$ \\
\midrule
1 & 256.25 & 186.57 & 1041 & 71.47 & 8000 & 6959 \\
2 & 254.24 & 183.95 & 1115 & 72.40 & 8600 & 7485 \\
3 & 254.76 & 184.10 & 1121 & 71.35 & 8300 & 7179 \\
4 & 255.47 & 184.79 & 1118 & 71.32 & 8300 & 7182 \\
5 & 256.10 & 185.09 & 1119 & 70.61 & 8400 & 7281 \\
6 & 255.76 & 184.74 & 1121 & 71.06 & 8450 & 7329 \\
7 & 255.09 & 184.36 & 1116 & 72.32 & 8400 & 7284 \\
8 & 255.45 & 184.88 & 1114 & 72.21 & 8500 & 7386 \\
9 & 255.46 & 185.97 & 1047 & 72.66 & 7650 & 6603 \\
10 & 256.06 & 186.53 & 1042 & 73.44 & 7800 & 6758 \\
\midrule
$\mathbb{E}[\cdot]$ & 255.46 & 185.10 & 1095.4 & 71.88 & 8240.0 & 7144.6 \\
$\mathrm{Std}[\cdot]$ & 0.60 & 0.90 & 34.2 & 0.82 & 299.8 & 269.4 \\
\bottomrule
\end{tabular}
\end{table}

%% file: proofs/proof_concentration_normal.tex
\begin{lemma}[Gaussian Initialization Concentration]
\label{lem:gaussian_initialization_concentration}
Let
\begin{align}
\theta_0
=
\left(
\theta_0^{(1)},\ldots,\theta_0^{(k)}
\right)
\in\mathbb R^p,
\qquad
\theta_0^{(j)}\in\mathbb R^{p_j},
\qquad
\sum_{j=1}^k p_j=p,
\end{align}
where the subvectors are independent and
$[\theta_0^{(j)}]_i\sim\mathcal N(0,\sigma_j^2)$. Define the effective
coordinate variance $\tilde{\sigma}^2$ and the effective dimension parameter
$\tilde p$ by
\begin{align}
\tilde{\sigma}^2
:=
\frac{1}{p}
\sum_{j=1}^k p_j\sigma_j^2,
\qquad
\frac{1}{\tilde p}
:=
\frac{1}{k}
\sum_{j=1}^k \frac{1}{p_j}.
\end{align}
Then the squared initialization radius satisfies
\begin{align}
\mathbb{E}\bigl[\|\theta_0\|_2^2\bigr]
=
\tilde{\sigma}^2 p,
\qquad
\mathrm{Var}\bigl[\|\theta_0\|_2^2\bigr]
=
2\sum_{j=1}^k p_j\sigma_j^4.
\end{align}
In particular, for initialization schemes used for numerical stability, each
block-wise coordinate scale is often chosen so that
$\sigma_j=O(p_j^{-1/2})$. Under this scaling,
\begin{align}
\tilde{\sigma}^2
=
O\!\left(\frac{k}{p}\right),
\qquad
\mathbb{E}\bigl[\|\theta_0\|_2^2\bigr]
=
O(k),
\qquad
\mathrm{Var}\bigl[\|\theta_0\|_2^2\bigr]
=
O\!\left(\frac{k}{\tilde p}\right).
\end{align}
\end{lemma}

\begin{proof}
We consider a Gaussian initialization with block-wise coordinate scales. Write
\begin{align}
\theta_0
=
\left(
\theta_0^{(1)},\ldots,\theta_0^{(k)}
\right),
\qquad
\theta_0^{(j)}
=
\sigma_j Z^{(j)},
\qquad
Z^{(j)}
\sim
\mathcal N(0,I_{p_j}),
\end{align}
where the subvectors $Z^{(j)}$ are independent. Then
\begin{align}
\|\theta_0\|_2^2
=
\sum_{j=1}^k
\|\theta_0^{(j)}\|_2^2
=
\sum_{j=1}^k
\sigma_j^2\|Z^{(j)}\|_2^2,
\end{align}
and it suffices to study the distribution of each $\|Z^{(j)}\|_2^2$. Let
\begin{align}
R_j
:=
\|Z^{(j)}\|_2^2 .
\end{align}
The random variable $R_j$ follows a chi-square distribution with $p_j$ degrees of freedom,
whose density is
\begin{align}
f_{R_j}(r)
=
\frac{1}{2^{\frac{p_j}{2}}\Gamma\!\left(\frac{p_j}{2}\right)}
r^{\frac{p_j}{2}-1}e^{-r/2},
\qquad r\ge 0 .
\end{align}

\paragraph{Computing $\mathbb{E}\bigl[\|\theta_0\|_2^2\bigr]$.}
The expected squared Euclidean norm of the initialization is therefore
\begin{align}
\mathbb{E}\bigl[\|\theta_0\|_2^2\bigr]
&=
\sum_{j=1}^k
\sigma_j^2\,\mathbb{E}[R_j]
\\
&=
\sum_{j=1}^k
\sigma_j^2
\int_0^\infty r f_{R_j}(r)\,\dd r
\\
&=
\sum_{j=1}^k
\frac{\sigma_j^2}{2^{\frac{p_j}{2}}\Gamma\!\left(\frac{p_j}{2}\right)}
\int_0^\infty r^{\frac{p_j}{2}}e^{-r/2}\,\dd r .
\end{align}
Applying the change of variables $u=r/2$ yields
\begin{align}
\int_0^\infty r^{\frac{p_j}{2}}e^{-r/2}\,\dd r
=
2^{\frac{p_j}{2}+1}
\Gamma\!\left(\frac{p_j}{2}+1\right).
\end{align}
Substituting back, we obtain
\begin{align}
\mathbb{E}\bigl[\|\theta_0\|_2^2\bigr]
&=
\sum_{j=1}^k
2\sigma_j^2
\frac{\Gamma\!\left(\frac{p_j}{2}+1\right)}
{\Gamma\!\left(\frac{p_j}{2}\right)}
\\
&=
\sum_{j=1}^k
p_j\sigma_j^2
\\
&=
\tilde{\sigma}^2p .
\end{align}

\paragraph{Computing $\mathrm{Var}\bigl[\|\theta_0\|_2^2\bigr]$.}
We next compute the fluctuations of the squared initialization radius. Since the
subvectors are independent and
$\|\theta_0\|_2^2=\sum_{j=1}^k\sigma_j^2R_j$, it follows that
\begin{align}
\mathrm{Var}\bigl[\|\theta_0\|_2^2\bigr]
=
\sum_{j=1}^k
\sigma_j^4\mathrm{Var}[R_j].
\end{align}
To compute $\mathrm{Var}[R_j]$, we first compute the second moment:
\begin{align}
\mathbb{E}[R_j^2]
&=
\int_0^\infty r^2 f_{R_j}(r)\,\dd r
\\
&=
\frac{1}{2^{\frac{p_j}{2}}\Gamma\!\left(\frac{p_j}{2}\right)}
\int_0^\infty r^{\frac{p_j}{2}+1}e^{-r/2}\,\dd r .
\end{align}
Applying again the change of variables $u=r/2$ gives
\begin{align}
\int_0^\infty r^{\frac{p_j}{2}+1}e^{-r/2}\,\dd r
=
2^{\frac{p_j}{2}+2}
\Gamma\!\left(\frac{p_j}{2}+2\right).
\end{align}
Thus,
\begin{align}
\mathbb{E}[R_j^2]
&=
4
\frac{\Gamma\!\left(\frac{p_j}{2}+2\right)}
{\Gamma\!\left(\frac{p_j}{2}\right)}
\\
&=
4
\left(\frac{p_j}{2}+1\right)
\left(\frac{p_j}{2}\right)
\\
&=
p_j(p_j+2).
\end{align}
Therefore,
\begin{align}
\mathrm{Var}[R_j]
&=
\mathbb{E}[R_j^2] - \bigl(\mathbb{E}[R_j]\bigr)^2
\\
&=
p_j(p_j+2)-p_j^2
\\
&=
2p_j .
\end{align}
Consequently,
\begin{align}
\mathrm{Var}\bigl[\|\theta_0\|_2^2\bigr]
&=
\sum_{j=1}^k
\sigma_j^4\mathrm{Var}[R_j]
\\
&=
2\sum_{j=1}^k p_j\sigma_j^4 .
\end{align}

\paragraph{Coordinate Scaling.}
Finally, in common initialization schemes used for numerical stability, each
block-wise per-coordinate variance is scaled inversely with its block dimension.
Equivalently, one may write
\begin{align}
\sigma_j
=
O\!\left(\frac{1}{\sqrt{p_j}}\right),
\qquad
\sigma_j^2
=
O\!\left(\frac{1}{p_j}\right).
\end{align}
Substituting this scaling into the definition of the effective coordinate variance gives
\begin{align}
\tilde{\sigma}^2
=
\frac{1}{p}
\sum_{j=1}^k p_j\sigma_j^2
=
O\!\left(\frac{k}{p}\right).
\end{align}
Therefore,
\begin{align}
\mathbb{E}\bigl[\|\theta_0\|_2^2\bigr]
&=
\sum_{j=1}^k p_j\sigma_j^2
=
\tilde{\sigma}^2p
=
O(k),
\\
\mathrm{Var}\bigl[\|\theta_0\|_2^2\bigr]
&=
2\sum_{j=1}^k p_j\sigma_j^4
=
O\!\left(\sum_{j=1}^k\frac{1}{p_j}\right).
\end{align}
Using the definition
\begin{align}
\frac{1}{\tilde p}
=
\frac{1}{k}
\sum_{j=1}^k\frac{1}{p_j},
\end{align}
we obtain
\begin{align}
\mathrm{Var}\bigl[\|\theta_0\|_2^2\bigr]
=
O\!\left(\frac{k}{\tilde p}\right).
\end{align}
\end{proof}

%% file: proofs/proof_concentration_uniform.tex
\begin{lemma}[Uniform Initialization Concentration]
\label{lem:uniform_initialization_concentration}
Let
\begin{align}
\theta_0
=
\left(
\theta_0^{(1)},\ldots,\theta_0^{(k)}
\right)
\in\mathbb R^p,
\qquad
\theta_0^{(j)}\in\mathbb R^{p_j},
\qquad
\sum_{j=1}^k p_j=p,
\end{align}
where the subvectors are independent and the coordinates of each subvector satisfy
$[\theta_0^{(j)}]_i\sim \mathcal{U}(-\varepsilon_j,\varepsilon_j)$. Define the
effective coordinate variance $\tilde{\varepsilon}^2$ and the effective dimension
parameter $\tilde p$ by
\begin{align}
\tilde{\varepsilon}^2
:=
\frac{1}{p}
\sum_{j=1}^k p_j\varepsilon_j^2,
\qquad
\frac{1}{\tilde p}
:=
\frac{1}{k}
\sum_{j=1}^k \frac{1}{p_j}.
\end{align}
Then the squared initialization radius satisfies
\begin{align}
\mathbb{E}[\|\theta_0\|_2^2]
=
\frac{\tilde{\varepsilon}^2p}{3},
\qquad
\mathrm{Var}(\|\theta_0\|_2^2)
=
\frac{4}{45}\sum_{j=1}^k p_j\varepsilon_j^4.
\end{align}
Moreover, by the central limit theorem,
\begin{align}
\|\theta_0\|_2^2
\;\overset{d}{\approx}\;
\mathcal{N}\!\left(
\frac{\tilde{\varepsilon}^2p}{3},
\frac{4}{45}\sum_{j=1}^k p_j\varepsilon_j^4
\right),
\qquad p\to\infty .
\end{align}
In particular, for initialization schemes with block-wise coordinate half-widths
$\varepsilon_j=O(p_j^{-1/2})$, the squared initialization radius satisfies
\begin{align}
\tilde{\varepsilon}^2
=
O\!\left(\frac{k}{p}\right),
\qquad
\mathbb{E}[\|\theta_0\|_2^2]
=
O(k),
\qquad
\mathrm{Var}(\|\theta_0\|_2^2)
=
O\!\left(\frac{k}{\tilde p}\right).
\end{align}
For example, a common numerically stable block-wise choice is
$\varepsilon_j=p_j^{-1/2}$, in which case
\begin{align}
\|\theta_0\|_2^2
\;\overset{d}{\approx}\;
\mathcal{N}\!\left(
\frac{k}{3},
\frac{4k}{45\tilde p}
\right).
\end{align}
Thus uniform initialization concentrates on a thin hyperspherical shell with
squared radius $k/3$ and squared-radius variance $4k/(45\tilde p)$.
\end{lemma}

\begin{proof}
For each coordinate $X_i^{(j)} := [\theta_0^{(j)}]_i \sim \mathcal{U}(-\varepsilon_j,\varepsilon_j)$,
the density of $X_i^{(j)}$ is $f_{X_i^{(j)}}(x)=\frac{1}{2\varepsilon_j}$ for
$x \in [-\varepsilon_j,\varepsilon_j]$. Then the second moment is computed as
\begin{align}
\mathbb{E}[(X_i^{(j)})^2]
=
\int_{-\varepsilon_j}^{\varepsilon_j} x^2 f_{X_i^{(j)}}(x)\,\dd x
=
\frac{1}{2\varepsilon_j}\int_{-\varepsilon_j}^{\varepsilon_j} x^2\,\dd x
=
\frac{\varepsilon_j^2}{3}.
\end{align}
The fourth moment is computed as
\begin{align}
\mathbb{E}[(X_i^{(j)})^4]
=
\int_{-\varepsilon_j}^{\varepsilon_j} x^4 f_{X_i^{(j)}}(x)\,\dd x
=
\frac{1}{2\varepsilon_j}\int_{-\varepsilon_j}^{\varepsilon_j} x^4\,\dd x
=
\frac{\varepsilon_j^4}{5}.
\end{align}

Defining $S_i^{(j)} := (X_i^{(j)})^2$, we have
$\mathbb{E}[S_i^{(j)}]=\mathbb{E}[(X_i^{(j)})^2]$ and
$\mathbb{E}[(S_i^{(j)})^2]=\mathbb{E}[(X_i^{(j)})^4]$, so that
\begin{align}
\mathrm{Var}((X_i^{(j)})^2)
=
\mathrm{Var}(S_i^{(j)})
&=
\mathbb{E}[(S_i^{(j)})^2] - \mathbb{E}[S_i^{(j)}]^2
\\
&=
\mathbb{E}[(X_i^{(j)})^4] - \bigl(\mathbb{E}[(X_i^{(j)})^2]\bigr)^2
\\
&=
\frac{\varepsilon_j^4}{5}
-
\left(\frac{\varepsilon_j^2}{3}\right)^2
\\
&=
\frac{4\varepsilon_j^4}{45}.
\end{align}

Now set
\begin{align}
R^2
=
\sum_{j=1}^k\sum_{i=1}^{p_j} (X_i^{(j)})^2
=
\sum_{j=1}^k\sum_{i=1}^{p_j} S_i^{(j)} .
\end{align}
Since the random variables $S_i^{(j)}$ are independent, we obtain
\begin{align}
\mathbb{E}[R^2]
&=
\sum_{j=1}^k p_j\,\mathbb{E}[S_i^{(j)}]
=
\frac{1}{3}\sum_{j=1}^k p_j\varepsilon_j^2
=
\frac{\tilde{\varepsilon}^2p}{3},
\\
\mathrm{Var}(R^2)
&=
\sum_{j=1}^k p_j\,\mathrm{Var}(S_i^{(j)})
=
\frac{4}{45}\sum_{j=1}^k p_j\varepsilon_j^4.
\end{align}
Moreover, by the central limit theorem,
\begin{align}
R^2
\overset{d}{\approx}
\mathcal{N}\!\left(
\sum_{j=1}^k p_j\,\mathbb{E}[S_i^{(j)}],
\sum_{j=1}^k p_j\,\mathrm{Var}(S_i^{(j)})
\right)
=
\mathcal{N}\!\left(
\frac{\tilde{\varepsilon}^2p}{3},
\frac{4}{45}\sum_{j=1}^k p_j\varepsilon_j^4
\right),
\qquad p\to\infty .
\end{align}

For numerically stable uniform initialization, the coordinate half-width of each block is often
chosen to scale as
\begin{align}
\varepsilon_j
=
O\!\left(\frac{1}{\sqrt{p_j}}\right).
\end{align}
Substituting this scaling into the preceding identities gives
\begin{align}
\tilde{\varepsilon}^2
=
\frac{1}{p}
\sum_{j=1}^k p_j\varepsilon_j^2
=
O\!\left(\frac{k}{p}\right),
\end{align}
and therefore
\begin{align}
\mathbb{E}[\|\theta_0\|_2^2]
=
\frac{\tilde{\varepsilon}^2p}{3}
=
O(k),
\qquad
\mathrm{Var}(\|\theta_0\|_2^2)
=
\frac{4}{45}\sum_{j=1}^k p_j\varepsilon_j^4
=
O\!\left(\sum_{j=1}^k\frac{1}{p_j}\right).
\end{align}
Using the definition
\begin{align}
\frac{1}{\tilde p}
=
\frac{1}{k}
\sum_{j=1}^k \frac{1}{p_j},
\end{align}
we obtain
\begin{align}
\mathrm{Var}(\|\theta_0\|_2^2)
=
O\!\left(\frac{k}{\tilde p}\right).
\end{align}
For example, taking the common block-wise choice
\begin{align}
\varepsilon_j
=
\frac{1}{\sqrt{p_j}},
\end{align}
we obtain
\begin{align}
R^2
\overset{d}{\approx}
\mathcal{N}\!\left(
\frac{k}{3},
\frac{4k}{45\tilde p}
\right).
\end{align}
\end{proof}

%% file: tikz/tikz_proof_sketch_adam_sde_limit.tex
\begin{tikzpicture}[
    x=1cm,
    y=1cm,
    every node/.style={font=\scriptsize},
    adamflowwide/.style={
        rectangle,
        rounded corners=4pt,
        draw=black,
        line width=0.35pt,
        align=center,
        inner xsep=5pt,
        inner ysep=4pt,
        text width=12.6cm
    },
    adamflowmid/.style={
        rectangle,
        rounded corners=4pt,
        draw=black,
        line width=0.35pt,
        align=center,
        inner xsep=5pt,
        inner ysep=4pt,
        text width=8.3cm
    },
    adamflowsmall/.style={
        rectangle,
        rounded corners=4pt,
        draw=black,
        line width=0.35pt,
        align=center,
        inner xsep=4pt,
        inner ysep=4pt,
        text width=3.75cm
    },
    adamflownote/.style={
        rectangle,
        rounded corners=4pt,
        draw=black,
        dashed,
        line width=0.35pt,
        align=center,
        inner xsep=4pt,
        inner ysep=4pt,
        text width=7.4cm
    },
    adamflowarrow/.style={
        ->,
        line width=0.45pt
    },
    adamflowdasharrow/.style={
        ->,
        dashed,
        line width=0.45pt
    }
]

\node[adamflowwide] (adamflow_wiener) at (0,0) {
\textbf{Modeling Mini-Batch Gradient as Wiener Process by Central-Limit Theorem}\\[1.5mm]
\(\displaystyle
g_k
=
\nabla \mathscr{L}_f(\theta_k)
-
\frac{1}{\sqrt b}\Sigma(\theta_k)^{1/2}\zeta_k,
\qquad
\zeta_k
\sim \mathcal{N}(0, I_p) \in \mathbb{R}^p
\)
};

\node[adamflowwide] (adamflow_discrete) at (0,-2) {
\textbf{Adam's Discrete Update Rule}\\[1.5mm]
\(\displaystyle
\begin{aligned}
m_{k+1}
&=
\beta_1m_k+(1-\beta_1)g_k,
&
v_{k+1}
&=
\beta_2v_k+(1-\beta_2)(g_k\odot g_k),
\\
\widehat m_{k+1}
&=
\frac{m_{k+1}}{1-\beta_1^{k+1}},
&
\widehat v_{k+1}
&=
\frac{v_{k+1}}{1-\beta_2^{k+1}},
\\
\theta_{k+1}
&=
\theta_k
-
\eta\,\widehat m_{k+1}\oslash\sqrt{\widehat v_{k+1}} .
\end{aligned}
\)
};

\node[adamflowmid] (adamflow_interp) at (0,-4) {
\textbf{Continuous-Time Interpolation (Discrete--Continuous Scaling)}\\[1.5mm]
\(\displaystyle
t=\eta k,
\qquad
\beta_i=1-\eta\alpha_i,
\qquad
\eta\to0
\)
};

\node[adamflowsmall] (adamflow_mlimit) at (-4.55,-6.55) {
\textbf{Continuous-Time Limit of First-Moment Estimate}\\[1.5mm]
\(\displaystyle
\begin{aligned}
\dd m_t
&=
-\alpha_1(m_t-\bar g_t)\dd t
\\
&\quad+
\sqrt{\eta/b}\,
\alpha_1\Sigma_t^{1/2}\dd W_t .
\end{aligned}
\)
};

\node[adamflowsmall] (adamflow_vlimit) at (0,-6.55) {
\textbf{Continuous-Time Limit of Second-Moment Estimate}\\[1.5mm]
\(\displaystyle
\begin{aligned}
\dd v_t
&=
-\alpha_2(v_t-\overline{g\odot g}_t)\dd t
\\
&\quad+
\sqrt{\eta/b}\,
\alpha_2D_t\dd W_t .
\end{aligned}
\)
};

\node[adamflowsmall] (adamflow_square) at (4.55,-6.55) {
\textbf{Expansion of Element-Wise Squared Gradient}\\[1.5mm]
\(\displaystyle
\begin{aligned}
\overline{g\odot g}_t
&:=
\bar g_t\odot\bar g_t
+
\frac{1}{b}\mathrm{diag}\,\Sigma_t, \\
D_t
&:=
2\,\mathrm{diag}(\bar g_t)\Sigma_t^{1/2}.
\end{aligned}
\)
};

\node[adamflowmid] (adamflow_theta) at (0,-9.45) {
\textbf{Continuous-Time Limit of Parameter Dynamics}\\[1.5mm]
\(\displaystyle
\begin{aligned}
\dd\theta_t
&=
-B(t)\left(m_t\oslash\sqrt{v_t}\right)\dd t,
\\[1.5mm]
\text{where}~B(t)
&=
\frac{\sqrt{1-e^{-\alpha_2t}}}{1-e^{-\alpha_1t}} .
\end{aligned}
\)
};

\node[adamflowwide] (adamflow_joint) at (0,-12.20) {
\textbf{Joint-State SDE}\\[1.5mm]
\(\displaystyle
\begin{aligned}
\dd S_t
&=
\mu(S_t)\dd t
+
\sqrt{\eta/b}\,\sigma(S_t)\dd W_t,
\qquad
S_t=(\theta_t,m_t,v_t),
\\
\mu(S_t)
&=
\begin{pmatrix}
-B(t)m_t\oslash\sqrt{v_t}\\
-\alpha_1(m_t-\bar g_t)\\
-\alpha_2(v_t-\overline{g\odot g}_t)
\end{pmatrix},
\qquad
\sigma(S_t)
=
\begin{pmatrix}
0\\
\alpha_1\Sigma_t^{1/2}\\
\alpha_2D_t
\end{pmatrix}.
\end{aligned}
\)
};

\node[adamflowmid] (adamflow_radius) at (0,-14.65) {
\textbf{Sanity Check via Lemma~\ref{lem:induced_radius_sde_restated}:\nameref{lem:induced_radius_sde_restated} (Derived via It\^o's Lemma)}\\[1.5mm]
\(\displaystyle
\dd r_t^2
=
-2B(t)\theta_t^\top
\left(m_t\oslash\sqrt{v_t}\right)\dd t  .
\)
};

\node[adamflowmid] (adamflow_verify) at (0,-16.45) {
\textbf{Sanity Check via Numerical Experiment in Figure~\ref{fig:radius_sde_sanity_appendix}}\\[1.5mm]
\(\displaystyle
\text{numerically compare the radius dynamics predicted by the induced radius SDE.}
\)
};

\node[adamflownote] (adamflow_preconditioned) at (0,-18.25) {
\textbf{Additional Technical Correctness Verification}\\[1.5mm]
Technical correctness of Adam's joint-state SDE is also checked by the later decomposed preconditioned radius SDE in Lemma~\ref{lem:preconditioned_radius_sde_appendix} and an approximated form in Lemma~\ref{lem:reduced_late_stage_radius_sde_restated}. Further numerical verification experiments are in Figures~\ref{fig:preconditioned_radius_sde_appendix} and \ref{fig:late_stage_radius_sde_appendix}. 
};

\coordinate (adamflow_split_interp) at (0,-5.05);
\coordinate (adamflow_join_theta) at (0,-8.25);
\coordinate (adamflow_join_joint) at (0,-10.85);
\coordinate (adamflow_square_drop) at (4.55,-10.85);

\draw[adamflowarrow] (adamflow_wiener) -- (adamflow_discrete);
\draw[adamflowarrow] (adamflow_discrete) -- (adamflow_interp);

\draw[adamflowarrow] (adamflow_interp.south) -- (adamflow_split_interp);
\draw[adamflowarrow] (adamflow_split_interp) -| (adamflow_mlimit.north);
\draw[adamflowarrow] (adamflow_split_interp) -- (adamflow_vlimit.north);
\draw[adamflowarrow] (adamflow_split_interp) -| (adamflow_square.north);

\draw[adamflowarrow] (adamflow_mlimit.south) |- (adamflow_join_theta) -- (adamflow_theta.north);
\draw[adamflowarrow] (adamflow_vlimit.south) -- (adamflow_theta.north);

\draw[adamflowarrow] (adamflow_theta.south) -- (adamflow_join_joint) -- (adamflow_joint.north);

\draw[adamflowarrow]
    (adamflow_square.south)
    -- (adamflow_square_drop)
    -- (adamflow_join_joint);

\draw[adamflowarrow] (adamflow_joint) -- (adamflow_radius);
\draw[adamflowarrow] (adamflow_radius) -- (adamflow_verify);
\draw[adamflowdasharrow] (adamflow_verify) -- (adamflow_preconditioned);

\end{tikzpicture}

%% file: proofs/proof_adam_joint_state_SDE.tex
\begin{lemma}[Adam's Joint-State Continuous-Time SDE Limit]
\label{lem:adam_joint_sde_restated}

For a network $f$ parameterized with $\theta\in \mathbb{R}^p$, let $\ell_f(s;\theta)$ be the loss for sample $s$, let $\mathscr{L}_f(\xi;\theta)$ be the mini-batch loss for batch $\xi$, and let $\mathscr{L}_f(\theta)$ be the population loss over true dataset distribution. Let $\theta_k,\theta_t$ denote the parameter at iteration $k$ and time $t$, respectively. 

\paragraph{Discrete-Time Update Rules.} Consider Adam with small learning rate $\eta$, large batch size $b$, and coefficients $(\beta_1,\beta_2)$. Let
\begin{align}
S_k:=(\theta_k^\top,m_k^\top,v_k^\top)^\top\in\mathbb R^{3p},
\quad m_0 = 0_p, \quad v_0 = 0_p,
\end{align}
denote the discrete-time representation of the parameter, first-moment, and second-moment estimates at iteration $k$. For a batch $\xi_k$ at step $k$, the discrete Adam update rules with the bias-correction step are:
\begin{align}
g_k &= \nabla \mathscr{L}_f(\xi_k;\theta_k) &(\text{mini-batch gradient})\\
m_{k+1} &= \beta_1 m_k + (1-\beta_1)g_k  &(\text{first-moment estimate}) \\
v_{k+1} &= \beta_2 v_k + (1-\beta_2)(g_k \odot g_k) &(\text{second-moment estimate}) \\
\hat m_{k+1} &= \frac{m_{k+1}}{1-\beta_1^{\,k+1}} &(\text{bias-corrected first-moment estimate}) \\
\hat v_{k+1} &= \frac{v_{k+1}}{1-\beta_2^{\,k+1}} &(\text{bias-corrected second-moment estimate}) \\
\theta_{k+1} &= \theta_k - \eta \, h(\hat m_{k+1}, \hat v_{k+1}) &(\text{bias-corrected parameter update}) 
,
\end{align}
where $h(\hat m_{k+1}, \hat v_{k+1})$ is the per-coordinate Adam update direction, defined as
\begin{equation}
h(\hat m, \hat v) \;=\;
\hat m \oslash (\sqrt{\hat v} + \epsilon\mathbf{1}_p) \;=\; \begin{bmatrix}\dfrac{\hat m^{(1)}}{\sqrt{\hat v^{(1)}} + \epsilon}\\[4pt]\vdots\\\dfrac{\hat m^{(p)}}{\sqrt{\hat v^{(p)}} + \epsilon}\end{bmatrix}
,
\end{equation}
$\oslash$ denotes element-wise division, $\odot$ denotes element-wise product, and $\varepsilon$ is a small positive constant for numerical stability.

\paragraph{Distributional Limit of Mini-Batch Gradient.} 
For the parameter $\theta_k$, the mini-batch gradient $g_k:= \nabla \mathscr{L}_f(\xi_k;\theta_k)$ is an empirical average of $b$ i.i.d. per-sample gradient $g_k(s):=\nabla \ell_f(s;\theta_k)$. Hence, by the central limit theorem, as batch size $b\to\infty$
\begin{align}
  g_k = \frac{1}{b}\sum_{i=1}^b   g_k(s_i)
  \overset{d}{\longrightarrow} \mathcal N\!\left(\bar g_k,\frac{1}{b}\Sigma_k\right),
  \quad
  \bar g_k:=\mathbb{E}_s[g_k(s)], \quad  \Sigma_k:=\mathrm{Cov}_s[g_k(s)]
,
\end{align}
where $\bar g_k$ and $\Sigma_k$ are the per-sample gradient expectation and covariance at step $k$, respectively. 

\paragraph{Continuous-Time SDE Limit.} 
Under the continuous-time interpolation $t=\eta k$, with the conditions of sufficiently small learning rate $\eta \to 0$ and large batch size $b \to \infty$, the discrete Adam updates admit the It\^o SDE limit
\begin{align}
\dd S_t
=
\mu(S_t)\dd t
+
\sqrt{\frac{\eta}{b}}\,
\sigma(S_t)\dd W_t,
\label{eq:adam_joint_sde_restated}
\end{align}
where $W_t\in\mathbb R^p$ is a Wiener process adapted to the filtration generated by the mini-batch sampling process $\{\xi_t\}$. The drift and diffusion factors are
\begin{align}
\mu(S_t)
=
\begin{pmatrix}
-\,B(t)\,m_t\oslash\sqrt{v_t}
\\[2pt]
-\,\alpha_1\bigl(m_t-\bar g_t\bigr)
\\[2pt]
-\,\alpha_2\bigl(v_t-\overline{g\odot g}_t\bigr)
\end{pmatrix},
\qquad
\sigma(S_t)
=
\begin{pmatrix}
0
\\[2pt]
\alpha_1\Sigma_t^{1/2}
\\[2pt]
\alpha_2D_t
\end{pmatrix},
\label{eq:adam_joint_drift_diffusion_restated}
\end{align}
with
\begin{align}
B(t)
:=
\frac{\sqrt{1-e^{-\alpha_2 t}}}
{1-e^{-\alpha_1 t}}, 
\qquad
\alpha_i:=\frac{1-\beta_i}{\eta},
\qquad i=1,2,
\label{eq:bias_correction_factor_restated}
\end{align}
and
\begin{align}
\overline{g\odot g}_t
:=
\bar g_t\odot\bar g_t
+
\frac{1}{b}\operatorname{diag}(\Sigma_t),
\qquad
D_t
:=
2\,\operatorname{diag}(\bar g_t)\Sigma_t^{1/2}.
\label{eq:diffusion_cross_term_def_restated}
\end{align}
\end{lemma}

\begin{proof}
We derive the continuous-time limit of the Adam optimizer starting from its discrete update rules. For a network $f$ parameterized with $\theta\in \mathbb{R}^p$, let $\ell_f(s;\theta)$ be the loss for sample $s$, let $\mathscr{L}_f(\xi;\theta)$ be the mini-batch loss for batch $\xi$, and let $\mathscr{L}_f(\theta)$ be the population loss over true dataset distribution. Let $\theta_k,\theta_t$ denote the parameter at iteration $k$ and time $t$, respectively.


\textbf{Mini-Batch Gradient as Wiener Process.} Let $g_k(s):= \nabla \ell_f(s;\theta_k)$ be the per-sample gradient for sample $s$ with the mean and covariance
\begin{align}
    \bar g_k = \mathbb{E}_s[g_k(s)],
    \qquad 
    \Sigma_k = \mathrm{Cov}_s\bigl[ g_k(s) \bigr]
    ,
\end{align}
respectively. For a batch $\xi_k:=\{s_i\}_{i=1}^{b}$ of size $b$, the gradient at time $k$ is given by:
\begin{align}
g_k &= \frac{1}{b}\sum_{i=1}^b g_k{(s_i)} 
,
\end{align}
by central limit theorem, so that $g_k$ distributionally converge to
\begin{align}
  g_k 
  \overset{d}{\longrightarrow} \mathcal N\!\left(\bar g_k,\frac{1}{b}\Sigma_k\right),
\end{align}
as batch size is sufficiently large $b \to \infty$. We therefore can model mini-batch gradient $g_k$ as a Wiener process 
\begin{align}
g_k = \bar g_k - \frac{1}{\sqrt{b}} \Sigma_k^{1/2} \zeta_k  ,
\qquad
(\Sigma_k^{1/2})(\Sigma_k^{1/2})^\top = \Sigma_k
,
\end{align}
where $\zeta_k \sim \mathcal{N}(0, I_p)$ is Wiener process adapted to the filtration generated from mini-batch sampling sequence $\{\xi_k\}$. Because gradient descent uses negative gradient, we take a negative sign on $\zeta_k$ for simplifying the later analysis algebraically.

\paragraph{Discrete Adam with Bias Correction.} Let $\eta > 0$ be sufficiently small learning rate, $\beta_1, \beta_2 \in (0,1)$ be exponential decay rates, and $\epsilon > 0$ be numerical stability constant. Let $\theta_t \in \mathbb{R}^p$, $m_t$, $v_t$ be parameter, first-moment, and second-moment estimates, with the initialization $m_0=0$ and $v_0=0$. Let $k = 0, 1, 2, \ldots$ be discrete iteration time points. The discrete Adam update rules with the bias-correction step written explicitly are:
\begin{align}
\bar g_k &:= \nabla[\mathscr{L}_f(\theta_k)] = \nabla[\mathscr{L}^{*}_f(\theta_k) + \frac{\lambda}{2} \|\theta_k\|_2^2 ] & (\text{coupled $\ell_2$-regularizer})\\
g_k &= \bar g_k - \frac{1}{\sqrt{b}}\Sigma_k^{1/2}\zeta_k & (\text{mini-batch gradient}) \label{eq:adam_minibatch_gradient_appendix}\\
m_{k+1} &= \beta_1 m_k + (1-\beta_1)g_k  &(\text{first-moment estimate}) \label{eq:adam_first_moment_appendix}\\
v_{k+1} &= \beta_2 v_k + (1-\beta_2)(g_k \odot g_k) &(\text{second-moment estimate}) \label{eq:adam_second_moment_appendix}\\
\hat m_{k+1} &= \frac{m_{k+1}}{1-\beta_1^{\,k+1}} &(\text{bias-corrected first-moment estimate}) \label{eq:adam_first_moment_biascorr_appendix}\\
\hat v_{k+1} &= \frac{v_{k+1}}{1-\beta_2^{\,k+1}} &(\text{bias-corrected second-moment estimate}) \label{eq:adam_second_moment_biascorr_appendix}\\
\theta_{k+1} &= \theta_k - \eta \, h(\hat m_{k+1}, \hat v_{k+1}) &(\text{bias-corrected parameter update}) \label{eq:adam_param_update_appendix}
,
\end{align}
where $\lambda$ is the regularization coefficient, and $\mathscr{L}_f(\theta_k)$ is the population coupled loss and $\mathscr{L}^{*}_f(\theta_k)$ is the population task loss, then the per-coordinate Adam update direction is
\begin{equation}
h(\hat m, \hat v) \;=\;
\hat m \oslash (\sqrt{\hat v} + \epsilon\mathbf{1}_p) \;=\; \begin{bmatrix}\dfrac{\hat m^{(1)}}{\sqrt{\hat v^{(1)}} + \epsilon}\\[4pt]\vdots\\\dfrac{\hat m^{(p)}}{\sqrt{\hat v^{(p)}} + \epsilon}\end{bmatrix}
,
\end{equation}
$\oslash$ denotes element-wise division, and $\varepsilon$ is a small positive constant for numerical stability.

\begin{remark}
Throughout this paper, we use the convention that the $\ell_2$ regularizer is coupled through the loss function. Thus, the gradients induced by the regularizer enter the computation of the optimizer's first- and second-moment estimates. When the $\ell_2$ penalty is instead decoupled and applied as a standalone update term,
\begin{align}
    \theta_{k+1}
    =
    \theta_k
    -
    \eta\, h(\hat{m}_{k+1}, \hat{v}_{k+1})
    -
    \eta\,\lambda\,\theta_k,
\end{align}
where the resulting optimizer is referred to as \emph{AdamW}.
\end{remark}

\paragraph{Continuous--Time Interpolation.} Taking continuous-time interpolation in discrete iteration $k$ by:
\begin{equation}
t = k\eta,
\end{equation}
therefore:
\begin{equation}
\dd t = \lim_{\Delta t \to 0}  
\Delta t = \lim_{\Delta t \to 0} \eta = \eta. 
\end{equation}

\paragraph{Wiener Increment in Gradient.} Let $g_k(s):= \nabla \ell_f(s;\theta_k)$ be the per-sample gradient for sample $s$ with the mean and covariance
\begin{align}
    \bar g_k = \mathbb{E}_s[g_k(s)],
    \qquad 
    \Sigma_k = \mathrm{Cov}_s\bigl[ g_k(s) \bigr]
    ,
\end{align}
respectively. For a batch $\xi_k:=\{s_i\}_{i=1}^{b}$ of size $b$, the gradient at time $k$ is given by:
\begin{align}
g_k &= \frac{1}{b}\sum_{i=1}^b g_k{(s_i)} 
,
\end{align}
by central limit theorem, so that $g_k$ distributionally converge to
\begin{align}
  g_k 
  \overset{d}{\longrightarrow} \mathcal N\!\left(\bar g_k,\frac{1}{b}\Sigma_k\right),
\end{align}
as batch size is sufficiently large $b \to \infty$. We therefore can model mini-batch gradient $g_k$ as a Wiener process 
\begin{align}
g_k = \bar g_k - \frac{1}{\sqrt{b}} \Sigma_k^{1/2} \zeta_k  ,
\qquad
(\Sigma_k^{1/2})(\Sigma_k^{1/2})^\top = \Sigma_k
\label{eq:discrete_minibatch_gradient_appendix}
,
\end{align}
where $\zeta_k \sim \mathcal{N}(0, I_p)$ is Wiener process adapted to the filtration generated from mini-batch sampling sequence $\{\xi_k\}$. Because gradient descent uses negative gradient, we take a negative sign on $\zeta_k$ for simplifying the later analysis algebraically. 

For the adapted noise vector $\zeta_k \sim \mathcal{N}(0, I_p) \in \mathbb{R}^p$ at iteration $k$ (time $t = k\eta$) corresponds to the Wiener increment over $[t, t+\eta]$:
\begin{align}
    \zeta_k = \frac{1} {\sqrt{\eta}}\zeta_k \sqrt{\eta} = \frac{1} {\sqrt{\eta}}\zeta_k \sqrt{\Delta t} 
    = \frac{1}{\sqrt{\eta}} \Delta W_k
    ,
\end{align}
where $\Delta t = \eta$ and the increment follows the property of Brownian motion:
\begin{align}
\Delta W_k =\zeta_k \sqrt{\Delta t}  \sim \mathcal{N}(0, \Delta t\; I_p).    
\end{align}

\begin{remark}[Sanity-Check Reference: Continuous SDE for SGD]
Therefore, for SGD, the continuous-time SDE limit is, with $t = k\eta$ and $\Delta t = \eta$:
\begin{align}
    \dd \theta_t &= \lim_{\Delta t \to 0} \bigl(\theta_{t+\Delta t} - \theta_t\bigr) \\
    &= \lim_{\Delta t \to 0} \Biggl[-\eta \Bigl(\bar g_k - \frac{1}{\sqrt{b}}\Sigma_k^{1/2}\zeta_k\Bigr) \Biggr] \\
    &=\lim_{\Delta t \to 0} \Biggl[ -\eta \bar g_k + \sqrt{\frac{\eta}{b}} \Sigma_k^{1/2} \sqrt{\eta} \zeta_k \Biggr] \\
    &=\lim_{\Delta t \to 0} \Biggl[ -\eta \bar g_k + \sqrt{\frac{\eta}{b}} \Sigma_k^{1/2} \Delta W_k \Biggr] \\
    &= - \bar g_t \dd t + \sqrt{\frac{\eta}{b}} \Sigma_t^{1/2} \dd W_t.
\end{align}
\end{remark}

\paragraph{Discrete Increment of First-Moment Estimate.} From equation~\eqref{eq:adam_first_moment_appendix}:
\begin{align}
m_{k+1} &= \beta_1 m_k + (1-\beta_1)g_k\\
m_{k+1} - m_k &= \beta_1 m_k + (1-\beta_1)g_k - m_k\\
&= (\beta_1 - 1)m_k + (1-\beta_1)g_k\\
&= -(1-\beta_1)m_k + (1-\beta_1)g_k\\
&= (1-\beta_1)(g_k - m_k)\\
&= -(1-\beta_1)(m_k - g_k) \label{eq:first_moment_increment_appendix}
\end{align}

Substituting equation~\eqref{eq:discrete_minibatch_gradient_appendix}:
\begin{align}
g_k = \bar g_k - \frac{1}{\sqrt{b}}\Sigma_k^{1/2}\zeta_k    
\end{align}
into equation~\eqref{eq:first_moment_increment_appendix} yields:
\begin{align}
m_{k+1} - m_k &= -(1-\beta_1)\left[m_k - \bar g_k + \frac{1}{\sqrt{b}}\Sigma_k^{1/2}\zeta_k\right]\\
&= -(1-\beta_1)[m_k - \bar g_k] - (1-\beta_1)\frac{1}{\sqrt{b}}\Sigma_k^{1/2}\zeta_k \\
&=-(1-\beta_1)[m_k - \bar g_k] - \frac{1-\beta_1}{\sqrt{b}}\Sigma_k^{1/2}\zeta_k
.
\end{align}

\paragraph{Continuous-Time Limit of First-Moment Estimate.} Use $\Delta t = \eta$ and $\Delta W_k = \sqrt{\Delta}\zeta_k$,
\begin{align}
    \dd m &= \lim_{\Delta t \to 0} \Delta m \\
    &= \lim_{\Delta t \to 0} (m_{k+1} - m_k) \\
    &= \lim_{\Delta t \to 0}
    \Biggl[
    -(1-\beta_1)[m_k - \bar g_k] - \frac{1-\beta_1}{\sqrt{b}}\Sigma_k^{1/2}\zeta_k
    \Biggr] \\
    &= \lim_{\Delta t \to 0}
    \Biggl[
    -(1-\beta_1)[m_k - \bar g_k] \frac{1}{\eta} \Delta t - \frac{1-\beta_1}{\sqrt{b}}\Sigma_k^{1/2} \frac{1}{\sqrt{\eta}} \sqrt{\Delta t}\zeta_k
    \Biggr] \\
    &= -\frac{1-\beta_1}{\eta}[m_t - \bar g_k]\dd t - \frac{1-\beta_1}{\sqrt{b\eta}}\Sigma_k^{1/2}\dd W_t \label{eq:first_moment_ct_limit_appendix}
    .
\end{align}

\paragraph{Expansion of Element-Wise Squared Gradient.} To derive the continuous-time limit of the second-moment estimate $\dd v$. We need to expand $g_k \odot g_k$ of equation~\eqref{eq:adam_second_moment_biascorr_appendix} into a tractable form. Consider:
\begin{align}
g_k = \bar g_k - \frac{1}{\sqrt{b}}\Sigma_k^{1/2}\zeta_k,    
\end{align}
we expand the element-wise quadratic product $g_k \odot g_k$ by:
\begin{align}
g_k \odot g_k &= \left[ \bar g_k - \frac{1}{\sqrt{b}}\Sigma_k^{1/2}\zeta_k \right] \odot \left[ \bar g_k - \frac{1}{\sqrt{b}}\Sigma_k^{1/2}\zeta_k \right] \\
&=\bar g_k \odot \bar g_k - 2\left( \bar g_k \odot \frac{1}{\sqrt{b}}\Sigma_k^{1/2}\zeta_k\right)  + \frac{1}{b}\left(\Sigma_k^{1/2}\zeta_k\right) \odot \left(\Sigma_k^{1/2}\zeta_k\right)\\
&= \left[\bar g_k\right]^{\odot 2} - \frac{2}{\sqrt{b}} \bar g_k \odot \left[\Sigma_k^{1/2}\zeta_k\right] + \frac{1}{b}(\Sigma_k^{1/2}\zeta_k)^{\odot 2}
.
\end{align}

Using the fundamental identity, for two vectors $u$ and $v$:
\begin{align}
    u \odot v = \mathrm{diag}(u)v,
\end{align}
then the Hadamard cross product can be written as:
\begin{equation}
\bar g_k \odot \left[\Sigma_k^{1/2}\zeta_k\right] = \mathrm{diag}(\bar g_k)\Sigma_k^{1/2}\zeta_k
.
\end{equation}
We write $g_k \odot g_k$ as:
\begin{align}
g_k \odot g_k=
 {\bar g_k}^{\odot 2} - \frac{2}{\sqrt{b}} \mathrm{diag}({\bar g_k}) \left[\Sigma_k^{1/2}\zeta_k\right] + \frac{1}{b}(\Sigma_k^{1/2}\zeta_k)^{\odot 2}
 \label{eq:squared_gradient_expansion_appendix}
 .
\end{align}

\paragraph{Statistics of Element-Wise Squared Gradient.}
Fix $\theta_k$ over all batches, the element-wise squared gradient $g_k \odot g_k$ admits a mean
\begin{align}
\mathbb{E}_{\xi_k}[g_k\odot g_k \mid \theta_k]
&=
\mathbb{E}
\left[
{\bar g_k}^{\odot 2}
-
\frac{2}{\sqrt{b}}
\operatorname{diag}\!\left({\bar g_k}\right)
\Sigma_k^{1/2}\zeta_k
+
\frac{1}{b}
\left(\Sigma_k^{1/2}\zeta_k\right)^{\odot 2}
\right]
\nonumber\\
&=
{\bar g_k}^{\odot 2}
-
\frac{2}{\sqrt{b}}
\operatorname{diag}\!\left({\bar g_k}\right)
\mathbb{E}\!\left[\Sigma_k^{1/2}\zeta_k\right]
+
\frac{1}{b}
\mathbb{E}
\left[
\left(\Sigma_k^{1/2}\zeta_k\right)^{\odot 2}
\right]
\nonumber\\
&=
{\bar g_k}^{\odot 2}
+
\frac{1}{b}
\operatorname{diag}\!\left(\Sigma_k\right)
\label{equ:gk_odot_gk_mean}
,
\end{align}
and a covariance
\begin{align}
\mathrm{Cov}_{\xi_k}(g_k\odot g_k \mid \theta_k)
&=
\mathrm{Cov}
\left(
-\frac{2}{\sqrt{b}}
\operatorname{diag}\!\left({\bar g_k}\right)
\Sigma_k^{1/2}\zeta_k
+
\frac{1}{b}
\left(\Sigma_k^{1/2}\zeta_k\right)^{\odot 2}
\right)
\nonumber\\
&=
\frac{4}{b}
\operatorname{diag}\!\left({\bar g_k}\right)
\Sigma_k
\operatorname{diag}\!\left({\bar g_k}\right)
+
\frac{1}{b^2}
\mathrm{Cov}
\left(
\left(\Sigma_k^{1/2}\zeta_k\right)^{\odot 2}
\right)
\nonumber\\
&=
\frac{4}{b}
\operatorname{diag}\!\left({\bar g_k}\right)
\Sigma_k
\operatorname{diag}\!\left({\bar g_k}\right)
+
\frac{1}{b^2}
\left[
2\left(\Sigma_k\odot \Sigma_k\right)
\right]
\nonumber\\
&=
\frac{4}{b}
\operatorname{diag}\!\left({\bar g_k}\right)
\Sigma_k
\operatorname{diag}\!\left({\bar g_k}\right)
+
\frac{2}{b^2}
\left(\Sigma_k\odot \Sigma_k\right)
\label{equ:gk_odot_gk_cov}
,
\end{align}
where the last identity follows componentwise from
\begin{align}
\left[
\mathrm{Cov}
\left(
\left(\Sigma_k^{1/2}\zeta_k\right)^{\odot 2}
\right)
\right]_{ij} &=
\mathbb{E}
\left[
\left(\Sigma_k^{1/2}\zeta_k\right)_i^2
\left(\Sigma_k^{1/2}\zeta_k\right)_j^2
\right]
-
\mathbb{E}
\left[
\left(\Sigma_k^{1/2}\zeta_k\right)_i^2
\right]
\mathbb{E}
\left[
\left(\Sigma_k^{1/2}\zeta_k\right)_j^2
\right]
\nonumber\\
&=
\left[
(\Sigma_k)_{ii}(\Sigma_k)_{jj}
+
2(\Sigma_k)_{ij}^2
\right]
-
(\Sigma_k)_{ii}(\Sigma_k)_{jj}
\nonumber\\
&=
2(\Sigma_k)_{ij}^2 = 2 (\Sigma_k \odot \Sigma_k)_{ij}.
\end{align}

\paragraph{Continuous-Time Limit of Second-Moment Estimate.} 
From equation~\eqref{eq:adam_second_moment_biascorr_appendix}:
\begin{align}
v_{k+1} - v_k &= \beta_2 v_k + (1-\beta_2)(g_k \odot g_k) - v_k\\
&= -(1-\beta_2)v_k + (1-\beta_2)(g_k \odot g_k)\\
&= -(1-\beta_2)[v_k - g_k \odot g_k]
,
\end{align}
substituting equation~\eqref{eq:squared_gradient_expansion_appendix} yields:
\begin{align}
v_{k+1} - v_k = &-(1-\beta_2)\Bigl[v_k - {\bar g_k}^{\odot 2} + \frac{2}{\sqrt{b}}\mathrm{diag}({\bar g_k})\Sigma_k^{1/2}\zeta_k - \frac{1}{b}(\Sigma_k^{1/2}\zeta_k)^{\odot 2}\Bigr]
\end{align}

Expanding yields:
\begin{align}
v_{k+1} - v_k = &-(1-\beta_2)\Bigl[v_k - {\bar g_k}^{\odot 2}\Bigr]
- \frac{2(1-\beta_2)}{\sqrt{b}}\mathrm{diag}({\bar g_k})\Sigma_k^{1/2}\zeta_k  + \frac{1-\beta_2}{b}(\Sigma_k^{1/2}\zeta_k)^{\odot 2} 
.
\end{align}

Consider the diffusion terms:
\begin{align}
\frac{2(1-\beta_2)}{\sqrt{b}}\mathrm{diag}({\bar g_k})\Sigma_k^{1/2}\zeta_k &= \frac{2(1-\beta_2)}{\sqrt{b}}\mathrm{diag}({\bar g_k})\Sigma_k^{1/2}\frac{\Delta W_k}{\sqrt{\eta}}\\
&= \frac{2(1-\beta_2)}{\sqrt{b\eta}}\mathrm{diag}({\bar g_k})\Sigma_k^{1/2}\Delta W_k
,
\end{align}
and:
\begin{align}
\left[\Sigma_k^{1/2}\zeta_k\right]_i^{2} &= \left[\Sigma_k^{1/2}\frac{\Delta W_k}{\sqrt{\eta}}\right]_i^{2}\\
&= \left(\sum_{j=1}^p (\Sigma_k)_{ij}^{1/2}\frac{\Delta W_k^{(j)}}{\sqrt{\eta}}\right)^2\\
&= \frac{1}{\eta}\sum_{j=1}^p\sum_{h=1}^p (\Sigma_k)_{ij}^{1/2}(\Sigma_k)_{ih}^{1/2}\Delta W_k^{(j)}\Delta W_k^{(h)}
\\
&= \frac{1}{\eta}\sum_{j=1}^p\sum_{h=1}^p (\Sigma_k)_{ij}^{1/2}(\Sigma_k)_{ih}^{1/2}\delta_{jh} \Delta t\\
&= \frac{1}{\eta}\sum_{j=1}^p [(\Sigma_k)_{ij}^{1/2}]^2 \, \Delta t \\
&= \frac{\Delta t}{\eta}\,(\Sigma_k)_{ii} \;=\; \frac{\Delta t}{\eta}\,\mathrm{diag}\bigl(\Sigma_k\bigr)_i\,,
\end{align}
so $(\Sigma^{1/2}\zeta_k)^{\odot 2} = (\Delta t/\eta)\,\mathrm{diag}(\Sigma_k)$. 

Take limits $\Delta t \to \dd t, \Delta W_k \to \dd W_t$, then the continuous-time limit of $\Delta v = v_{k+1} - v_k$ is:
\begin{align}
\dd v_t = &-\frac{1-\beta_2}{\eta}\Bigl[v_t - {\bar g_k}^{\odot 2}\Bigr]\dd t \nonumber\\
&\quad -\frac{2(1-\beta_2)}{\sqrt{b\eta}}\,\mathrm{diag}({\bar g_t})\,\Sigma_t^{1/2}\dd W_t
+ \frac{1-\beta_2}{\eta b}\,\mathrm{diag}\bigl(\Sigma_t\bigr)\dd t.
\label{eq:second_moment_ct_limit_appendix}
\end{align}

\paragraph{Continuous-Time Limit of Parameter Dynamics.} From equation~\eqref{eq:adam_param_update_appendix} (with $\varepsilon$ dropped), the discrete update uses the bias-corrected moments
\begin{align}
\hat m_k = m_k/(1-\beta_1^k), \qquad
\hat v_k = v_k/(1-\beta_2^k)
,
\end{align}
then
\begin{equation}
\theta_{k+1} \;=\; \theta_k - \eta\, \hat m_k \oslash \sqrt{\hat v_k} \;=\; \theta_k - \eta\,\frac{\sqrt{1-\beta_2^k}}{1-\beta_1^k}\;m_k \oslash \sqrt{v_k}.
\end{equation}
Under the joint scaling $\eta\to 0$ with $(1-\beta_i)/\eta = \alpha_i$ fixed and $t = \eta k$,
\begin{align}
1 - \beta_i^k \;=\; 1 - (1-\alpha_i\eta)^{t/\eta} \;\xrightarrow[\eta\to 0]{}\; 1 - e^{-\alpha_i t}, \notag\\
\frac{\sqrt{1-\beta_2^k}}{1-\beta_1^k} \;\xrightarrow[\eta\to 0]{}\; 
\frac{\sqrt{1-e^{-\alpha_2t}}}{1-e^{-\alpha_1t}} =:
B(t),
\label{eq:bias_correction_limit_appendix}
\end{align}
with $B(t)$ defined in equation~\eqref{eq:bias_correction_factor_restated}. The continuous-time $\theta$-dynamics is therefore
\begin{equation}
\dd\theta_t \;=\; -\,B(t)\,m_t \oslash \sqrt{v_t}\,\dd t,
\label{eq:theta_ct_limit_appendix}
\end{equation}
which is time-inhomogeneous through $B(t)$.

\paragraph{Joint-State SDE.} 
To simplify discussion, by equations~\eqref{eq:squared_gradient_expansion_appendix},~\eqref{equ:gk_odot_gk_mean},~\eqref{equ:gk_odot_gk_cov},
\begin{align}
&g_t \odot g_t=
 {\bar g_t}\odot {\bar g_t} - \frac{2}{\sqrt{b}} \mathrm{diag}({\bar g_t}) \left[\Sigma_t^{1/2}\zeta_t\right] + \frac{1}{b}(\Sigma_t^{1/2}\zeta_t)^{\odot 2}
 , \\
 &\mathbb{E}_{\xi_t}[g_t\odot g_t \mid \theta_t] = {\bar g_t} \odot {\bar g_t}
+
\frac{1}{b}
\operatorname{diag}\!\left(\Sigma_t\right), \\
&\mathrm{Cov}_{\xi_t}(g_t\odot g_t \mid \theta_t)=\frac{4}{b}
\operatorname{diag}\!\left({\bar g_t}\right)
\Sigma_t
\operatorname{diag}\!\left({\bar g_t}\right)
+
\frac{2}{b^2}
\left(\Sigma_t\odot \Sigma_t\right)
,
\end{align}
we set
\begin{align}
&\overline{g \odot g}_t \;:=\; \mathbb{E}[g_t\odot g_t \mid \theta_k] = \bar g_t \odot \bar g_t \;+\; \tfrac{1}{b}\,\mathrm{diag}\,(\Sigma_t), \label{eq:gsq_mean_def_appendix}\\
&D_t \;:=\; 2\,\mathrm{diag}\bigl(\bar g_t\bigr)\,\Sigma_t^{1/2}, \quad \frac{1}{b}D_tD_t^\top = \mathrm{Cov}[g_k\odot g_k \mid \theta_k] - O(\frac{1}{b^2})  \label{eq:diffusion_cross_term_def_assembled_appendix}
,
\end{align}
where $D_t$ is the diffusion factor of the leading cross-term in $g_k\odot g_k$; the sub-leading term has Wick's variance at a scale by $O(1/b^2)$ and is dropped if $b \to \infty$.

Set 
\begin{align}
\alpha_i := (1-\beta_i)/\eta
,
\end{align}
and combining equations~\eqref{eq:theta_ct_limit_appendix},~\eqref{eq:first_moment_ct_limit_appendix},~\eqref{eq:second_moment_ct_limit_appendix},
\begin{align}
\dd\theta_t
&=
-\,B(t)\,m_t \oslash \sqrt{v_t}\,\dd t,
\\
\dd m_t
&=
-\frac{1-\beta_1}{\eta}
\left[
m_t-{\bar g_t}
\right]\dd t
-
\frac{1-\beta_1}{\sqrt{b\eta}}
\Sigma_t^{1/2}\dd W_t
\nonumber\\
&=
-\alpha_1(m_t-\bar g_t)\dd t
-
\sqrt{\frac{\eta}{b}}\,
\alpha_1\Sigma_t^{1/2}\dd W_t,
\\
\dd v_t
&=
-\frac{1-\beta_2}{\eta}
\Bigl[
v_t-{\bar g_t}\odot{\bar g_t}
\Bigr]\dd t
\nonumber\\
&\phantom{=}\;
-
\frac{2(1-\beta_2)}{\sqrt{b\eta}}\,
\operatorname{diag}({\bar g_t})\,
\Sigma_t^{1/2}\dd W_t
+
\frac{1-\beta_2}{\eta b}\,
\operatorname{diag}\bigl(\Sigma_t\bigr)\dd t
\nonumber\\
&=
-\alpha_2
\Bigl[
v_t-\bar g_t\odot\bar g_t
-
\frac{1}{b}\operatorname{diag}(\Sigma_t) \Bigr]\dd t
-
\sqrt{\frac{\eta}{b}}\,
\alpha_2D_t\dd W_t
\nonumber\\
&=
-\alpha_2
\left(
v_t-\overline{g\odot g}_t
\right)\dd t
-
\sqrt{\frac{\eta}{b}}\,
\alpha_2D_t\dd W_t 
,
\end{align}
so that:
\begin{align}
    \dd
    \begin{pmatrix}
    \theta_t \\
    m_t \\
    v_t
    \end{pmatrix}
    =
    \begin{pmatrix}
- B(t)\,m_t \oslash \sqrt{v_t} \\
-\alpha_1\,(m_t - \bar g_t) \\
-\alpha_2\,(v_t - \overline{g\odot g}_t)
\end{pmatrix}
\dd t +
\sqrt{\frac{\eta}{b}}
\begin{pmatrix}
0_{p\times p} \\
\alpha_1\,\Sigma_t^{1/2} \\
\alpha_2\,D_t
\end{pmatrix}
\dd W_t.
\end{align}

Write the joint-state as
\begin{align}
S_t :=
\begin{pmatrix}
    \theta_t \\
    m_t \\
    v_t
    \end{pmatrix}
    \in 
    \mathbb{R}^{3p}
    ,
\end{align}
then Adam's optimization is characterized by
\begin{equation}
\dd S_t \;=\; \mu(S_t)\,\dd t \;+\; \sqrt{\frac{\eta}{b}}\,\sigma(S_t)\,\dd W_t,
\label{eq:adam_joint_sde_assembled_appendix}
\end{equation}
with $W_t$ a $p$-dimensional Wiener process and
\begin{equation}
\mu(S_t) \;=\;
\begin{pmatrix}
- B(t)\,m_t \oslash \sqrt{v_t} \\
-\alpha_1\,(m_t - \bar g_t) \\
-\alpha_2\,(v_t - \overline{g\odot g}_t)
\end{pmatrix} \in \mathbb{R}^{3p},
\qquad
\sigma(S_t) \;=\;
\begin{pmatrix}
0_{p\times p} \\
\alpha_1\,\Sigma_t^{1/2} \\
\alpha_2\,D_t
\end{pmatrix}
\in \mathbb{R}^{3p \times p}
.
\label{eq:adam_joint_drift_diffusion_assembled_appendix}
\end{equation}

The technical correctness verification is provided in later sections and experiments, such as Appendix~\ref{app:induced_radius_sde_appendix} through an numerical experiment.

\end{proof}

\begin{remark}[SGD's SDE as Special Case]
The joint-state SDE in equation~\eqref{eq:adam_joint_sde_assembled_appendix} contains the SGD's SDE as a limiting case. Specifically, take the no-momentum limit $\beta_1=0$, so that $m_t=g_t$, and ignore the second-moment state $v_t$ equivalently, take $\beta_2=0$ and replace $\sqrt{v_t}$ by $\mathbf 1_p$. Then Adam's $\theta$-update reduces to the SGD update
\begin{align}
\theta_{k+1}
=
\theta_k-\eta g_k,
\end{align}
and its continuous-time interpolation gives the standard SGD SDE
\begin{align}
\dd\theta_t
=
-\bar g_t\dd t
+
\sqrt{\eta/b}\,\Sigma_t^{1/2}\dd W_t .
\end{align}
\end{remark}

%% file: proofs/proof_induced_radius_SDE.tex
\begin{lemma}[Induced Radius SDE]
\label{lem:induced_radius_sde_restated}
Let $S_t=(\theta_t^\top,m_t^\top,v_t^\top)^\top\in\mathbb R^{3p}$ evolve according to Adam's joint-state SDE
\begin{align}
\dd S_t
=
\mu(S_t)\dd t
+
\sqrt{\frac{\eta}{b}}\,
\sigma(S_t)\dd W_t,
\end{align}
where
\begin{align}
\mu(S_t)
=
\begin{pmatrix}
-\,B(t)\,m_t\oslash\sqrt{v_t}
\\
-\,\alpha_1(m_t-\bar g_t)
\\
-\,\alpha_2(v_t-\overline{g\odot g}_t)
\end{pmatrix},
\qquad
\sigma(S_t)
=
\begin{pmatrix}
0
\\
\alpha_1\Sigma_t^{1/2}
\\
\alpha_2D_t
\end{pmatrix}
,
\end{align}
$W_t\in\mathbb R^p$ is a Wiener process, and $B(t)$ is the Adam bias-correction factor. Define
\begin{align}
r_t^2:=\|\theta_t\|_2^2 
,
\end{align}
then $r_t^2$ admits the dynamics
\begin{align}
\dd r_t^2
=
-2B(t)\theta_t^\top
\left(m_t\oslash\sqrt{v_t}\right)\dd t .
\label{eq:induced_radius_sde_restated}
\end{align}
This is not an ordinary differential equation (ODE) as the $m_t$ and $v_t$ are stochastic through its dependence on the joint state $S_t$.
\end{lemma}

\begin{proof}
Recall from Lemma~\ref{lem:adam_joint_sde_restated} the joint-state SDE on $S_t = (\theta_t^\top, m_t^\top, v_t^\top)^\top \in \mathbb{R}^{3p}$,
\begin{align}
\dd S_t \;=\; \mu(S_t)\,\dd t \;+\; \sqrt{\eta/b}\,\sigma(S_t)\,\dd W_t,
\label{eq:adam_joint_sde_recall_appendix}
\end{align}
with $\mu(S_t)$ and $\sigma(S_t)$ as in equations~\eqref{eq:adam_joint_drift_diffusion_assembled_appendix}.

Define a projector $E \in \mathbb{R}^{3p \times 3p}$
\begin{align}
E \;:=\; \begin{pmatrix} I_p & 0_p & 0_p \\ 0_p & 0_p & 0_p \\ 0_p & 0_p & 0_p \end{pmatrix},
\label{eq:theta_block_projector_def_appendix}
\end{align}
where each block is $p \times p$. Then $E$ is symmetric and idempotent, $E^{\!\top} = E$ and $E^2 = E$, and for $S = (\theta^\top, m^\top, v^\top)^\top$ the squared parameter norm is the quadratic form
\begin{align}
r^2 \;=\; \|\theta\|_2^2 \;=\; S^{\!\top} E\, S.
\label{eq:squared_radius_quadform_appendix}
\end{align}

Apply It\^o's lemma to $f(S) := S^{\!\top} E\, S$ under equation~\eqref{eq:adam_joint_sde_recall_appendix}. The gradient and Hessian of $f$ are
\begin{align}
\nabla_S f(S) \;=\; 2\,E\,S, \qquad \nabla_S^2 f(S) \;=\; 2\,E,
\label{eq:quadform_grad_hess_appendix}
\end{align}
so
\begin{align}
\dd r_t^2
\;=\; 2\,S_t^{\!\top} E\,\mu(S_t)\,\dd t
\;+\; 2\sqrt{\eta/b}\,S_t^{\!\top} E\,\sigma(S_t)\,\dd W_t
\;+\; \tfrac{\eta}{b}\,\mathrm{tr}\!\bigl(\sigma(S_t)^{\!\top} E\,\sigma(S_t)\bigr)\,\dd t.
\label{eq:ito_on_quadform_appendix}
\end{align}
Each contraction with $E$ keeps only the $\theta$-block:
\begin{align}
S_t^{\!\top} E\,\mu(S_t) \;&=\; \theta_t^{\!\top}\,\mu_\theta(S_t) \;=\; -\,B(t)\,\theta_t^{\!\top}\!\bigl(m_t \oslash \sqrt{v_t}\bigr), \label{eq:projected_drift_appendix} \\
S_t^{\!\top} E\,\sigma(S_t) \;&=\; \theta_t^{\!\top}\,\sigma_\theta(S_t) \;=\; 0, \label{eq:projected_diffusion_zero_appendix} \\
\mathrm{tr}\!\bigl(\sigma(S_t)^{\!\top} E\,\sigma(S_t)\bigr) \;&=\; \mathrm{tr}\!\bigl(\sigma_\theta(S_t)^{\!\top}\sigma_\theta(S_t)\bigr) \;=\; 0. \label{eq:projected_ito_correction_zero_appendix}
\end{align}
Equations~\eqref{eq:projected_diffusion_zero_appendix}--\eqref{eq:projected_ito_correction_zero_appendix} use $\sigma_\theta \equiv 0$. Substituting equations~\eqref{eq:projected_drift_appendix}--\eqref{eq:projected_ito_correction_zero_appendix} into equation~\eqref{eq:ito_on_quadform_appendix} yields
\begin{align}
\dd r_t^2 \;=\; -\,2\,B(t)\,\theta_t^{\!\top}\!\bigl(m_t \oslash \sqrt{v_t}\bigr)\,\dd t,
\label{eq:induced_radius_sde_final_appendix}
\end{align}
which is equation~\eqref{eq:induced_radius_sde_restated}.
\end{proof}

%% file: proofs/proof_mean_field_mt_and_vt_limit.tex
\begin{lemma}[Mean-Field Limit of First- and Second-Moment Estimates]
\label{lem:meanfield_mt_and_vt_limit}
Let $m_t$ and $v_t$ be Adam's first- and second-moment states, let $g_t$ be
the mini-batch gradient at time $t$, let
$\bar g_t:= \mathbb{E}[g_t \mid \theta_t]$ be the mean of $g_t$, and let
$\overline{g\odot g}_t := \mathbb{E}[g_t\odot g_t \mid \theta_t]$ be the mean of
$g_t\odot g_t$. Assume that $\bar g_t$ and $\overline{g\odot g}_t$ vary slowly
in the long-term dynamics as $t \to \infty$. Then, under initialization
$m_0=0,v_0=0$, the first- and second-moment estimates admit the mean-field
limit
\begin{align}
m_t
&\to
\bar g_t
\left(
1-e^{-\alpha_1t}
\right),
\qquad
v_t
\to
\overline{g\odot g}_t
\left(
1-e^{-\alpha_2t}
\right),
\end{align}
so that
\begin{align}
m_t\oslash\sqrt{v_t}
&\to
\bar g_t\bigl(1-e^{-\alpha_1 t}\bigr)
\oslash
\sqrt{\overline{g\odot g}_t (1 - e^{-\alpha_2t})}
\nonumber\\
&\approx
\bar g_t\bigl(1-e^{-\alpha_1 t}\bigr)
\oslash
\sqrt{\overline{g\odot g}_t}.
\label{eq:meanfield_evolution}
\end{align}
\end{lemma}

\begin{proof}
We start from the stochastic moment dynamics of the joint-state Adam SDE,
\begin{align}
\dd m_t
&=
-\alpha_1(m_t-\bar g_t)\dd t
-
\sqrt{\frac{\eta}{b}}\,
\alpha_1\Sigma_t^{1/2}\dd W_t,
\label{eq:full_mt_sde_for_meanfield}
\\
\dd v_t
&=
-\alpha_2
\left(
v_t-\overline{g\odot g}_t
\right)\dd t
-
\sqrt{\frac{\eta}{b}}\,
\alpha_2D_t\dd W_t .
\label{eq:full_vt_sde_for_meanfield}
\end{align}
The mean-field approximation replaces the stochastic moment dynamics by their
conditional mean dynamics. Since the stochastic terms are martingale
increments,
\begin{align}
\mathbb E[\dd W_t\mid \theta_t,m_t,v_t]
&=
0.
\end{align}
Taking conditional expectation in
\eqref{eq:full_mt_sde_for_meanfield} and
\eqref{eq:full_vt_sde_for_meanfield} gives
\begin{align}
\mathbb E[\dd m_t\mid \theta_t,m_t,v_t]
&=
-\alpha_1(m_t-\bar g_t)\dd t,
\\
\mathbb E[\dd v_t\mid \theta_t,m_t,v_t]
&=
-\alpha_2
\left(
v_t-\overline{g\odot g}_t
\right)\dd t.
\end{align}
Thus, under the mean-field dynamics,
\begin{align}
\dd m_t
&=
-\alpha_1(m_t-\bar g_t)\dd t,
\label{eq:meanfield_mt_dynamics}
\\
\dd v_t
&=
-\alpha_2
\left(
v_t-\overline{g\odot g}_t
\right)\dd t.
\label{eq:meanfield_vt_dynamics}
\end{align}

Because $\bar g_t$ varies slowly on the $m_t$ relaxation time scale, we freeze
$\bar g_t$ when solving equation~\eqref{eq:meanfield_mt_dynamics}. Hence
\begin{align}
\frac{\dd m_t}{\dd t}
&=
-\alpha_1m_t+\alpha_1\bar g_t .
\end{align}
Multiplying by $e^{\alpha_1t}$ gives
\begin{align}
e^{\alpha_1t}\frac{\dd m_t}{\dd t}
+
\alpha_1e^{\alpha_1t}m_t
&=
\alpha_1e^{\alpha_1t}\bar g_t
\nonumber\\
\frac{\dd}{\dd t}
\left(
e^{\alpha_1t}m_t
\right)
&=
\alpha_1e^{\alpha_1t}\bar g_t .
\end{align}
Integrating from $0$ to $t$,
\begin{align}
e^{\alpha_1t}m_t-m_0
&=
\int_0^t
\alpha_1e^{\alpha_1u}\bar g_t\,\dd u
\nonumber\\
&=
\bar g_t
\left(
e^{\alpha_1t}-1
\right).
\end{align}
Using $m_0=0$, we obtain
\begin{align}
m_t
&=
\bar g_t
\left(
1-e^{-\alpha_1t}
\right).
\label{eq:meanfield_mt_solution}
\end{align}

Similarly, because $\overline{g\odot g}_t$ varies slowly on the $v_t$
relaxation time scale, we freeze $\overline{g\odot g}_t$ when solving
\eqref{eq:meanfield_vt_dynamics}. Hence
\begin{align}
\frac{\dd v_t}{\dd t}
&=
-\alpha_2v_t+\alpha_2\overline{g\odot g}_t .
\end{align}
Multiplying by $e^{\alpha_2t}$ gives
\begin{align}
e^{\alpha_2t}\frac{\dd v_t}{\dd t}
+
\alpha_2e^{\alpha_2t}v_t
&=
\alpha_2e^{\alpha_2t}\overline{g\odot g}_t
\nonumber\\
\frac{\dd}{\dd t}
\left(
e^{\alpha_2t}v_t
\right)
&=
\alpha_2e^{\alpha_2t}\overline{g\odot g}_t .
\end{align}
Integrating from $0$ to $t$,
\begin{align}
e^{\alpha_2t}v_t-v_0
&=
\int_0^t
\alpha_2e^{\alpha_2u}\overline{g\odot g}_t\,\dd u
\nonumber\\
&=
\overline{g\odot g}_t
\left(
e^{\alpha_2t}-1
\right).
\end{align}
Therefore,
\begin{align}
v_t
&=
\overline{g\odot g}_t
+
\left(
v_0-\overline{g\odot g}_t
\right)e^{-\alpha_2t}.
\label{eq:meanfield_vt_solution_general}
\end{align}
Using $v_0=0$, this becomes
\begin{align}
v_t
&=
\overline{g\odot g}_t
\left(
1-e^{-\alpha_2t}
\right).
\label{eq:meanfield_vt_solution}
\end{align}
Combining equation~\eqref{eq:meanfield_mt_solution} and
\eqref{eq:meanfield_vt_solution}, we get
\begin{align}
m_t\oslash\sqrt{v_t}
&=
\bar g_t\bigl(1-e^{-\alpha_1 t}\bigr)
\oslash
\sqrt{\overline{g\odot g}_t (1 - e^{-\alpha_2t})}.
\end{align}
For large $t$ on the $v_t$ relaxation time scale, $\sqrt{
1-e^{-\alpha_2t}}$ decays faster than $1-e^{-\alpha_1t}$, so that
\begin{align}
m_t\oslash\sqrt{v_t}
&\approx
\bar g_t\bigl(1-e^{-\alpha_1 t}\bigr)
\oslash
\sqrt{\overline{g\odot g}_t}.
\end{align}
\end{proof}

%% file: proofs/proof_preconditioned_decomposition_of_theta_SDE.tex
\begin{lemma}[Preconditioned-Decomposition of Adam's $\theta$-SDE]
\label{lem:preconditioned_theta_sde_appendix}
Let $S_t=(\theta_t^\top,m_t^\top,v_t^\top)^\top$ evolve under Adam's joint-state SDE. Define
\begin{align}
&a(t)
:=
\sqrt{1-e^{-\alpha_2 t}},
\qquad
B(t)
:=
\frac{\sqrt{1-e^{-\alpha_2 t}}}
{1-e^{-\alpha_1 t}},\\
&\pi(\theta_t)
:=
\mathrm{diag}\!\left(\overline{g\odot g}_t\right)^{-1/2},
\qquad
G(\theta_t)=\frac{1}{\sqrt{b}}B(t)\pi(\theta_t)\Sigma_t^{1/2}.
\end{align}
where $\pi(\theta_t)$ is referred to as the preconditioner of Adam's SDE and $G(\theta_t)$ is referred to as the preconditioned diffusion factor. By dropping higher correction orders, Adam's parameter dynamics admit the preconditioned decomposition
\begin{align}
\dd\theta_t
\approx
-a(t)\pi(\theta_t)\bar g_t\dd t
+
\mathcal R_{\theta,\mathrm{SM}}(t)\dd t
+
\sqrt{\eta}\,G(\theta_t)\dd W_t
,
\label{eq:preconditioned_theta_sde_with_residual_appendix}
\end{align}
where 
\begin{align}
\mathcal R_{\theta,\mathrm{SM}}(t)
:=
-B(t)
\left[
\mathbb{E}\left[m_t\oslash\sqrt{v_t}\right]
-
\bar g_t\bigl(1-e^{-\alpha_1 t}\bigr)
\oslash
\sqrt{\overline{g\odot g}_t}
\right]
\label{eq:theta_slow_manifold_residual_def_appendix}
,
\end{align}
is referred to as the slow-manifold residual, which represents the residual term to the preconditioned term. This result immediately recovers standard SGD's SDE with $B(t)=1,m_t=g_t,v_t=\mathbf{1}_p,\pi(\theta_t)=I_p$.
\end{lemma}

\begin{proof}
We decompose the $\theta$-component of Adam's joint-state SDE
\begin{align}
\dd\theta_t
=
-B(t)\left(m_t\oslash\sqrt{v_t}\right)\dd t
\label{eq:theta_exact_adam_start_appendix}
,
\end{align}
where
\begin{align}
B(t)
=
\frac{\sqrt{1-e^{-\alpha_2 t}}}{1-e^{-\alpha_1 t}}
,
\qquad
a(t)
=
\sqrt{1-e^{-\alpha_2 t}} .
\end{align}
The SDE preconditioner is defined as
\begin{align}
\pi(\theta_t)
:=
\mathrm{diag}\!\left(\overline{g\odot g}_t\right)^{-1/2}.
\end{align}

\paragraph{Mean-Field Decomposition.}
By Lemma~\ref{lem:meanfield_mt_and_vt_limit}~(\nameref{lem:meanfield_mt_and_vt_limit}), consider the evolution
\begin{align}
m_t\oslash\sqrt{v_t}
\to
\bar g_t\bigl(1-e^{-\alpha_1 t}\bigr)
\oslash
\sqrt{\overline{g\odot g}_t}
,
\end{align}
we add and subtract the mean-field limit
\begin{align}
\bar g_t\bigl(1-e^{-\alpha_1 t}\bigr)
\oslash
\sqrt{\overline{g\odot g}_t}
\end{align}
from equation~\eqref{eq:theta_exact_adam_start_appendix}, then
\begin{align}
\dd\theta_t
&=
-
B(t)\left(m_t\oslash\sqrt{v_t}\right)\dd t \nonumber\\
&=
-
\underbrace{
B(t)
\left[
\bar g_t\bigl(1-e^{-\alpha_1 t}\bigr)
\oslash
\sqrt{\overline{g\odot g}_t}
\right]
}_{\text{mean-field component}}
\dd t
\nonumber\\
&\quad
-
\underbrace{
B(t)
\left[
m_t\oslash\sqrt{v_t}
-
\bar g_t\bigl(1-e^{-\alpha_1 t}\bigr)
\oslash
\sqrt{\overline{g\odot g}_t}
\right]
}_{\text{mean-field residual}}
\dd t .
\label{eq:theta_add_subtract_slow_manifold_appendix}
\end{align}
Since
\begin{align}
B(t)\bigl(1-e^{-\alpha_1 t}\bigr)
=
a(t),
\label{eq:theta_bias_correction_identity_appendix}
\end{align}
and
\begin{align}
\bar g_t
\oslash
\sqrt{\overline{g\odot g}_t}
=
\pi(\theta_t)\bar g_t,
\label{eq:theta_preconditioner_identity_appendix}
\end{align}
the first term (mean-field component) in equation~\eqref{eq:theta_add_subtract_slow_manifold_appendix} becomes
\begin{align}
-B(t)
\left[
\bar g_t\bigl(1-e^{-\alpha_1 t}\bigr)
\oslash
\sqrt{\overline{g\odot g}_t}
\right]\dd t
=
-a(t)\pi(\theta_t)\bar g_t\dd t .
\label{eq:theta_preconditioned_drift_derived_appendix}
\end{align}

\paragraph{Martingale Decomposition of Mean-Field Residual.}
Let
\begin{align}
\mathbb E_t[\cdot]
:=
\mathbb E[\cdot\mid \theta_t].
\end{align}
Define
\begin{align}
\mathcal R_{\theta,\mathrm{SM}}(t)
:=
-B(t)
\left[
\mathbb E_t\!\left[m_t\oslash\sqrt{v_t}\right]
-
\bar g_t\bigl(1-e^{-\alpha_1 t}\bigr)
\oslash
\sqrt{\overline{g\odot g}_t}
\right]
\label{eq:theta_slow_manifold_residual_def_appendix_2}
\end{align}
as the small-manifold residual. Then the second term in equation~\eqref{eq:theta_add_subtract_slow_manifold_appendix} admits a martingale decomposition
\begin{align}
&-B(t)
\left[
m_t\oslash\sqrt{v_t}
-
\bar g_t\bigl(1-e^{-\alpha_1 t}\bigr)
\oslash
\sqrt{\overline{g\odot g}_t}
\right]\dd t
\nonumber\\
&\qquad
=
\mathcal R_{\theta,\mathrm{SM}}(t)\dd t
-
B(t)
\left[
m_t\oslash\sqrt{v_t}
-
\mathbb E_t\!\left[m_t\oslash\sqrt{v_t}\right]
\right]\dd t .
\label{eq:theta_residual_mean_centered_split_appendix}
\end{align}

\paragraph{Effective Diffusion.}
To analyze the martingale residual
\begin{align}
    m_t\oslash\sqrt{v_t}
-
\mathbb E_t\!\left[m_t\oslash\sqrt{v_t}\right]
,
\end{align}
define
\begin{align}
\bar m_t
:=
\bigl(1-e^{-\alpha_1 t}\bigr)\bar g_t,
\qquad
\bar v_t
:=
\overline{g\odot g}_t,
\end{align}
and
\begin{align}
\delta m_t
:=
m_t-\bar m_t,
\qquad
\delta v_t
:=
v_t-\bar v_t .
\end{align}
A componentwise Taylor expansion of $m\oslash\sqrt v$ around $(\bar m_t,\bar v_t)$ gives
\begin{align}
m_t\oslash\sqrt{v_t}
&=
\bar m_t\oslash\sqrt{\bar v_t}
+
\pi(\theta_t)\delta m_t
-
\frac12
\left(
\bar m_t\oslash \bar v_t^{3/2}
\right)\odot \delta v_t
+
\mathcal Q_{\theta}(t),
\label{eq:theta_taylor_adaptive_normalization_appendix}
\end{align}
where
\begin{align}
\mathcal Q_{\theta}(t)
=
O\!\left(
\|\delta m_t\|\,\|\delta v_t\|
+
\|\delta v_t\|_2^2
\right).
\label{eq:theta_taylor_remainder_order_appendix}
\end{align}
Taking the conditional mean of equation~\eqref{eq:theta_taylor_adaptive_normalization_appendix} and subtracting it from equation~\eqref{eq:theta_taylor_adaptive_normalization_appendix}, we obtain
\begin{align}
m_t\oslash\sqrt{v_t}
-
\mathbb E_t\!\left[m_t\oslash\sqrt{v_t}\right]
&=
\pi(\theta_t)
\left(
\delta m_t-\mathbb E_t[\delta m_t]
\right)
\nonumber\\
&\quad
-
\frac12
\left(
\bar m_t\oslash \bar v_t^{3/2}
\right)
\odot
\left(
\delta v_t-\mathbb E_t[\delta v_t]
\right)
\nonumber\\
&\quad
+
\mathcal Q_{\theta}(t)
-
\mathbb E_t[\mathcal Q_{\theta}(t)] .
\label{eq:centered_adaptive_direction_expansion_appendix}
\end{align}

\paragraph{Variation Residual of First-Moment Estimate.}
We now analyze the variation residual of first-moment estimate
\begin{align}
\delta m_t-\mathbb E_t[\delta m_t]
.
\end{align}
From Adam's joint-state SDE, the first-moment component satisfies
\begin{align}
\dd m_t
=
-\alpha_1 m_t\dd t
+
\alpha_1 g_t\dd t .
\label{eq:first_moment_sde_appendix}
\end{align}
Rearranging equation~\eqref{eq:first_moment_sde_appendix} gives
\begin{align}
m_t\dd t
=
g_t\dd t
-
\frac{1}{\alpha_1}\dd m_t .
\label{eq:mt_dt_identity_appendix}
\end{align}
Taking conditional expectations in equation~\eqref{eq:mt_dt_identity_appendix} gives
\begin{align}
\mathbb E_t[m_t]\dd t
=
\bar g_t\dd t
-
\frac{1}{\alpha_1}\dd\mathbb E_t[m_t] .
\label{eq:emt_dt_identity_appendix}
\end{align}
Subtracting equation~\eqref{eq:emt_dt_identity_appendix} from equation~\eqref{eq:mt_dt_identity_appendix} yields
\begin{align}
\left(m_t-\mathbb E_t[m_t]\right)\dd t
=
\left(g_t-\bar g_t\right)\dd t
-
\frac{1}{\alpha_1}\dd\left(m_t-\mathbb E_t[m_t]\right).
\label{eq:centered_mt_dt_identity_appendix}
\end{align}
Since
\begin{align}
\delta m_t-\mathbb E_t[\delta m_t]
=
m_t-\mathbb E_t[m_t],
\label{eq:delta_m_centered_identity_appendix}
\end{align}
we get
\begin{align}
\left(
\delta m_t-\mathbb E_t[\delta m_t]
\right)\dd t
=
\left(g_t-\bar g_t\right)\dd t
-
\frac{1}{\alpha_1}
\dd\left(
\delta m_t-\mathbb E_t[\delta m_t]
\right).
\label{eq:delta_m_centered_dt_identity_appendix}
\end{align}
In continuous-time diffusion scaling,
\begin{align}
\left(g_t-\bar g_t\right)\dd t
=
\sqrt{\eta/b}\,
\Sigma_t^{1/2}\dd W_t .
\label{eq:centered_minibatch_gradient_diffusion_appendix}
\end{align}
Substituting equation~\eqref{eq:centered_minibatch_gradient_diffusion_appendix} into equation~\eqref{eq:delta_m_centered_dt_identity_appendix} gives
\begin{align}
\left(
\delta m_t-\mathbb E_t[\delta m_t]
\right)\dd t
=
\sqrt{\eta/b}\,
\Sigma_t^{1/2}\dd W_t
-
\frac{1}{\alpha_1}
\dd\left(
\delta m_t-\mathbb E_t[\delta m_t]
\right).
\label{eq:first_moment_centered_fluctuation_appendix}
\end{align}
Define
\begin{align}
\dd\mathcal R_{m}(t)
:=
-
\frac{1}{\alpha_1}
\dd\left(
\delta m_t-\mathbb E_t[\delta m_t]
\right).
\label{eq:first_moment_higher_order_remainder_def_appendix}
\end{align}
Then
\begin{align}
\left(
\delta m_t-\mathbb E_t[\delta m_t]
\right)\dd t
=
\sqrt{\eta/b}\,
\Sigma_t^{1/2}\dd W_t
+
\dd\mathcal R_{m}(t).
\label{eq:first_moment_centered_fluctuation_final_appendix}
\end{align}

Substituting equation~\eqref{eq:first_moment_centered_fluctuation_final_appendix} into the first term of equation~\eqref{eq:centered_adaptive_direction_expansion_appendix} gives
\begin{align}
&-B(t)\pi(\theta_t)
\left(
\delta m_t-\mathbb E_t[\delta m_t]
\right)\dd t
\nonumber\\
&\qquad
=
-B(t)\sqrt{\eta/b}\,
\pi(\theta_t)\Sigma_t^{1/2}\dd W_t
-
B(t)\pi(\theta_t)\dd\mathcal R_{m}(t).
\label{eq:theta_leading_martingale_extraction_appendix}
\end{align}
Absorbing the sign into $W_t$, and using the definition
\begin{align}
G(\theta_t)
:=
\frac{1}{\sqrt{b}}B(t)\pi(\theta_t)\Sigma_t^{1/2},
\label{eq:theta_effective_diffusion_factor_appendix}
\end{align}
we obtain
\begin{align}
-B(t)\pi(\theta_t)
\left(
\delta m_t-\mathbb E_t[\delta m_t]
\right)\dd t
=
\sqrt{\eta}\,G(\theta_t)\dd W_t
-
B(t)\pi(\theta_t)\dd\mathcal R_{m}(t).
\label{eq:theta_leading_martingale_final_appendix}
\end{align}

\paragraph{Produce Claims.}
Combining equations~\eqref{eq:centered_adaptive_direction_expansion_appendix} and~\eqref{eq:theta_leading_martingale_final_appendix}, we get
\begin{align}
&-B(t)
\left[
m_t\oslash\sqrt{v_t}
-
\mathbb E_t\!\left[m_t\oslash\sqrt{v_t}\right]
\right]\dd t
\nonumber\\
&\qquad
=
\sqrt{\eta}\,G(\theta_t)\dd W_t
+
\dd\mathcal W_{\theta}(t),
\label{eq:theta_centered_residual_martingale_split_appendix}
\end{align}
where
\begin{align}
\dd\mathcal W_{\theta}(t)
&:=
-B(t)\pi(\theta_t)\dd\mathcal R_{m}(t)
\nonumber\\
&\quad
+
\frac{B(t)}{2}
\left[
\left(
\bar m_t\oslash \bar v_t^{3/2}
\right)
\odot
\left(
\delta v_t-\mathbb E_t[\delta v_t]
\right)
\right]\dd t
\nonumber\\
&\quad
-
B(t)
\left[
\mathcal Q_{\theta}(t)
-
\mathbb E_t[\mathcal Q_{\theta}(t)]
\right]\dd t .
\label{eq:theta_higher_order_martingale_residual_appendix}
\end{align}

Combining equations~\eqref{eq:theta_preconditioned_drift_derived_appendix}, \eqref{eq:theta_residual_mean_centered_split_appendix}, and~\eqref{eq:theta_centered_residual_martingale_split_appendix}, we obtain
\begin{align}
\dd\theta_t
=
-a(t)\pi(\theta_t)\bar g_t\dd t
+
\mathcal R_{\theta,\mathrm{SM}}(t)\dd t
+
\sqrt{\eta}\,G(\theta_t)\dd W_t
+
\dd\mathcal W_{\theta}(t).
\label{eq:preconditioned_theta_sde_with_residual_exact_appendix}
\end{align}

Dropping higher-order terms and keeping the leading diffusion order gives
\begin{align}
\dd\theta_t
\approx
-a(t)\pi(\theta_t)\bar g_t\dd t
+
\mathcal R_{\theta,\mathrm{SM}}(t)\dd t
+
\sqrt{\eta}\,G(\theta_t)\dd W_t
 .
\label{eq:preconditioned_theta_sde_with_residual_derived_appendix}
\end{align}
This is the claimed preconditioned decomposition.
\end{proof}

\begin{remark}
In particular, in late-stage training with sufficiently large batch size $b$, the reduction of $m_t,v_t$
\begin{align}
m_t
\to
\bar g_t\bigl(1-e^{-\alpha_1 t}\bigr),
\qquad
v_t
\to
\overline{g\odot g}_t,
\end{align}
is referred to as the slow-manifold reduction. When the residual $\mathcal R_{\theta,\mathrm{SM}}(t)$ is negligible, then the reduced preconditioned $\theta$-SDE can further be simplified into
\begin{align}
\dd\theta_t
\approx
-a(t)\pi(\theta_t)\bar g_t\dd t
+
\sqrt{\eta}\,G(\theta_t)\dd W_t.
\label{eq:reduced_preconditioned_theta_sde_appendix}
\end{align}

\end{remark}

\begin{remark}[Quick Sanity Check]
As a sanity check, if $\pi(\theta_t)=I_p$, $B(t)\to 1$, and the residual is ignored at leading order, then
\begin{align}
G(\theta_t)
\to
\frac{1}{\sqrt{b}}\Sigma_t^{1/2},
\end{align}
and the preconditioned $\theta$-SDE reduces to
\begin{align}
\dd\theta_t
=
-\bar g_t\dd t
+
\sqrt{\eta/b}\,\Sigma_t^{1/2}\dd W_t,
\end{align}
which is the standard continuous-time SDE approximation of mini-batch SGD.
\end{remark}

%% file: proofs/proof_preconditioned_decomposition_of_radius_SDE.tex
\begin{lemma}[Preconditioned-Decomposition of Adam's Radius SDE]
\label{lem:preconditioned_radius_sde_appendix}
Let $S_t=(\theta_t^\top,m_t^\top,v_t^\top)^\top$ evolve under Adam's joint-state SDE, and let $r_t^2:=\|\theta_t\|_2^2$. Define
\begin{align}
&a(t)
:=
\sqrt{1-e^{-\alpha_2 t}},
\qquad B(t)
:=
\frac{\sqrt{1-e^{-\alpha_2 t}}}
{1-e^{-\alpha_1 t}} \nonumber \\
&\pi(\theta)
:=
\mathrm{diag}\!\left(\overline{g\odot g}(\theta)\right)^{-1/2},
\qquad
G(\theta_t)=\frac{1}{\sqrt{b}}B(t)\pi(\theta_t)\Sigma_t^{1/2}.
\end{align}
Then the residual-corrected preconditioned squared-radius dynamics are
\begin{align}
\dd r_t^2
\approx
\left[
-2a(t)\theta_t^\top\pi(\theta_t)\bar g_t
+
\mathcal R_{\mathrm{SM}}(t)
+
\eta\,
\mathrm{tr}\!\left(G(\theta_t)G(\theta_t)^\top\right)
\right]\dd t
+
2\sqrt{\eta}\,\theta_t^\top G(\theta_t)\dd W_t ,
\label{eq:preconditioned_radius_sde_with_residual_appendix}
\end{align}
where the slow-manifold residual is
\begin{align}
\mathcal R_{\mathrm{SM}}(t)
:=
-2B(t)\theta_t^\top
\left[
\mathbb{E}\left[m_t\oslash\sqrt{v_t}\right]
-
\bar g_t\bigl(1-e^{-\alpha_1 t}\bigr)
\oslash
\sqrt{\overline{g\odot g}_t}
\right].
\label{eq:slow_manifold_residual_def_precond_appendix}
\end{align}
\end{lemma}

\begin{proof}
Let
\begin{align}
E_\theta
:=
\begin{pmatrix}
I_p & 0 & 0\\
0 & 0 & 0\\
0 & 0 & 0
\end{pmatrix}
\in \mathbb R^{3p\times 3p}
\end{align}
be a projector. Since $S_t=(\theta_t^\top,m_t^\top,v_t^\top)^\top$, the squared radius can be written as the joint-state quadratic form
\begin{align}
r_t^2
=
\|\theta_t\|_2^2
=
S_t^\top E_\theta S_t.
\label{eq:radius_joint_quadform_appendix}
\end{align}
Therefore,
\begin{align}
\nabla_S r_t^2
=
2E_\theta S_t
=
\begin{pmatrix}
2\theta_t\\
0\\
0
\end{pmatrix},
\qquad
\nabla_S^2 r_t^2
=
2E_\theta.
\label{eq:joint_radius_grad_hess_appendix}
\end{align}

\begin{remark}
Since $\theta_t$ is the $\theta$-component of the joint state $S_t=(\theta_t^\top,m_t^\top,v_t^\top)^\top$, and the components of $S_t$ are coupled through Adam's dynamics, we should regard $\|\theta_t\|_2^2$ as a function of the full joint state rather than of an isolated variable $\theta_t$. Thus, when applying It\^o's lemma to the joint-state SDE, we introduce the projection map $E_\theta:S_t\mapsto\theta_t$ and write the squared radius as the bilinear form $S_t^\top E_\theta S_t$.
\end{remark}

\paragraph{Radius Dynamics.}
Applying It\^o's lemma to $r_t^2=S_t^\top E_\theta S_t$ under Adam's joint-state SDE gives
\begin{align}
\dd r_t^2
&=
(\nabla_S r_t^2)^\top \dd S_t
+
\frac12
\mathrm{tr}
\left[
\left(\frac{\eta}{b}\sigma(S_t)\sigma(S_t)^\top\right)
\nabla_S^2 r_t^2
\right]\dd t.
\label{eq:joint_radius_ito_start_appendix}
\end{align}
The joint-state diffusion has zero $\theta$-block, because the Brownian noise enters only through the $(m,v)$ components. Since $\nabla_S^2 r_t^2=2E_\theta$ only selects the $\theta$-block, the second-order term vanishes:
\begin{align}
\mathrm{tr}
\left[
\left(\frac{\eta}{b}\sigma(S_t)\sigma(S_t)^\top\right)
\nabla_S^2 r_t^2
\right]
=
0.
\label{eq:joint_radius_ito_second_order_zero_appendix}
\end{align}
Thus,
\begin{align}
\dd r_t^2
&=
(\nabla_S r_t^2)^\top \dd S_t
\nonumber\\
&=
2\theta_t^\top \dd\theta_t.
\label{eq:joint_radius_reduces_to_theta_appendix}
\end{align}
From the $\theta$-component of Adam's joint-state SDE,
\begin{align}
\dd\theta_t
=
-B(t)\left(m_t\oslash\sqrt{v_t}\right)\dd t,
\label{eq:theta_component_dynamics_appendix}
\end{align}
we obtain the exact induced squared-radius identity
\begin{align}
\dd r_t^2
=
-2B(t)\theta_t^\top
\left(m_t\oslash\sqrt{v_t}\right)\dd t.
\label{eq:exact_radius_identity_appendix}
\end{align}

\paragraph{Mean-Field Decomposition.} By Lemma~\ref{lem:meanfield_mt_and_vt_limit}~(\nameref{lem:meanfield_mt_and_vt_limit}), consider the mean-field limit
\begin{align}
m_t\oslash\sqrt{v_t}
\to
\bar g_t\bigl(1-e^{-\alpha_1 t}\bigr)
\oslash
\sqrt{\overline{g\odot g}_t}
,
\end{align}
we add and subtract
\begin{align}
\bar g_t\bigl(1-e^{-\alpha_1 t}\bigr)
\oslash
\sqrt{\overline{g\odot g}_t}
\end{align}
inside the Adam drift:
\begin{align}
-2B(t)\theta_t^\top
\left(m_t\oslash\sqrt{v_t}\right)
&=
-2B(t)\theta_t^\top
\left[
\bar g_t\bigl(1-e^{-\alpha_1 t}\bigr)
\oslash
\sqrt{\overline{g\odot g}_t}
\right]
\nonumber\\
&\quad
-2B(t)\theta_t^\top
\left[
m_t\oslash\sqrt{v_t}
-
\bar g_t\bigl(1-e^{-\alpha_1 t}\bigr)
\oslash
\sqrt{\overline{g\odot g}_t}
\right].
\label{eq:radius_drift_add_subtract_appendix}
\end{align}
Using the continuous-time bias-correction factor
\begin{align}
B(t)
=
\frac{\sqrt{1-e^{-\alpha_2 t}}}{1-e^{-\alpha_1 t}},
\end{align}
and the definition $a(t)=\sqrt{1-e^{-\alpha_2 t}}$, we have
\begin{align}
B(t)\bigl(1-e^{-\alpha_1 t}\bigr)
=
a(t).
\label{eq:B_times_a_identity_appendix}
\end{align}
Moreover, since
\begin{align}
\pi(\theta_t)
=
\mathrm{diag}\!\left(\overline{g\odot g}_t\right)^{-1/2},
\end{align}
we have the element-wise identity
\begin{align}
\bar g_t
\oslash
\sqrt{\overline{g\odot g}_t}
=
\pi(\theta_t)\bar g_t.
\label{eq:preconditioner_identity_appendix}
\end{align}
Therefore, the first term in equation~\eqref{eq:radius_drift_add_subtract_appendix} becomes
\begin{align}
-2B(t)\theta_t^\top
\left[
\bar g_t\bigl(1-e^{-\alpha_1 t}\bigr)
\oslash
\sqrt{\overline{g\odot g}_t}
\right]
&=
-2a(t)\theta_t^\top\pi(\theta_t)\bar g_t.
\label{eq:preconditioned_drift_part_appendix}
\end{align}

\paragraph{Martingale Residual of Mean-Field Decomposition.}
The second term in equation~\eqref{eq:radius_drift_add_subtract_appendix} is exactly the slow-manifold residual:
\begin{align}
\mathcal R_{\mathrm{SM}}(t)
:=
-2B(t)\theta_t^\top
\left[
\mathbb{E}\left[
m_t\oslash\sqrt{v_t}\right]
-
\bar g_t\bigl(1-e^{-\alpha_1 t}\bigr)
\oslash
\sqrt{\overline{g\odot g}_t}
\right].
\label{eq:slow_manifold_residual_derived_appendix}
\end{align}
Combining equations~\eqref{eq:radius_drift_add_subtract_appendix}--\eqref{eq:slow_manifold_residual_derived_appendix}, the exact induced squared-radius has a drift
\begin{align}
\left[
-2a(t)\theta_t^\top\pi(\theta_t)\bar g_t
+
\mathcal R_{\mathrm{SM}}(t)
\right]\dd t.
\label{eq:exact_radius_decomposed_appendix}
\end{align}

\paragraph{Effective Diffusion.}
Finally, in the reduced preconditioned diffusion approximation, the mini-batch gradient fluctuations induce the $\theta$-diffusion term $\sqrt{\eta}\,G(\theta_t)\dd W_t$. Applying It\^o's lemma to this diffusion contribution gives the additional correction
\begin{align}
\eta\,
\mathrm{tr}\!\left(G(\theta_t)G(\theta_t)^\top\right)\dd t
+
2\sqrt{\eta}\,\theta_t^\top G(\theta_t)\dd W_t.
\label{eq:diffusion_correction_appendix}
\end{align}

\paragraph{Produce Claims.}
Adding equation~\eqref{eq:diffusion_correction_appendix} to equation~\eqref{eq:exact_radius_decomposed_appendix} yields
\begin{align}
\dd r_t^2
\approx
\left[
-2a(t)\theta_t^\top\pi(\theta_t)\bar g_t
+
\mathcal R_{\mathrm{SM}}(t)
+
\eta\,
\mathrm{tr}\!\left(G(\theta_t)G(\theta_t)^\top\right)
\right]\dd t
+
2\sqrt{\eta}\,\theta_t^\top G(\theta_t)\dd W_t.
\end{align}
This proves equation~\eqref{eq:preconditioned_radius_sde_with_residual_appendix}.
\end{proof}

%% file: proofs/proof_memorization_regime_identities.tex
\begin{lemma}[Memorization-Regime Preconditioner and Effective Diffusion Identities]
\label{lem:late_stage_preconditioned_diffusion_isotropy}
Let
\begin{align}
\pi(\theta_t)
:=
\mathrm{diag}\!\left(\overline{g\odot g}_t\right)^{-1/2},
\end{align}
and define
\begin{align}
G(\theta_t)
:=
\frac{1}{\sqrt b}B(t)\pi(\theta_t)\Sigma_t^{1/2},
\qquad
B(t)
:=
\frac{\sqrt{1-e^{-\alpha_2 t}}}{1-e^{-\alpha_1 t}}.
\end{align}
Assume that, in the memorization regime,
\begin{align}
\overline{g\odot g}_t
=
\bar g_t\odot\bar g_t
+
\frac{1}{b}\mathrm{diag}(\Sigma_t)
\approx
\frac{1}{b}\mathrm{diag}(\Sigma_t),
\label{eq:late_stage_second_moment_noise_floor}
\end{align}
and that the diagonal Adam preconditioner admits the trace-matched scalar approximation
\begin{align}
\pi(\theta_t)
\approx
s(\theta_t)^{-1}I_p,
\qquad
\frac{1}{s(\theta)}
:=
\frac{1}{p}\mathrm{tr}\bigl(\pi(\theta)\bigr).
\label{eq:trace_matched_scalar_preconditioner}
\end{align}
Equivalently, under the corresponding scalar covariance approximation,
\begin{align}
\frac{1}{b}\Sigma_t
\approx
s(\theta_t)^2I_p .
\label{eq:late_stage_scalar_covariance_matching}
\end{align}
Then the following two late-stage identities hold:
\begin{align}
\mathrm{tr}(\pi(\theta_t))
&\approx
\sqrt b\,
\mathrm{tr}\!\bigl[\left(
\mathrm{diag}(\Sigma_t)
\right)^{-1/2}\bigr],
\label{eq:late_stage_trace_pi_identity}
\\
\mathrm{tr}\!\left(G(\theta_t)G(\theta_t)^\top\right)
&\approx
p.
\label{eq:late_stage_trace_GGt_identity}
\end{align}
Moreover,
\begin{align}
G(\theta_t)G(\theta_t)^\top
\approx
I_p
\qquad
(t\to\infty).
\label{eq:late_stage_GGt_identity}
\end{align}
\end{lemma}

\begin{proof}
From equation~\eqref{eq:late_stage_second_moment_noise_floor},
\begin{align}
\overline{g\odot g}_t
\approx
\frac{1}{b}\mathrm{diag}(\Sigma_t).
\end{align}
Therefore
\begin{align}
\pi(\theta_t)
&=
\mathrm{diag}\!\left(\overline{g\odot g}_t\right)^{-1/2}
\nonumber\\
&\approx
\mathrm{diag}\!\left(
\frac{1}{b}\mathrm{diag}(\Sigma_t)
\right)^{-1/2}
\nonumber\\
&=
\sqrt b\,
\mathrm{diag}\!\left(
\mathrm{diag}(\Sigma_t)
\right)^{-1/2}.
\label{eq:pi_noise_floor_expansion}
\end{align}
Taking traces gives
\begin{align}
\mathrm{tr}(\pi(\theta_t))
&\approx
\sqrt b\,
\mathrm{tr}\!\bigl[\left(
\mathrm{diag}(\Sigma_t)
\right)^{-1/2}\bigr].
\end{align}
This proves equation~\eqref{eq:late_stage_trace_pi_identity}.

Next, by definition,
\begin{align}
G(\theta_t)
=
\frac{1}{\sqrt b}B(t)\pi(\theta_t)\Sigma_t^{1/2}.
\end{align}
Hence
\begin{align}
G(\theta_t)G(\theta_t)^\top
&=
\frac{B(t)^2}{b}
\pi(\theta_t)
\Sigma_t^{1/2}
\left(\Sigma_t^{1/2}\right)^\top
\pi(\theta_t)^\top
\nonumber\\
&=
\frac{B(t)^2}{b}
\pi(\theta_t)
\Sigma_t
\pi(\theta_t)^\top .
\label{eq:GGt_exact_late_stage_appendix}
\end{align}
Since $\pi(\theta_t)$ is diagonal, it is symmetric. Thus
\begin{align}
\pi(\theta_t)^\top
=
\pi(\theta_t),
\end{align}
and equation~\eqref{eq:GGt_exact_late_stage_appendix} becomes
\begin{align}
G(\theta_t)G(\theta_t)^\top
=
\frac{B(t)^2}{b}
\pi(\theta_t)
\Sigma_t
\pi(\theta_t).
\label{eq:GGt_exact_symmetric_late_stage_appendix}
\end{align}
Using the trace-matched scalar approximation
\begin{align}
\pi(\theta_t)
\approx
s(\theta_t)^{-1}I_p,
\end{align}
we obtain
\begin{align}
G(\theta_t)G(\theta_t)^\top
&\approx
\frac{B(t)^2}{b}
\left(s(\theta_t)^{-1}I_p\right)
\Sigma_t
\left(s(\theta_t)^{-1}I_p\right)
\nonumber\\
&=
\frac{B(t)^2}{b\,s(\theta_t)^2}
\Sigma_t.
\label{eq:GGt_trace_matched_scalar_appendix}
\end{align}
By
\begin{align}
\frac{1}{b}\Sigma_t
\approx
s(\theta_t)^2I_p,
\end{align}
or equivalently,
\begin{align}
\Sigma_t
\approx
b\,s(\theta_t)^2I_p.
\label{eq:sigma_late_stage_matching_appendix}
\end{align}
Substituting equation~\eqref{eq:sigma_late_stage_matching_appendix} into equation~\eqref{eq:GGt_trace_matched_scalar_appendix} gives
\begin{align}
G(\theta_t)G(\theta_t)^\top
&\approx
\frac{B(t)^2}{b\,s(\theta_t)^2}
b\,s(\theta_t)^2I_p
\nonumber\\
&=
B(t)^2I_p .
\label{eq:GGt_B_squared_identity_appendix}
\end{align}
Finally,
\begin{align}
B(t)^2
=
\frac{1-e^{-\alpha_2 t}}{\left(1-e^{-\alpha_1 t}\right)^2}
\to
1
\qquad
(t\to\infty).
\end{align}
Therefore,
\begin{align}
G(\theta_t)G(\theta_t)^\top
\approx
I_p
\qquad
(t\to\infty).
\end{align}
Taking traces yields
\begin{align}
\mathrm{tr}\!\left(G(\theta_t)G(\theta_t)^\top\right)
\approx
\mathrm{tr}(I_p)
=
p.
\end{align}
This proves equations~\eqref{eq:late_stage_trace_GGt_identity} and~\eqref{eq:late_stage_GGt_identity}.
\end{proof}

%% file: proofs/proof_preconditioned_latestage_radius_SDE.tex
\begin{lemma}[Reduced Late-Stage Radius SDE with Slow-Manifold, Preconditioned Residual (Restated)]
\label{lem:reduced_late_stage_radius_sde_restated}
Let $r_t:=\|\theta_t\|_2$, where $r_t^2$ follows the residual-corrected preconditioned radius SDE in equation~\eqref{eq:preconditioned_radius_sde_with_residual_appendix} of Lemma~\ref{lem:preconditioned_radius_sde_appendix}~(\nameref{lem:preconditioned_radius_sde_appendix}). In the late-stage regime, assume
\begin{align}
B(t)
:=
\frac{\sqrt{1-e^{-\alpha_2 t}}}
{1-e^{-\alpha_1 t}} \to 1,
\qquad
G(\theta_t)G(\theta_t)^\top\approx I_p,
\label{eq:late_stage_zero_residual_observations_appendix}
\end{align}
so that the residuals from $a(t)$ and the diffusion covariance are neglected. Then the reduced late-stage squared-radius SDE is
\begin{align}
\dd r_t^2
\approx
\left[
-\frac{2\lambda}{s(\theta_t)}r_t^2
-
\frac{2}{s(\theta_t)}
\theta_t^\top\bigl(\bar g_t-\lambda\theta_t\bigr)
+
\eta\,p
+
\mathcal R_{\mathrm{SM}}(t)
+
\mathcal R_{\pi}(t)
\right]\dd t
+
2\sqrt{\eta}\,r_t\,\dd W_t^{(r)},
\label{eq:reduced_late_stage_radius_sde_restated}
\end{align}
where
\begin{align}
\mathcal R_{\mathrm{SM}}(t)
&:=
-2B(t)\theta_t^\top
\left[
\mathbb{E}\left[
m_t\oslash\sqrt{v_t}\right]
-
\bar g_t\bigl(1-e^{-\alpha_1 t}\bigr)
\oslash
\sqrt{\overline{g\odot g}_t}
\right],
\label{eq:slow_manifold_residual_def_late_stage_appendix}
\\
\mathcal R_{\pi}(t)
&:=
-2\theta_t^\top
\left[
\pi(\theta_t)-s(\theta_t)^{-1}I_p
\right]
\bar g_t,
\label{eq:precond_residual_def_late_stage_appendix} \\
W_t^{(r)}
&:=
\int_0^t e_r(\theta_u)^\top G(\theta_u)\dd W_u,
\qquad
e_r(\theta)
:=
\frac{\theta}{\|\theta\|_2},
\label{eq:radial_brownian_def_late_stage_appendix}
\end{align}
and under $G(\theta_t)G(\theta_t)^\top\approx I_p$, $W_t^{(r)}$ is a one-dimensional Brownian motion.

\end{lemma}

\begin{remark}
Moreover,
\begin{align}
\mathcal R_{\mathrm{SM}}(t)+\mathcal R_{\pi}(t)
&=
-2B(t)\theta_t^\top
\mathbb{E}\left[m_t\oslash\sqrt{v_t}\right]
+
2a(t)\theta_t^\top\pi(\theta_t)\bar g_t
\nonumber\\
&\quad
-2\theta_t^\top\pi(\theta_t)\bar g_t
+
\frac{2}{s(\theta_t)}\theta_t^\top\bar g_t
\nonumber\\
&=
-2B(t)\theta_t^\top
\mathbb{E}\left[m_t\oslash\sqrt{v_t}\right]
+
2\bigl(a(t)-1\bigr)\theta_t^\top\pi(\theta_t)\bar g_t
+
\frac{2}{s(\theta_t)}\theta_t^\top\bar g_t.
\label{eq:residual_sum_exact_appendix}
\end{align}
In particular, in the late-stage limit $a(t)\to 1$,
\begin{align}
\mathcal R_{\mathrm{SM}}(t)+\mathcal R_{\pi}(t)
\approx
-2B(t)\theta_t^\top
\mathbb{E}\left[m_t\oslash\sqrt{v_t}\right]
+
\frac{2}{s(\theta_t)}\theta_t^\top\bar g_t
.
\label{eq:residual_sum_late_stage_appendix}
\end{align}
\end{remark}

\begin{proof}
Starting from the residual-corrected preconditioned radius SDE in equation~\eqref{eq:preconditioned_radius_sde_with_residual_appendix} of Lemma~\ref{lem:preconditioned_radius_sde_appendix}~(\nameref{lem:preconditioned_radius_sde_appendix}),
\begin{align}
\dd r_t^2
&\approx
\left[
-2a(t)\theta_t^\top\pi(\theta_t)\bar g_t
+
\mathcal R_{\mathrm{SM}}(t)
+
\eta\,
\mathrm{tr}\!\left(G(\theta_t)G(\theta_t)^\top\right)
\right]\dd t
\nonumber\\
&\quad+
2\sqrt{\eta}\,\theta_t^\top G(\theta_t)\dd W_t.
\label{eq:late_stage_radius_start_appendix}
\end{align}
Under $a(t)\to 1$,
\begin{align}
-2a(t)\theta_t^\top\pi(\theta_t)\bar g_t
\approx
-2\theta_t^\top\pi(\theta_t)\bar g_t.
\label{eq:late_stage_drop_a_appendix}
\end{align}
Decompose the preconditioner as
\begin{align}
\pi(\theta_t)
=
s(\theta_t)^{-1}I_p
+
\left[
\pi(\theta_t)-s(\theta_t)^{-1}I_p
\right].
\label{eq:precond_decomposition_late_stage_appendix}
\end{align}
Substituting this decomposition gives
\begin{align}
-2\theta_t^\top\pi(\theta_t)\bar g_t
&=
-\frac{2}{s(\theta_t)}\theta_t^\top\bar g_t
+
\mathcal R_{\pi}(t),
\label{eq:derive_precond_residual_late_stage_appendix}
\end{align}
where
\begin{align}
\mathcal R_{\pi}(t)
=
-2\theta_t^\top
\left[
\pi(\theta_t)-s(\theta_t)^{-1}I_p
\right]
\bar g_t.
\label{eq:precond_residual_derived_late_stage_appendix}
\end{align}
Next, decompose the mean gradient as
\begin{align}
\bar g_t
=
\lambda\theta_t
+
\bigl(\bar g_t-\lambda\theta_t\bigr).
\label{eq:mean_gradient_decomposition_late_stage_appendix}
\end{align}
Then
\begin{align}
-\frac{2}{s(\theta_t)}\theta_t^\top\bar g_t
&=
-\frac{2\lambda}{s(\theta_t)}r_t^2
-
\frac{2}{s(\theta_t)}
\theta_t^\top
\bigl(\bar g_t-\lambda\theta_t\bigr).
\label{eq:derive_task_gradient_term_late_stage_appendix}
\end{align}
Combining equations~\eqref{eq:derive_precond_residual_late_stage_appendix} and~\eqref{eq:derive_task_gradient_term_late_stage_appendix}, we obtain
\begin{align}
-2a(t)\theta_t^\top\pi(\theta_t)\bar g_t
\approx
-\frac{2\lambda}{s(\theta_t)}r_t^2
-
\frac{2}{s(\theta_t)}
\theta_t^\top
\bigl(\bar g_t-\lambda\theta_t\bigr)
+
\mathcal R_{\pi}(t).
\label{eq:late_stage_drift_with_precond_residual_appendix}
\end{align}

Under $G(\theta_t)G(\theta_t)^\top\approx I_p$,
\begin{align}
\eta\,
\mathrm{tr}\!\left(G(\theta_t)G(\theta_t)^\top\right)
\approx
\eta\,p.
\label{eq:late_stage_ito_trace_reduction_appendix}
\end{align}
Also,
\begin{align}
2\sqrt{\eta}\,\theta_t^\top G(\theta_t)\dd W_t
&=
2\sqrt{\eta}\,r_t e_r(\theta_t)^\top G(\theta_t)\dd W_t
\nonumber\\
&=
2\sqrt{\eta}\,r_t\,\dd W_t^{(r)}.
\label{eq:late_stage_martingale_reduction_appendix}
\end{align}
The quadratic variation is
\begin{align}
\langle W^{(r)}\rangle_t
&=
\int_0^t
e_r(\theta_u)^\top
G(\theta_u)G(\theta_u)^\top
e_r(\theta_u)\dd u
\nonumber\\
&\approx
\int_0^t
e_r(\theta_u)^\top e_r(\theta_u)\dd u
=
t.
\label{eq:late_stage_radial_brownian_qv_appendix}
\end{align}
Thus, by L\'evy's characterization~\citep{karatzas1991brownian}, $W_t^{(r)}$ is a one-dimensional Brownian motion.

Substituting equations~\eqref{eq:late_stage_drift_with_precond_residual_appendix}, \eqref{eq:late_stage_ito_trace_reduction_appendix}, and~\eqref{eq:late_stage_martingale_reduction_appendix} into equation~\eqref{eq:late_stage_radius_start_appendix} yields equation~\eqref{eq:reduced_late_stage_radius_sde_restated}
\begin{align}
\dd r_t^2
\approx
\left[
-\frac{2\lambda}{s(\theta_t)}r_t^2
-
\frac{2}{s(\theta_t)}
\theta_t^\top\bigl(\bar g_t-\lambda\theta_t\bigr)
+
\eta\,p
+
\mathcal R_{\mathrm{SM}}(t)
+
\mathcal R_{\pi}(t)
\right]\dd t
+
2\sqrt{\eta}\,r_t\,\dd W_t^{(r)}.
\end{align}

\end{proof}

\begin{remark}
Empirical study in Figure~\ref{fig:late_stage_radius_sde_appendix} shows that $\mathcal R_{\mathrm{SM}}(t) + \mathcal R_{\pi}(t)$ is negligible during the memorization regime, so that the radius SDE in grokking can be approximated through
\begin{align}
\dd r_t^2
\approx
\left[
-\frac{2\lambda}{s(\theta_t)}r_t^2
-
\frac{2}{s(\theta_t)}
\theta_t^\top\bigl(\bar g_t-\lambda\theta_t\bigr)
+
\eta\,p
\right]\dd t
+
2\sqrt{\eta}\,r_t\,\dd W_t^{(r)}.
\end{align}

We now analyze the cancellation of $\mathcal R_{\mathrm{SM}}(t) + \mathcal R_{\pi}(t)$. Expanding $\mathcal R_{\mathrm{SM}}(t)$ gives
\begin{align}
\mathcal R_{\mathrm{SM}}(t)
&=
-2B(t)\theta_t^\top
\mathbb{E}\left[m_t\oslash\sqrt{v_t}\right]
\nonumber\\
&\quad
+
2B(t)\theta_t^\top
\left[
\bar g_t\bigl(1-e^{-\alpha_1 t}\bigr)
\oslash
\sqrt{\overline{g\odot g}_t}
\right].
\label{eq:slow_manifold_residual_expanded_appendix}
\end{align}
Using
\begin{align}
B(t)\bigl(1-e^{-\alpha_1 t}\bigr)
=
a(t),
\qquad
\bar g_t\oslash\sqrt{\overline{g\odot g}_t}
=
\pi(\theta_t)\bar g_t,
\end{align}
we obtain
\begin{align}
\mathcal R_{\mathrm{SM}}(t)
=
-2B(t)\theta_t^\top
\mathbb{E}\left[m_t\oslash\sqrt{v_t}\right]
+
2a(t)\theta_t^\top\pi(\theta_t)\bar g_t.
\label{eq:slow_manifold_residual_expanded_precond_appendix}
\end{align}
Adding the explicit expression for $\mathcal R_{\pi}(t)$,
\begin{align}
\mathcal R_{\mathrm{SM}}(t)+\mathcal R_{\pi}(t)
&=
-2B(t)\theta_t^\top\mathbb{E}\left[m_t\oslash\sqrt{v_t}\right]
+
2a(t)\theta_t^\top\pi(\theta_t)\bar g_t
\nonumber\\
&\quad
-2\theta_t^\top
\left[
\pi(\theta_t)-s(\theta_t)^{-1}I_p
\right]
\bar g_t
\nonumber\\
&=
-2B(t)\theta_t^\top\mathbb{E}\left[m_t\oslash\sqrt{v_t}\right]
+
2\bigl(a(t)-1\bigr)\theta_t^\top\pi(\theta_t)\bar g_t
+
\frac{2}{s(\theta_t)}\theta_t^\top\bar g_t.
\label{eq:residual_cancellation_derivation_appendix}
\end{align}
This proves equation~\eqref{eq:residual_sum_exact_appendix}. In the late-stage limit $a(t), B(t)\to 1$, the middle term vanishes, and therefore
\begin{align}
\mathcal R_{\mathrm{SM}}(t)+\mathcal R_{\pi}(t)
\approx
-2B(t)\theta_t^\top\mathbb{E}\left[m_t\oslash\sqrt{v_t}\right]
+
\frac{2}{s(\theta_t)}\theta_t^\top\bar g_t.
\end{align}

Hence the residual sum is small when gradient and its first-moment estimate are near zeros.

\end{remark}

%% file: proofs/proof_rho_M.tex
\begin{theorem}[Scaling Law of Memorization Radius (Restated)]
\label{thm:memorization_radius_scaling_restated}
Consider the joint Adam SDE in equation~\eqref{eq:adam_joint_sde_assembled_appendix} of Lemma~\ref{lem:adam_joint_sde_restated}~(\nameref{lem:adam_joint_sde_restated})
on the joint state $S_t=(\theta_t^\top,m_t^\top,v_t^\top)^\top\in\mathbb R^{3p}$,
\begin{align}
\dd S_t
=
\mu(S_t)\dd t
+
\sqrt{\eta/b}\,\sigma(S_t)\dd W_t,
\end{align}
and assume that the bias-correction factor satisfies $B(t)\to 1$ in the early hitting regime. Let
\begin{align}
\tau_M(S_0)
:=
\inf\{t\ge 0:\theta_t(S_0) \in \partial M\}
\end{align}
be the first-hitting time of the memorization boundary. Define the exit-value function
\begin{align}
u(S)
:=
\mathbb E\!\left[
\|\theta_{\tau_M(S_0)}\|_2^2
\,\middle|\,
S_0=S
\right],
\label{eq:memorization_exit_value_function_restated}
\end{align}
and the mean squared memorization radius
\begin{align}
\rho_M^2:=
\mathbb E_{S_0\sim\Theta_S}\!\left[u(S_0)\right],
\label{eq:rho_M_exit_value_def_restated}
\end{align}
where $\Theta_S$ is the initialization distribution of the joint state, with $S_0=(\theta_0^\top,0_p^\top,0_p^\top)^\top$. 

Let $E_\theta\in\mathbb R^{3p\times 3p}$ be the $\theta$-block projector
\begin{align}
E_\theta
:=
\begin{pmatrix}
I_p & 0 & 0\\
0 & 0 & 0\\
0 & 0 & 0
\end{pmatrix},
\label{eq:theta_projector_memorization_restated}
\end{align}
so that for any joint state $S=(\theta,m,v)$,
\begin{align}
\|\theta\|_2^2
=
S^\top E_\theta S.
\end{align}

Let $\Phi_t(S_0)$ denote the deterministic Adam flow generated by the drift field $\mu(S)$,
\begin{align}
\frac{\dd}{\dd t}\Phi_t(S_0)
=
\mu(\Phi_t(S_0)),
\qquad
\Phi_0(S_0)=S_0.
\label{eq:deterministic_adam_flow_restated}
\end{align}
Define the deterministic first-hitting time
\begin{align}
\tau_M^{(0)}(S_0)
:=
\inf\{t\ge 0: \theta(\Phi_t(S_0))\in\partial M\},
\label{eq:tau_M_det_exit_time_thm_appendix}
\end{align}
and the deterministic exit-radius-squared map
\begin{align}
R^2(S_0)
:=
\Phi_{\tau_M^{(0)}(S_0)}(S_0)^\top
E_\theta
\Phi_{\tau_M^{(0)}(S_0)}(S_0).
\label{eq:exit_radius_squared_map_thm_appendix}
\end{align}
Let
\begin{align}
\Sigma_S(S)
:=
\sigma(S)\sigma(S)^\top
\label{eq:joint_diffusion_def_appendix}
\end{align}
be the joint-state diffusion covariance. Then, for sufficiently small $\eta/b$,
\begin{align}
\rho_M^2 =
\rho_M^{(0)\,2}
+
(\eta/b)c_M
+
(\eta/b)^2c_M^{(2)}
+
O\bigl((\eta/b)^3\bigr),
\label{eq:law_rho_M_restated}
\end{align}
where
\begin{align}
\rho_M^{(0)\,2}
&:=
\mathbb E_{S_0\sim\Theta_S}
\!\left[
R^2(S_0)
\right],
\label{eq:rho_M0_def_restated}
\\
c_M
&:=
\mathbb E_{S_0\sim\Theta_S}
\int_0^{\tau_M^{(0)}(S_0)}
\frac12\,
\mathrm{tr}\!\left(
\Sigma_S(\Phi_t(S_0))
\nabla_S^2R^2(\Phi_t(S_0))
\right)
\dd t,
\label{eq:cM_def_restated}
\\
c_M^{(2)}
&:=
\mathbb E_{S_0\sim\Theta_S}
\int_0^{\tau_M^{(0)}(S_0)}
\frac12\,
\mathrm{tr}\!\left(
\Sigma_S(\Phi_t(S_0))
\nabla_S^2u_1(\Phi_t(S_0))
\right)
\dd t.
\label{eq:cM2_def_restated}
\end{align}
Here $u_1$ is the first-order perturbation corrector
\begin{align}
u_1(S_0)
:=
\int_0^{\tau_M^{(0)}(S_0)}
\frac12\,
\mathrm{tr}\!\left(
\Sigma_S(\Phi_t(S_0))
\nabla_S^2R^2(\Phi_t(S_0))
\right)
\dd t.
\label{eq:u1_line_integral_restated}
\end{align}
The constants $\rho_M^{(0)\,2}$, $c_M$, and $c_M^{(2)}$ are task-determined and independent of $\eta$ and $b$; $\ell_2$ regularization coefficient $\lambda$ enters $(\rho_M^{(0)})^2$ through deterministic gradient flow with scaling law $(\rho_M^{(0)})^2
\propto
\exp\!\left(-O(\lambda)\right)$.
\end{theorem}

\begin{proof}
Let
\begin{align}
\tau_M(S_0) := \inf\{t\ge 0:\theta_t(S_0) \in \partial M\}
\label{eq:tau_M_hitting_time_proof_appendix}
\end{align}
be the first-hitting time of the memorization boundary for the trajectory $\{\theta_t\}$ starting from $S_0=(\theta_0^\top,0_p^\top,0_p^\top)^\top$.

\paragraph{Bilinear Radius Projector.} Since the parameter $\theta_t$ is the $\theta$-component of the joint state $S_t$, define the $\theta$-block projector
\begin{align}
E_\theta
:=
\begin{pmatrix}
I_p & 0 & 0\\
0 & 0 & 0\\
0 & 0 & 0
\end{pmatrix}
\in\mathbb R^{3p\times 3p}.
\label{eq:theta_projector_memorization_proof_appendix}
\end{align}
Thus, for any joint state $S$, the squared parameter radius is the joint-state quadratic form
\begin{align}
r^2(S)
:=
\|\theta\|_2^2
=
S^\top E_\theta S.
\label{eq:joint_state_radius_def_memorization_proof_appendix}
\end{align}
In particular,
\begin{align}
\nabla_S r^2(S)=2E_\theta S,
\qquad
\nabla_S^2 r^2(S)=2E_\theta.
\label{eq:joint_state_radius_grad_hess_memorization_proof_appendix}
\end{align}

\paragraph{Define Dirichlet Boundary Problem.} We introduce the exit-value function
\begin{align}
u(S)
:=
\mathbb{E}\!\left[r^2(S_{\tau_M})\,\middle|\,S_0=S\right]
=
\mathbb{E}\!\left[S_{\tau_M}^{\top}E_\theta S_{\tau_M}\,\middle|\,S_0=S\right],
\label{eq:exit_value_function_def_appendix}
\end{align}
which is the expected squared parameter norm at the first-hitting time of $\partial M$, starting from $S$.

We now derive the PDE satisfied by $u$. The joint Adam SDE has the form
\begin{align}
\dd S_t
=
\mu(S_t)\,\dd t
+
\sqrt{\eta/b}\,\sigma(S_t)\,\dd W_t.
\label{eq:joint_sde_proof_local_appendix}
\end{align}
Therefore its infinitesimal generator is

\begin{figure}[h]
    \centering

\begin{align}
\mathcal{L}_S f(S)
=
\mu(S)^\top\nabla_S f(S)
+
\frac12\,\tikzmarknode{etab}{\boxed{\frac{\eta}{b}}}\,
\mathrm{tr}\!\left(
\Sigma_S(S)\nabla_S^2 f(S)
\right),
\qquad
\Sigma_S(S):=\sigma(S)\sigma(S)^\top .
\label{eq:generator_def_appendix}
\end{align}

\begin{tikzpicture}[overlay,remember picture]
\draw[<-] ([yshift=-2pt]etab.south) -- ++(0,-0.5)  
                                -- ++(4.2,0);   
\node[anchor=north west,align=left,font=\footnotesize]
  at ([yshift=-0.55cm,xshift=2pt]etab.south)
  {Learning rate $\eta$ and batch size $b$\\ modulate infinitesimal generator};
\end{tikzpicture}

\vspace{3\baselineskip}

\end{figure}

\begin{remark}
The structure of the infinitesimal generator $\mathcal{L}_S f(S)$ shows that
this operator contains a term modulated by the learning rate $\eta$ and the
batch size $b$ through the coefficient $\eta/b$. This motivates us to treat
$\mathcal{L}_S f(S)$ as an operator expansion with respect to the coefficient
$\eta/b$, so that we can apply a regular perturbation expansion to study the
corresponding expansion of the memorization radius with respect to $\eta/b$.
\end{remark}

\paragraph{Dirichlet Boundary PDE.} Since $u(S)$ is the expected boundary value of the process stopped at $\partial M$, it solves the Dirichlet boundary value problem~\citep{oksendal2003sde}
\begin{align}
\mathcal{L}_S u(S) &= 0,
\qquad &&\theta\in M^\circ,
\label{eq:dirichlet_pde_appendix}\\
u(S) &= S^\top E_\theta S,
\qquad &&\theta\in\partial M.
\label{eq:dirichlet_bc_appendix}
\end{align}
Here equation~\eqref{eq:dirichlet_pde_appendix} states that $u(S_t)$ is harmonic with respect to the stopped Adam diffusion before hitting the boundary, while equation~\eqref{eq:dirichlet_bc_appendix} assigns the squared radius as the boundary payoff.

Let $\varepsilon:=\eta/b$. We split the generator into its deterministic and stochastic parts:
\begin{align}
\mathcal{L}_S
=
\mathcal{L}_S^{(0)}
+
\varepsilon \mathcal{L}_S^{(1)},
\qquad
\mathcal{L}_S^{(0)}f
:=
\mu^\top\nabla_S f,
\qquad
\mathcal{L}_S^{(1)}f
:=
\frac12\,
\mathrm{tr}\!\left(
\Sigma_S\nabla_S^2 f
\right).
\label{eq:generator_split_appendix}
\end{align}
Thus the Dirichlet PDE becomes
\begin{align}
\left(\mathcal{L}_S^{(0)}+\varepsilon\mathcal{L}_S^{(1)}\right)u(S)=0,
\qquad \theta\in M^\circ,
\label{eq:dirichlet_pde_eps_appendix}
\end{align}
with boundary condition
\begin{align}
u(S)=S^\top E_\theta S,
\qquad \theta\in\partial M.
\label{eq:dirichlet_bc_eps_appendix}
\end{align}

\paragraph{Solving PDE via Regular Perturbation Expansion.} 
We solve this PDE by a regular perturbation expansion in the small parameter $\varepsilon$. Write
\begin{align}
u(S)
=
u_0(S)
+
\varepsilon u_1(S)
+
\varepsilon^2 u_2(S)
+
O(\varepsilon^3).
\label{eq:perturbation_ansatz_appendix}
\end{align}
Substituting equation~\eqref{eq:perturbation_ansatz_appendix} into equation~\eqref{eq:dirichlet_pde_eps_appendix} gives
\begin{align}
0
&=
\left(\mathcal{L}_S^{(0)}+\varepsilon\mathcal{L}_S^{(1)}\right)
\left(u_0+\varepsilon u_1+\varepsilon^2u_2+O(\varepsilon^3)\right)
\nonumber\\
&=
\mathcal{L}_S^{(0)}u_0
+
\varepsilon
\left(
\mathcal{L}_S^{(0)}u_1+\mathcal{L}_S^{(1)}u_0
\right)
+
\varepsilon^2
\left(
\mathcal{L}_S^{(0)}u_2+\mathcal{L}_S^{(1)}u_1
\right)
+
O(\varepsilon^3).
\label{eq:perturbation_expanded_pde_appendix}
\end{align}
Since this identity must hold for all sufficiently small $\varepsilon$, each coefficient of $\varepsilon$ must vanish. Therefore,
\begin{align}
\varepsilon^0:\quad
&\mathcal{L}_S^{(0)}u_0=0,
\label{eq:order0_pde_appendix}\\
\varepsilon^1:\quad
&\mathcal{L}_S^{(0)}u_1=-\mathcal{L}_S^{(1)}u_0,
\label{eq:order1_pde_appendix}\\
\varepsilon^2:\quad
&\mathcal{L}_S^{(0)}u_2=-\mathcal{L}_S^{(1)}u_1.
\label{eq:order2_pde_appendix}
\end{align}
The boundary condition is expanded in the same way:
\begin{align}
u_0+\varepsilon u_1+\varepsilon^2u_2+O(\varepsilon^3)
=
S^\top E_\theta S,
\qquad \theta\in\partial M.
\label{eq:perturbation_expanded_bc_appendix}
\end{align}
Matching powers of $\varepsilon$ on the boundary gives
\begin{align}
u_0\big|_{\theta\in\partial M}
&=
S^\top E_\theta S,
\label{eq:order0_bc_appendix}\\
u_1\big|_{\theta\in\partial M}
&=
0,
\label{eq:order1_bc_appendix}\\
u_2\big|_{\theta\in\partial M}
&=
0.
\label{eq:order2_bc_appendix}
\end{align}

\paragraph{Zero-Order Expansion.}
We now solve these equations by characteristics. The characteristic curves of $\mathcal{L}_S^{(0)}=\mu^\top\nabla_S$ are the deterministic Adam trajectories
\begin{align}
\frac{\dd}{\dd t}\Phi_t(S_0)
=
\mu(\Phi_t(S_0)),
\qquad
\Phi_0(S_0)=S_0
\label{eq:deterministic_flow_def_appendix}
,
\end{align}
so that
\begin{align}
\dot\theta^{(0)}
&=
-m^{(0)}\oslash \sqrt{v^{(0)}},
\nonumber\\
\dot m^{(0)}
&=
-\alpha_1\bigl(m^{(0)}-\bar g(\theta^{(0)})\bigr),
\nonumber\\
\dot v^{(0)}
&=
-\alpha_2\bigl(v^{(0)}-\overline{g\odot g}(\theta^{(0)})\bigr).
\label{eq:joint_deterministic_ode_appendix}
\end{align}
For any smooth function $w$, along the deterministic flow we have
\begin{align}
\frac{\dd}{\dd t}w(\Phi_t(S_0))
=
\mathcal{L}_S^{(0)}w(\Phi_t(S_0)).
\label{eq:transport_identity_appendix}
\end{align}

Let $\tau_M^{(0)}(S_0)$ be the deterministic hitting time
\begin{align}
\tau_M^{(0)}(S_0)
:=
\inf\{t\ge 0:(\Phi_t(S_0))_\theta\in\partial M\}.
\label{eq:tau_M_det_hitting_time_proof_appendix}
\end{align}
At order $\varepsilon^0$, equation~\eqref{eq:order0_pde_appendix} gives
\begin{align}
\frac{\dd}{\dd t}u_0(\Phi_t(S_0))=0.
\end{align}
Thus $u_0$ is constant along the deterministic trajectory. Evaluating it at the deterministic hitting time and using the boundary condition in equation~\eqref{eq:order0_bc_appendix}, we obtain
\begin{align}
u_0(S_0)
&=
u_0(\Phi_{\tau_M^{(0)}(S_0)}(S_0))
\nonumber\\
&=
\Phi_{\tau_M^{(0)}(S_0)}(S_0)^\top
E_\theta
\Phi_{\tau_M^{(0)}(S_0)}(S_0)
\nonumber\\
&=
\left\|
\bigl(\Phi_{\tau_M^{(0)}(S_0)}(S_0)\bigr)_\theta
\right\|_2^2.
\label{eq:u0_solution_appendix}
\end{align}
Define the deterministic exit-radius-squared map
\begin{align}
R^2(S_0)
:=
\Phi_{\tau_M^{(0)}(S_0)}(S_0)^\top
E_\theta
\Phi_{\tau_M^{(0)}(S_0)}(S_0).
\label{eq:exit_radius_squared_map_proof_appendix}
\end{align}
Equivalently,
\begin{align}
R^2(S_0)
=
\left\|
\bigl(\Phi_{\tau_M^{(0)}(S_0)}(S_0)\bigr)_\theta
\right\|_2^2.
\label{eq:exit_radius_squared_map_equiv_proof_appendix}
\end{align}
Then $u_0(S_0)=R^2(S_0)$.

\paragraph{First-Order Expansion.}
At order $\varepsilon^1$, equation~\eqref{eq:order1_pde_appendix} and the transport identity give
\begin{align}
\frac{\dd}{\dd t}u_1(\Phi_t(S_0))
=
-\mathcal{L}_S^{(1)}u_0(\Phi_t(S_0)).
\label{eq:u1_transport_appendix}
\end{align}
Integrating from $0$ to $\tau_M^{(0)}(S_0)$ gives
\begin{align}
u_1(\Phi_{\tau_M^{(0)}(S_0)}(S_0))
-
u_1(S_0)
=
-
\int_0^{\tau_M^{(0)}(S_0)}
\mathcal{L}_S^{(1)}u_0(\Phi_t(S_0))
\,\dd t.
\label{eq:u1_integrated_appendix}
\end{align}
Since $u_1=0$ on $\partial M$, the first term on the left vanishes. Therefore,
\begin{align}
u_1(S_0)
=
\int_0^{\tau_M^{(0)}(S_0)}
\mathcal{L}_S^{(1)}u_0(\Phi_t(S_0))
\,\dd t.
\label{eq:u1_before_trace_appendix}
\end{align}
Using the definition of $\mathcal{L}_S^{(1)}$ and $u_0=R^2$ along the deterministic flow,
\begin{align}
u_1(S_0)
=
\int_0^{\tau_M^{(0)}(S_0)}
\frac12\,
\mathrm{tr}\!\left(
\Sigma_S(\Phi_t(S_0))
\nabla_S^2 R^2(\Phi_t(S_0))
\right)
\,\dd t.
\label{eq:u1_line_integral_appendix}
\end{align}

\paragraph{Second-Order Expansion.}
At order $\varepsilon^2$, the same argument gives
\begin{align}
\frac{\dd}{\dd t}u_2(\Phi_t(S_0))
=
-\mathcal{L}_S^{(1)}u_1(\Phi_t(S_0)).
\label{eq:u2_transport_appendix}
\end{align}
Since $u_2=0$ on $\partial M$, integrating along the deterministic characteristic yields
\begin{align}
u_2(S_0)
=
\int_0^{\tau_M^{(0)}(S_0)}
\mathcal{L}_S^{(1)}u_1(\Phi_t(S_0))
\,\dd t.
\label{eq:u2_before_trace_appendix}
\end{align}
Thus,
\begin{align}
u_2(S_0)
=
\int_0^{\tau_M^{(0)}(S_0)}
\frac12\,
\mathrm{tr}\!\left(
\Sigma_S(\Phi_t(S_0))
\nabla_S^2 u_1(\Phi_t(S_0))
\right)
\,\dd t.
\label{eq:u2_line_integral_appendix}
\end{align}

\paragraph{Produce Claims.} 
Combining the three orders, we have
\begin{align}
u(S_0)
=
R^2(S_0)
+
\varepsilon u_1(S_0)
+
\varepsilon^2u_2(S_0)
+
O(\varepsilon^3).
\label{eq:u_expansion_full_appendix}
\end{align}

Finally, the mean squared memorization radius is obtained by averaging this exit-value function over the initialization distribution:
\begin{align}
\rho_M^2&:=
\mathbb{E}\!\left[\|\theta_{\tau_M}\|_2^2\right]
=
\mathbb{E}\!\left[S_{\tau_M}^{\top}E_\theta S_{\tau_M}\right]
=
\mathbb{E}_{S_0\sim\Theta_S}\!\left[u(S_0)\right]
\nonumber\\
&=
\mathbb{E}_{S_0\sim\Theta_S}\!\left[R^2(S_0)\right]
+
\varepsilon
\mathbb{E}_{S_0\sim\Theta_S}\!\left[u_1(S_0)\right]
+
\varepsilon^2
\mathbb{E}_{S_0\sim\Theta_S}\!\left[u_2(S_0)\right]
+
O(\varepsilon^3).
\label{eq:rho_M_average_appendix}
\end{align}
Define
\begin{align}
\rho_M^{(0)\,2}
&:=
\mathbb{E}_{S_0\sim\Theta_S}\!\left[R^2(S_0)\right],
\nonumber\\
c_M
&:=
\mathbb{E}_{S_0\sim\Theta_S}\!\left[u_1(S_0)\right],
\nonumber\\
c_M^{(2)}
&:=
\mathbb{E}_{S_0\sim\Theta_S}\!\left[u_2(S_0)\right].
\label{eq:rho_M_constants_identification_appendix}
\end{align}
Substituting these definitions into equation~\eqref{eq:rho_M_average_appendix} gives
\begin{align}
\rho_M^2=
\rho_M^{(0)\,2}
+
\varepsilon c_M
+
\varepsilon^2 c_M^{(2)}
+
O(\varepsilon^3).
\label{eq:rho_M_eps_law_appendix}
\end{align}
Since $\varepsilon=\eta/b$, we obtain
\begin{align}
\rho_M^2=
\rho_M^{(0)\,2}
+
(\eta/b)c_M
+
(\eta/b)^2c_M^{(2)}
+
O\bigl((\eta/b)^3\bigr),
\end{align}
which proves the theorem.
\end{proof}

\paragraph{Dependence of $\rho_M^{(0)}$ on $\lambda$ under weak task gradient.}
The leading constant $(\rho_M^{(0)})^2$ is determined by the deterministic
gradient flow. In the regime where the task gradient is small,
\begin{align}
\nabla\mathscr L_f^*(\theta_t)
&\approx
0,
\end{align}
the deterministic flow is dominated by the $\ell_2$ drift:
\begin{align}
\dd\theta_t
&\approx
-\lambda\theta_t\,\dd t .
\end{align}
Hence
\begin{align}
\dd\|\theta_t\|_2^2
&=
2\theta_t^\top\dd\theta_t
\nonumber\\
&\approx
-2\lambda\|\theta_t\|_2^2\,\dd t .
\end{align}
Equivalently,
\begin{align}
\frac{1}{\|\theta_t\|_2^2}\dd\|\theta_t\|_2^2
&\approx
-2\lambda\,\dd t .
\end{align}
Integrating along the deterministic memorization path gives
\begin{align}
\|\theta_{\tau_M^{(0)}}\|_2^2
&\approx
\|\theta_0\|_2^2
\exp\!\left(-2\lambda\tau_M^{(0)}\right).
\end{align}
Therefore,
\begin{align}
(\rho_M^{(0)})^2
&\propto
\exp\!\left(-O(\lambda\tau_M^{(0)})\right).
\end{align}
When $\tau_M^{(0)}$ is weakly dependent on $\lambda$ in the considered regime,
\begin{align}
(\rho_M^{(0)})^2
&\propto
\exp\!\left(-O(\lambda)\right).
\end{align}
Since
\begin{align}
\rho_M^2
=
(\rho_M^{(0)})^2
+
O(\eta),
\end{align}
the same leading dependence gives
\begin{align}
\rho_M^2
&\propto
\exp\!\left(-O(\lambda)\right).
\end{align}


%% file: proofs/proof_rho_G.tex
\begin{theorem}[Scaling Law of Generalization Radius, restated]
\label{thm:generalization_radius_scaling_restated}
We do not assume a global minizer, for each initialization $S_0\in\Theta_S$, let
\begin{align}
\theta^\star(S_0)
:=
\arg\min_{\theta(S_0)}
\left\{
\mathscr{L}^{*}_f(\theta(S_0))
+
\frac{\lambda}{2}\|\theta(S_0)\|_2^2
\right\} \in G
\label{eq:regularized_minimizer_def_appendix}
\end{align}
denote the regularized local minimizer selected by the trajectory starting from initial state $S_0$, and $\mathscr{L}^{*}_f(\theta(S_0))$ is the task loss. The trajectory $\{S_t\}$ will be confined in the basin centered at $\theta^\star(S_0)$. Define
\begin{align}
\bigl(\rho_G^{(0)}\bigr)^2
:=
\mathbb E_{S_0\in\Theta_S}
\left[
\|\theta^\star(S_0)\|_2^2
\right].
\label{eq:rho_G_landscape_def_appendix}
\end{align}
Let
\begin{align}
\tau_G(S_0)
:=
\inf\{t\ge 0:\theta_t(S_0) \in \partial G\},
\end{align}
and define the generalization radius by
\begin{align}
\rho_G^2:=
\mathbb E_{S_0\in\Theta_S}
\left[
\mathbb E
\left[
\|\theta_{\tau_G(S_0)}\|_2^2
\mid
S_0
\right]
\right].
\label{eq:rho_G_tau_def_appendix}
\end{align}

Let
\begin{align}
s^\star(S_0)
:=
s(\theta^\star(S_0)),
\qquad
D^\star(S_0)
:=
G(\theta^\star(S_0))G(\theta^\star(S_0))^\top,
\label{eq:s_D_star_def_generalization_appendix}
\end{align}
where the effective scalar preconditioner at the selected minimizer is defined by
\begin{align}
\frac{1}{s^\star(S_0)}
=
\frac{1}{p}\mathrm{tr}\bigl(\pi(\theta^\star(S_0))\bigr),
\qquad
\pi(\theta^\star(S_0))
=
\mathrm{diag}\!\left(\overline{g\odot g}(\theta^\star(S_0))\right)^{-1/2}.
\label{eq:effective_preconditioner_def_appendix}
\end{align}
Then, generalization radius approximately admits,
\begin{align}
\rho_G^2 \approx
\bigl(\rho_G^{(0)}\bigr)^2
+
\frac{\eta}{\lambda}\,c_G
+
O\!\left(\frac{\eta^2}{\lambda^2}\right),
\label{eq:rho_G_squared_expansion_restated}
\end{align}
where
\begin{align}
\|\theta^\star(S_0)\|_2^2
\propto
\exp(-O(\lambda))
\end{align}
and
\begin{align}
c_G
:=
\mathbb E_{S_0\in\Theta_S}
\left[
\tfrac12 s^\star(S_0)\,
\mathrm{tr}\!\left(D^\star(S_0)\right)
\right]
>0 .
\label{eq:c_G_bar_def_appendix}
\end{align}
\end{theorem}

\begin{proof}
For a trajectory $\{S_t\}$ starting from $S_0\in\Theta_S$, write
\begin{align}
\theta^\star
&:=
\theta^\star(S_0),
&
\theta_t
&=
\theta^\star+\delta_t .
\end{align}
By definition of the local minimizer,
\begin{align}
\bar g^*(\theta^\star)+\lambda\theta^\star
=
0
\label{eq:generalization_local_minimizer_stationary_condition_appendix}
,
\end{align}
where $\bar g^*(\bullet)$ denotes the task loss mean. By Lemma~\ref{lem:reduced_late_stage_radius_sde_restated}~(\nameref{lem:reduced_late_stage_radius_sde_restated}), together with  Lemma~\ref{lem:late_stage_preconditioned_diffusion_isotropy}~(\nameref{lem:late_stage_preconditioned_diffusion_isotropy}),
\begin{align}
G(\theta_t)G(\theta_t)^\top
\approx
I_p,
\label{eq:generalization_GG_identity_appendix}
\end{align}
and the residual
\begin{align}
\mathcal R_{\mathrm{SM}}(t)+\mathcal R_{\pi}(t)\approx 0,
\end{align}
the reduced late-stage squared-radius SDE gives
\begin{align}
\dd r_t^2
\approx
\left[
-\frac{2}{s(\theta_t)}
\theta_t^\top
\left(
\bar g^*_t+\lambda\theta_t
\right)
+
\eta p
\right]\dd t
+
2\sqrt{\eta}\,r_t\,\dd W_t^{(r)},
\label{eq:generalization_dr2_start_appendix}
\end{align}
where
\begin{align}
r_t^2
=
\|\theta_t\|_2^2 .
\end{align}

\paragraph{Linearization Near Local Minizer.}
The Taylor expansion of $\bar g^*$ around $\theta^\star$ gives
\begin{align}
\bar g^*_t
&=
\bar g^*(\theta^\star+\delta_t)
\nonumber\\
&=
\bar g^*(\theta^\star)
+
\nabla \bar g^*(\theta^\star)\delta_t
+
O\!\left(\|\delta_t\|_2^2\right).
\label{eq:generalization_mean_gradient_taylor_appendix}
\end{align}
Therefore,
\begin{align}
\bar g^*_t+\lambda\theta_t
&=
\bar g^*(\theta^\star)
+
\nabla \bar g^*(\theta^\star)\delta_t
+
\lambda\theta^\star
+
\lambda\delta_t
+
O\!\left(\|\delta_t\|_2^2\right)
\nonumber\\
&=
\left(
\nabla \bar g^*(\theta^\star)+\lambda I_p
\right)\delta_t
+
O\!\left(\|\delta_t\|_2^2\right),
\label{eq:generalization_regularized_drift_linearization_appendix}
\end{align}
where equation~\eqref{eq:generalization_local_minimizer_stationary_condition_appendix} was used. Under the dominant-regularization approximation,
\begin{align}
\nabla \bar g^*(\theta^\star)+\lambda I_p
&\approx
\lambda I_p,
\label{eq:generalization_dominant_regularization_appendix}
\\
\bar g^*_t+\lambda\theta_t
&\approx
\lambda\delta_t
+
O\!\left(\|\delta_t\|_2^2\right).
\label{eq:generalization_local_drift_expansion_appendix}
\end{align}
Using $\theta_t=\theta^\star+\delta_t$,
\begin{align}
\theta_t^\top
\left(
\bar g^*_t+\lambda\theta_t
\right)
&\approx
(\theta^\star+\delta_t)^\top
\left[
\lambda\delta_t
+
O\!\left(\|\delta_t\|_2^2\right)
\right]
\nonumber\\
&=
\lambda\theta^{\star\top}\delta_t
+
\lambda\|\delta_t\|_2^2
+
O\!\left(\|\delta_t\|_2^2\right).
\label{eq:generalization_radial_drift_expansion_appendix}
\end{align}

\paragraph{Local Stationary Approximation.}
Under the local stationary approximation,
\begin{align}
s(\theta_t)
&=
s^\star(S_0)
+
O\!\left(\|\delta_t\|_2\right),
&
s^\star(S_0)
&:=
s(\theta^\star(S_0)).
\label{eq:generalization_sstar_freezing_appendix}
\end{align}
Taking conditional expectation in equation~\eqref{eq:generalization_dr2_start_appendix} gives
\begin{align}
0
&\approx
-\frac{2}{s^\star(S_0)}
\mathbb E
\left[
\theta_t^\top
\left(
\bar g^*_t+\lambda\theta_t
\right)
\mid
S_0
\right]
+
\eta p
+
O\!\left(\frac{\eta^2}{\lambda^2}\right),
\label{eq:generalization_stationary_dr2_balance_appendix}
\end{align}
since
\begin{align}
\mathbb E
\left[
r_t\,\dd W_t^{(r)}
\mid
S_0
\right]
=
0.
\end{align}
The local stationary fluctuation is centered, so
\begin{align}
\mathbb E[\delta_t\mid S_0]
&=
0,
&
\mathbb E[\theta^{\star\top}\delta_t\mid S_0]
&=
0.
\label{eq:generalization_centered_fluctuation_appendix}
\end{align}
Combining equations~\eqref{eq:generalization_radial_drift_expansion_appendix},
\eqref{eq:generalization_stationary_dr2_balance_appendix}, and
\eqref{eq:generalization_centered_fluctuation_appendix},
\begin{align}
0
&\approx
-\frac{2\lambda}{s^\star(S_0)}
\mathbb E
\left[
\|\delta_t\|_2^2
\mid
S_0
\right]
+
\eta p
+
O\!\left(\frac{\eta^2}{\lambda^2}\right).
\end{align}
Hence
\begin{align}
\mathbb E
\left[
\|\delta_t\|_2^2
\mid
S_0
\right]
\approx
\frac{\eta}{\lambda}
\frac{s^\star(S_0)}{2}p
+
O\!\left(\frac{\eta^2}{\lambda^2}\right).
\label{eq:generalization_delta_second_moment_appendix}
\end{align}
Also,
\begin{align}
\|\theta_t\|_2^2
&=
\|\theta^\star+\delta_t\|_2^2
\nonumber\\
&=
\|\theta^\star\|_2^2
+
2\theta^{\star\top}\delta_t
+
\|\delta_t\|_2^2 .
\end{align}
Taking conditional expectation gives
\begin{align}
\mathbb E
\left[
\|\theta_t\|_2^2
\mid
S_0
\right]
&=
\|\theta^\star(S_0)\|_2^2
+
\mathbb E
\left[
\|\delta_t\|_2^2
\mid
S_0
\right]
\nonumber\\
&\approx
\|\theta^\star(S_0)\|_2^2
+
\frac{\eta}{\lambda}
\frac{s^\star(S_0)}{2}p
+
O\!\left(\frac{\eta^2}{\lambda^2}\right).
\label{eq:generalization_conditional_radius_appendix}
\end{align}

\paragraph{Solving Generalization Radius.}
Under the local stationary approximation in the generalization basin,
\begin{align}
\mathbb E
\left[
\|\theta_{\tau_G(S_0)}\|_2^2
\mid
S_0
\right]
&\approx
\mathbb E
\left[
\|\theta_t\|_2^2
\mid
S_0
\right].
\end{align}
Therefore,
\begin{align}
\mathbb E
\left[
\|\theta_{\tau_G(S_0)}\|_2^2
\mid
S_0
\right]
&\approx
\|\theta^\star(S_0)\|_2^2
+
\frac{\eta}{\lambda}
\frac{s^\star(S_0)}{2}p
+
O\!\left(\frac{\eta^2}{\lambda^2}\right).
\label{eq:generalization_hitting_radius_conditional_appendix}
\end{align}
Averaging over $S_0\in\Theta_S$,
\begin{align}
\rho_G^2
&:=
\mathbb E_{S_0\in\Theta_S}
\left[
\mathbb E
\left[
\|\theta_{\tau_G}\|_2^2
\mid
S_0
\right]
\right]
\nonumber\\
&\approx
\mathbb E_{S_0\in\Theta_S}
\left[
\|\theta^\star(S_0)\|_2^2
\right]
+
\frac{\eta}{\lambda}
\mathbb E_{S_0\in\Theta_S}
\left[
\frac{s^\star(S_0)}{2}p
\right]
+
O\!\left(\frac{\eta^2}{\lambda^2}\right).
\end{align}
By definition,
\begin{align}
\bigl(\rho_G^{(0)}\bigr)^2
&:=
\mathbb E_{S_0\in\Theta_S}
\left[
\|\theta^\star(S_0)\|_2^2
\right],
\\
c_G
&:=
\mathbb E_{S_0\in\Theta_S}
\left[
\frac{s^\star(S_0)}{2}p
\right].
\end{align}
Thus,
\begin{align}
\rho_G^2
\approx
\bigl(\rho_G^{(0)}\bigr)^2
+
\frac{\eta}{\lambda}c_G
+
O\!\left(\frac{\eta^2}{\lambda^2}\right),
\end{align}
which proves the claimed scaling law. 

\paragraph{Regularizer $\lambda$-Dependence of $\rho_G^{(0)}$.}
By definition,
\begin{align}
\bigl(\rho_G^{(0)}\bigr)^2
:=
\mathbb E_{S_0\in\Theta_S}
\left[
\|\theta^\star(S_0)\|_2^2
\right].
\end{align}
The dependence on $\lambda$ enters $(\rho_G^{(0)})^2$ through the deterministic
regularized trajectory selecting the local minimizer $\theta^\star(S_0)$. Along
the deterministic component of the regularized flow,
\begin{align}
\dd \theta_t
\approx
-\left(\nabla\mathscr L_f^*(\theta_t)+\lambda\theta_t\right)\dd t .
\end{align}
When the task-gradient contribution is small in the generalization basin, the
radial dynamics satisfy
\begin{align}
\dd\|\theta_t\|_2^2
=
2\theta_t^\top \dd\theta_t
\approx
-2\lambda\|\theta_t\|_2^2\,\dd t .
\end{align}
Therefore,
\begin{align}
\frac{1}{\|\theta_t\|_2^2}
\dd\|\theta_t\|_2^2
\approx
-2\lambda\,\dd t,
\end{align}
and integration along the deterministic approach to the generalization basin gives
\begin{align}
\|\theta^\star(S_0)\|_2^2
\propto
\exp(-O(\lambda)).
\end{align}
Averaging over $S_0\in\Theta_S$, we obtain
\begin{align}
\bigl(\rho_G^{(0)}\bigr)^2
=
\mathbb E_{S_0\in\Theta_S}
\left[
\|\theta^\star(S_0)\|_2^2
\right]
\propto
\exp(-O(\lambda)).
\end{align}

\paragraph{Batch-Size $b$-Dependence of $c_G$.}
The weak batch-size dependence of $c_G$ enters through the effective scalar
preconditioner $s^\star(S_0)$. From the definition of $c_G$,
\begin{align}
c_G
&:=
\mathbb E_{S_0\in\Theta_S}
\left[
\frac{s^\star(S_0)}{2}p
\right].
\end{align}
By the late-stage preconditioning identity,
\begin{align}
\frac{1}{s^\star(S_0)}
&=
\frac{1}{p}
\mathrm{tr}
\left(
\pi(\theta^\star(S_0))
\right),
\\
\pi(\theta^\star(S_0))
&=
\mathrm{diag}
\left(
\overline{g\odot g}(\theta^\star(S_0))
\right)^{-1/2}.
\end{align}
In the memorization regime,
\begin{align}
\overline{g\odot g}(\theta^\star(S_0))
&\approx
\frac{1}{b}
\mathrm{diag}(\Sigma(\theta^\star(S_0))).
\end{align}
Therefore,
\begin{align}
\pi(\theta^\star(S_0))
&\approx
\sqrt b\,
\mathrm{diag}
\left(
\mathrm{diag}(\Sigma(\theta^\star(S_0)))
\right)^{-1/2},
\\
\frac{1}{s^\star(S_0)}
&=
O(\sqrt b),
\\
s^\star(S_0)
&=
O\!\left(\frac{1}{\sqrt b}\right).
\end{align}
Using the late-stage identity $G(\theta)G(\theta)^\top\approx I_p$, the diffusion trace
contribution is $\mathrm{tr}(G(\theta)G(\theta)^\top)\approx p$. Hence
\begin{align}
c_G
=
\mathbb E_{S_0\in\Theta_S}
\left[
\frac{s^\star(S_0)}{2}p
\right]
=
O\!\left(\frac{1}{\sqrt b}\right).
\end{align}
\end{proof}

%% file: proofs/proof_tau.tex
\begin{theorem}[Scaling Law of Solution Transition Time, restated]
\label{thm:solution_transition_time_scaling_restated}
Assume the reduced late-stage radius SDE in
Lemma~\ref{lem:reduced_late_stage_radius_sde_restated}~(\nameref{lem:reduced_late_stage_radius_sde_restated}), with conditions in Lemma~\ref{lem:late_stage_preconditioned_diffusion_isotropy}~(\nameref{lem:late_stage_preconditioned_diffusion_isotropy})
\begin{align}
G(\theta_t)G(\theta_t)^\top
\approx
I_p,
\end{align}
and the residual approximation
\begin{align}
\mathcal R_{\mathrm{SM}}(t)+\mathcal R_{\pi}(t)\approx 0 .
\label{eq:transition_residual_cancellation_restated}
\end{align}
Let $\bar g(\theta)$ be the population mean gradient and let
$s(\theta)>0$ denote the scalar preconditioning factor satisfying
\begin{align}
\frac{1}{s(\theta)}
=
\frac{1}{p}\mathrm{tr}\bigl(\pi(\theta)\bigr).
\label{eq:s_theta_transition_def_restated}
\end{align}
Let $\theta_t^{(0)}(S_0)$ be the deterministic late-stage trajectory from
$\partial M$ to $\partial G$ starting with initial state $S_0$, and let $\tau_{M\to G}^{(0)}(S_0)$ be its
deterministic transition time. 
For each deterministic late-stage trajectory
$\theta_t^{(0)}(S_0)$ from $\partial M$ to $\partial G$, define its pathwise
effective radial contraction rate by
\begin{align}
\kappa(S_0)
:=
\frac{1}{\tau_{M\to G}^{(0)}(S_0)}
\int_0^{\tau_{M\to G}^{(0)}(S_0)}
\left[
\frac{\lambda}{s(\theta_t^{(0)}(S_0))}
+
\frac{
\theta_t^{(0)\top}(S_0)
\bigl(\bar g(\theta_t^{(0)}(S_0))-\lambda\theta_t^{(0)}(S_0)\bigr)
}{
s(\theta_t^{(0)}(S_0))\|\theta_t^{(0)}(S_0)\|_2^2
}
\right]\dd t .
\end{align}
The scalar $\bar s$ is defined through the ensemble-averaged effective
contraction rate
\begin{align}
\frac{\lambda}{\bar s}
:=
\mathbb{E}_{S_0 \in \Theta_S}
\left[
\kappa(S_0)
\right].
\label{eq:sbar_transition_def_restated}
\end{align}


Let $\rho_M$ and $\rho_G$ be the memorization and generalization radii.
Assume the asymptotic expansions
\begin{align}
\rho_M^2
&=
(\rho_M^{(0)})^2
+
\frac{\eta}{b}c_M
+
O\bigl((\eta/b)^2\bigr),
\label{eq:rho_M_squared_asymptotic_transition_restated}
\\
\rho_G^2
&\approx
\bigl(\rho_G^{(0)}\bigr)^2
+
\frac{\eta}{\lambda}c_G
+
O\!\left(\frac{\eta^2}{\lambda^2}\right),
\label{eq:rho_G_squared_asymptotic_transition_restated}
\end{align}
where $c_M$ is the memorization-radius correction from
Theorem~\ref{thm:memorization_radius_scaling_restated}, and
\begin{align}
\bigl(\rho_G^{(0)}\bigr)^2
&:=
\mathbb E_{S_0\in\Theta_S}
\left[
\|\theta^\star(S_0)\|_2^2
\right],
\label{eq:rho_G_star_transition_def_restated}
\\
\theta^\star(S_0)
&:=
\arg\min_{\theta(S_0)}
\left\{
\mathscr L^{*}_f(\theta(S_0))
+
\frac{\lambda}{2}\|\theta(S_0)\|_2^2
\right\},
\label{eq:theta_star_transition_def_restated}
\\
c_G
&:=
\mathbb E_{S_0\in\Theta_S}
\left[
\frac{s^\star(S_0)}{2}
\mathrm{tr}\!\left(D^\star(S_0)\right)
\right],
\label{eq:c_G_transition_def_restated}
\\
s^\star(S_0)
&:=
s(\theta^\star(S_0)),
\qquad
D^\star(S_0)
:=
G(\theta^\star(S_0))G(\theta^\star(S_0))^\top .
\label{eq:sstar_Dstar_transition_def_restated}
\end{align}
Under the late-stage identity
$G(\theta^\star(S_0))G(\theta^\star(S_0))^\top\approx I_p$,
equation~\eqref{eq:c_G_transition_def_restated} reduces to
\begin{align}
c_G
=
\mathbb E_{S_0\in\Theta_S}
\left[
\frac{s^\star(S_0)}{2}p
\right].
\label{eq:c_G_transition_identity_restated}
\end{align}

Then, in iteration time, the expected transition time from $\partial M$ to
$\partial G$ satisfies
\begin{align}
\mathbb{E}_{S_0 \in \Theta_S}\bigl[\tau_{M\to G}\bigr]
\approx
\frac{\bar s}{\eta\,\lambda}
\log\!\frac{\rho_M^{(0)}}{\rho_G^{(0)}}
+
\frac{c_\tau}{b\,\lambda}
+
\frac{c_\tau^{(2)}}{\lambda^2}
+
O\!\left(\frac{\eta}{\lambda^3}\right),
\label{eq:law_transition_time_restated}
\end{align}
where
\begin{align}
c_\tau
&:=
\bar s\,\frac{c_M}{2(\rho_M^{(0)})^2},
\label{eq:c_tau_components_appendix}
\\
c_\tau^{(2)}
&:=
-\bar s\,\frac{c_G}{2(\rho_G^{(0)})^2}
+
\frac{\bar s^2(p-2)}{4}
\left[
\frac{1}{(\rho_G^{(0)})^2}
-
\frac{1}{(\rho_M^{(0)})^2}
\right].
\label{eq:c_tau2_components_appendix}
\end{align}
The constants $\bar s$, $c_\tau$, and $c_\tau^{(2)}$ are independent of
$\eta$ and $\lambda$ in the considered scaling regime; a weak batch-size $b$-dependence enters through the scalar preconditioning scale $\bar s = O(\frac{1}{\sqrt{b}})$, and hence $c_\tau = O(\frac{1}{\sqrt{b}})$ and $c_\tau^{(2)} = O(\frac{1}{b})$.
\end{theorem}

\begin{proof}
Let
\begin{align}
r_t:=\|\theta_t\|_2 .
\end{align}
By Lemma~\ref{lem:reduced_late_stage_radius_sde_restated}~(\nameref{lem:reduced_late_stage_radius_sde_restated}), under the late-stage observations
\begin{align}
a(t)\to 1,
\qquad
G(\theta_t)G(\theta_t)^\top\approx I_p,
\qquad
\mathcal R_{\mathrm{SM}}(t)+\mathcal R_{\pi}(t)\approx 0,
\end{align}
the squared-radius process satisfies
\begin{align}
\dd r_t^2
\approx
\left[
-\frac{2\lambda}{s(\theta_t)}r_t^2
-
\frac{2}{s(\theta_t)}
\theta_t^\top\bigl(\bar g_t-\lambda\theta_t\bigr)
+
\eta\,p
\right]\dd t
+
2\sqrt{\eta}\,r_t\,\dd W_t^{(r)} .
\label{eq:transition_squared_radius_start_appendix}
\end{align}

\paragraph{Effective Deterministic Radius Flow.}
The leading deterministic radial flow is obtained from
equation~\eqref{eq:transition_squared_radius_start_appendix} by keeping the
$O(1)$ drift terms and dropping the $O(\eta)$ It\^o correction. Thus,
\begin{align}
\dd r_t^2
&\approx
\left[
-\frac{2\lambda}{s(\theta_t)}r_t^2
-
\frac{2}{s(\theta_t)}
\theta_t^\top\bigl(\bar g_t-\lambda\theta_t\bigr)
\right]\dd t .
\end{align}
Since
\begin{align}
\dd r_t^2
=
2r_t\,\dd r_t
+
O(\eta)\dd t,
\end{align}
the leading deterministic radial drift is
\begin{align}
\dot r_t
&\approx
-\frac{\lambda}{s(\theta_t)}r_t
-
\frac{1}{s(\theta_t)r_t}
\theta_t^\top\bigl(\bar g_t-\lambda\theta_t\bigr)
\nonumber\\
&=
-
\left[
\frac{\lambda}{s(\theta_t)}
+
\frac{
\theta_t^\top\bigl(\bar g_t-\lambda\theta_t\bigr)
}{
s(\theta_t)\|\theta_t\|_2^2
}
\right]r_t .
\label{eq:transition_pathwise_radial_drift_appendix}
\end{align}
For each deterministic late-stage trajectory
$\theta_t^{(0)}(S_0)$ from $\partial M$ to $\partial G$, define its pathwise
effective radial contraction rate by
\begin{align}
\kappa(S_0)
:=
\frac{1}{\tau_{M\to G}^{(0)}(S_0)}
\int_0^{\tau_{M\to G}^{(0)}(S_0)}
\left[
\frac{\lambda}{s(\theta_t^{(0)}(S_0))}
+
\frac{
\theta_t^{(0)\top}(S_0)
\bigl(\bar g(\theta_t^{(0)}(S_0))-\lambda\theta_t^{(0)}(S_0)\bigr)
}{
s(\theta_t^{(0)}(S_0))\|\theta_t^{(0)}(S_0)\|_2^2
}
\right]\dd t .
\label{eq:pathwise_kappa_transition_def_appendix}
\end{align}
The scalar $\bar s$ is defined through the ensemble-averaged effective
contraction rate
\begin{align}
\frac{\lambda}{\bar s}
:=
\mathbb{E}_{S_0 \in \Theta_S}
\left[
\kappa(S_0)
\right].
\label{eq:sbar_transition_def_proof_appendix}
\end{align}
Thus, in the averaged effective radial model, the leading deterministic
late-stage radius flow is
\begin{align}
\dot r_t^{(0)}
=
-\frac{\lambda}{\bar s}r_t^{(0)} .
\label{eq:deterministic_radial_flow_appendix}
\end{align}
Therefore, along the averaged late-stage transition dynamics,
\begin{align}
-\frac{\lambda}{s(\theta_t)}r_t
-
\frac{1}{s(\theta_t)r_t}
\theta_t^\top\bigl(\bar g_t-\lambda\theta_t\bigr)
\approx
-\frac{\lambda}{\bar s}r_t .
\label{eq:effective_radial_scale_def_appendix}
\end{align}

\paragraph{Effective Radius SDE.}
We now pass from the squared-radius SDE to the radius SDE. Applying It\^o's
lemma to $r_t=(r_t^2)^{1/2}$ gives
\begin{align}
\dd r_t
=
\frac{1}{2r_t}\dd r_t^2
-
\frac{1}{8r_t^3}\dd\langle r^2\rangle_t .
\label{eq:ito_from_r2_to_r_appendix}
\end{align}
From the martingale term in equation~\eqref{eq:transition_squared_radius_start_appendix},
\begin{align}
\dd\langle r^2\rangle_t
=
4\eta r_t^2\,\dd t .
\label{eq:r2_quadratic_variation_appendix}
\end{align}
Substituting equations~\eqref{eq:transition_squared_radius_start_appendix}
and~\eqref{eq:r2_quadratic_variation_appendix} into
equation~\eqref{eq:ito_from_r2_to_r_appendix}, and using
equation~\eqref{eq:effective_radial_scale_def_appendix}, yields
\begin{align}
\dd r_t
&=
\left[
-\frac{\lambda}{\bar s}r_t
+
\eta\frac{p}{2r_t}
-
\eta\frac{1}{2r_t}
\right]\dd t
+
\sqrt{\eta}\,\dd W_t^{(r)}
\nonumber\\
&=
\left[
-\frac{\lambda}{\bar s}r_t
+
\eta\frac{p-1}{2r_t}
\right]\dd t
+
\sqrt{\eta}\,\dd W_t^{(r)} .
\label{eq:radial_sde_transition_appendix}
\end{align}
Let
\begin{align}
\kappa:=\frac{\lambda}{\bar s},
\qquad
\varepsilon:=\eta .
\label{eq:kappa_eps_def_appendix}
\end{align}
Then equation~\eqref{eq:radial_sde_transition_appendix} becomes
\begin{align}
\dd r_t
=
\left[
-\kappa r_t
+
\frac{\varepsilon(p-1)}{2r_t}
\right]\dd t
+
\sqrt{\varepsilon}\,\dd W_t^{(r)} .
\label{eq:radial_sde_kappa_eps_appendix}
\end{align}

\paragraph{Mean First-Passage Time PDE.}
Let $T(r)$ be the expected continuous time for the process in
equation~\eqref{eq:radial_sde_kappa_eps_appendix}, initialized at radius $r$, to hit
the absorbing boundary $r_G$, defined as
\begin{align}
T(r)
:=
\mathbb E\!\left[\tau_G\mid r_0=r\right],
\qquad
r>r_G .
\label{eq:mfpt_T_def_appendix}
\end{align}
The transition starts at the memorization boundary and ends at the generalization boundary:
\begin{align}
r_0
&:=
\rho_M,
&
r_G
&:=
\rho_G .
\end{align}
Its infinitesimal generator is
\begin{align}
\mathcal L_r[\bullet]
=
\left(
-\kappa r
+
\frac{\varepsilon(p-1)}{2r}
\right)\frac{\dd}{\dd r}[\bullet]
+
\frac{\varepsilon}{2}\frac{\dd^2}{\dd r^2}[\bullet]
.
\label{eq:radial_generator_appendix}
\end{align}
Then $T(r)$ solves a Dirichlet problem
\begin{align}
\mathcal L_rT(r)
=
-1,
\qquad
T(r_G)=0.
\label{eq:mfpt_backward_equation_appendix}
\end{align}

\paragraph{Solving Absorbing Boundary Dirichlet PDE.}
Equivalently,
\begin{align}
\frac{\varepsilon}{2}T''(r)
+
\left(
-\kappa r+\frac{\varepsilon(p-1)}{2r}
\right)T'(r)
=
-1.
\label{eq:backward_mfpt_pde_appendix}
\end{align}
Writing
\begin{align}
q(r):=T'(r),
\end{align}
we obtain
\begin{align}
q'(r)
+
\left(
-\frac{2\kappa}{\varepsilon}r+\frac{p-1}{r}
\right)q(r)
=
-\frac{2}{\varepsilon}.
\label{eq:q_first_order_ode_appendix}
\end{align}
The integrating factor is
\begin{align}
I(r)
=
r^{p-1}
\exp\!\left(-\frac{\kappa}{\varepsilon}r^2\right).
\label{eq:integrating_factor_appendix}
\end{align}
Using
\begin{align}
I(r)q(r)\to 0
\qquad
\text{as }r\to\infty,
\end{align}
we obtain
\begin{align}
q(r)
=
\frac{2}{\varepsilon}
\frac{1}{I(r)}
\int_r^\infty I(y)\dd y .
\label{eq:q_solution_appendix}
\end{align}
Since $T(r_G)=0$, for an initial radius $r_0>r_G$,
\begin{align}
T(r_0)
=
\frac{2}{\varepsilon}
\int_{r_G}^{r_0}
\frac{1}{I(x)}
\int_x^\infty I(y)\dd y\,\dd x.
\label{eq:mfpt_exact_integral_appendix}
\end{align}

Define
\begin{align}
\beta
:=
\frac{\kappa}{\varepsilon}
=
\frac{\lambda}{\bar s\,\eta}.
\label{eq:beta_def_appendix}
\end{align}
Then
\begin{align}
I(r)
=
r^{p-1}e^{-\beta r^2},
\end{align}
and
\begin{align}
\int_x^\infty y^{p-1}e^{-\beta y^2}\dd y
=
\frac{1}{2}\beta^{-p/2}
\Gamma(p/2,\beta x^2),
\label{eq:inner_incomplete_gamma_appendix}
\end{align}
where
\begin{align}
\Gamma(a,z):=\int_z^\infty t^{a-1}e^{-t}\dd t .
\end{align}
Thus,
\begin{align}
T(r_0)
=
\frac{1}{\varepsilon}\beta^{-p/2}
\int_{r_G}^{r_0}
\frac{e^{\beta x^2}}{x^{p-1}}
\Gamma(p/2,\beta x^2)\dd x .
\label{eq:mfpt_gamma_exact_appendix}
\end{align}

\paragraph{Small-Learning-Rate Regime.}
The learning rate $\eta \to 0$ is small, so
\begin{align}
\beta x^2
=
\frac{\lambda}{\bar s\,\eta}x^2
\gg 1 .
\end{align}
Therefore,
\begin{align}
\Gamma(p/2,\beta x^2)
=
(\beta x^2)^{p/2-1}e^{-\beta x^2}
\left[
1+\frac{p-2}{2\beta x^2}
+
O\!\left((\beta x^2)^{-2}\right)
\right].
\label{eq:incomplete_gamma_asymptotic_appendix}
\end{align}
Substituting equation~\eqref{eq:incomplete_gamma_asymptotic_appendix} into
equation~\eqref{eq:mfpt_gamma_exact_appendix} gives
\begin{align}
T(r_0)
&=
\frac{1}{\varepsilon\beta}
\int_{r_G}^{r_0}
\left[
\frac{1}{x}
+
\frac{p-2}{2\beta x^3}
+
O\!\left((\beta x^2)^{-2}x^{-1}\right)
\right]\dd x
\nonumber\\
&=
\frac{1}{\kappa}\log\frac{r_0}{r_G}
+
\frac{\varepsilon(p-2)}{4\kappa^2}
\left[
\frac{1}{r_G^2}
-
\frac{1}{r_0^2}
\right]
+
O\!\left(\frac{\varepsilon^2}{\kappa^3}\right).
\label{eq:mfpt_expansion_kappa_appendix}
\end{align}
Substituting $\kappa=\lambda/\bar s$ and $\varepsilon=\eta$ yields
\begin{align}
T(r_0)
=
\frac{\bar s}{\lambda}\log\frac{r_0}{r_G}
+
\eta
\frac{\bar s^2(p-2)}{4\lambda^2}
\left[
\frac{1}{r_G^2}
-
\frac{1}{r_0^2}
\right]
+
O\!\left(\frac{\eta^2}{\lambda^3}\right).
\label{eq:mfpt_expansion_continuous_appendix}
\end{align}

\paragraph{Asymptotic Expansions.}
The transition starts at the memorization radius and ends at the generalization radius.
By Theorem~\ref{thm:memorization_radius_scaling_restated},
\begin{align}
\rho_M^2
=
(\rho_M^{(0)})^2
+
\frac{\eta}{b}c_M
+
O\bigl((\eta/b)^2\bigr).
\end{align}
Taking the square root gives
\begin{align}
r_0
=
\rho_M
=
\rho_M^{(0)}
+
\frac{\eta}{b}\frac{c_M}{2\rho_M^{(0)}}
+
O\bigl((\eta/b)^2\bigr).
\label{eq:rho_M_asymptotic_appendix}
\end{align}
Similarly, by Theorem~\ref{thm:generalization_radius_scaling_restated},
\begin{align}
\rho_G^2
=
(\rho_G^{(0)})^2
+
\frac{\eta}{\lambda}c_G
+
O\!\left(\frac{\eta^2}{\lambda^2}\right),
\end{align}
where
\begin{align}
(\rho_G^{(0)})^2
=
\mathbb E_{S_0\in\Theta_S}
\left[
\|\theta^\star(S_0)\|_2^2
\right],
\end{align}
and
\begin{align}
\theta^\star(S_0)
=
\arg\min_{\theta(S_0)}
\left\{
\mathscr L^{*}_f(\theta(S_0))
+
\frac{\lambda}{2}\|\theta(S_0)\|_2^2
\right\}.
\end{align}
Therefore,
\begin{align}
r_G
=
\rho_G
=
\rho_G^{(0)}
+
\frac{\eta}{\lambda}
\frac{c_G}{2\rho_G^{(0)}}
+
O\!\left(\frac{\eta^2}{\lambda^2}\right).
\label{eq:rho_G_asymptotic_appendix}
\end{align}

Using
\begin{align}
\log(x+\Delta x)
=
\log x
+
\frac{\Delta x}{x}
+
O(\Delta x^2),
\label{eq:log_small_perturbation_appendix}
\end{align}
we have
\begin{align}
\log r_0
&=
\log \rho_M^{(0)}
+
\frac{\eta}{b}
\frac{c_M}{2(\rho_M^{(0)})^2}
+
O\bigl((\eta/b)^2\bigr),
\label{eq:log_r0_expansion_appendix}
\\
\log r_G
&=
\log \rho_G^{(0)}
+
\frac{\eta}{\lambda}
\frac{c_G}{2(\rho_G^{(0)})^2}
+
O\!\left(\frac{\eta^2}{\lambda^2}\right).
\label{eq:log_rG_expansion_appendix}
\end{align}
Hence
\begin{align}
\log\frac{r_0}{r_G}
&=
\log\frac{\rho_M^{(0)}}{\rho_G^{(0)}}
+
\frac{\eta}{b}
\frac{c_M}{2(\rho_M^{(0)})^2}
-
\frac{\eta}{\lambda}
\frac{c_G}{2(\rho_G^{(0)})^2}
+
O\!\left(\frac{\eta^2}{\lambda^2}\right).
\label{eq:log_asymptotic_ratio_expansion_appendix}
\end{align}
In the explicit $O(\eta)$ It\^o correction term of
equation~\eqref{eq:mfpt_expansion_continuous_appendix}, it is sufficient to use
the leading asymptotics:
\begin{align}
\frac{1}{r_G^2}-\frac{1}{r_0^2}
=
\frac{1}{(\rho_G^{(0)})^2}
-
\frac{1}{(\rho_M^{(0)})^2}
+
O(\eta/b)+O(\eta/\lambda).
\label{eq:inverse_radius_expansion_appendix}
\end{align}

Substituting equations~\eqref{eq:log_asymptotic_ratio_expansion_appendix}
and~\eqref{eq:inverse_radius_expansion_appendix} into
equation~\eqref{eq:mfpt_expansion_continuous_appendix} gives
\begin{align}
T_{M\to G}
&=
\frac{\bar s}{\lambda}
\log\frac{\rho_M^{(0)}}{\rho_G^{(0)}}
\nonumber\\
&\quad
+
\frac{\eta}{b\lambda}
\bar s
\frac{c_M}{2(\rho_M^{(0)})^2}
\nonumber\\
&\quad
+
\frac{\eta}{\lambda^2}
\left\{
-\bar s
\frac{c_G}{2(\rho_G^{(0)})^2}
+
\frac{\bar s^2(p-2)}{4}
\left[
\frac{1}{(\rho_G^{(0)})^2}
-
\frac{1}{(\rho_M^{(0)})^2}
\right]
\right\}
\nonumber\\
&\quad
+
O\!\left(\frac{\eta^2}{\lambda^3}\right).
\label{eq:mfpt_continuous_final_appendix}
\end{align}
Define
\begin{align}
c_\tau
&:=
\bar s\,\frac{c_M}{2(\rho_M^{(0)})^2},
\nonumber\\
c_\tau^{(2)}
&:=
-\bar s\,\frac{c_G}{2(\rho_G^{(0)})^2}
+
\frac{\bar s^2(p-2)}{4}
\left[
\frac{1}{(\rho_G^{(0)})^2}
-
\frac{1}{(\rho_M^{(0)})^2}
\right].
\end{align}
Then
\begin{align}
T_{M\to G}
\approx
\frac{\bar s}{\lambda}
\log\frac{\rho_M^{(0)}}{\rho_G^{(0)}}
+
\frac{\eta}{b\lambda}c_\tau
+
\frac{\eta}{\lambda^2}c_\tau^{(2)}
+
O\!\left(\frac{\eta^2}{\lambda^3}\right).
\label{eq:mfpt_continuous_compact_appendix}
\end{align}

\paragraph{Computing Number of Iterations.}
We compute the number of iterations, since the continuous-time interpolation is $t=\eta k$:
\begin{align}
\mathbb{E}_{S_0\in\Theta_S}[\tau_{M\to G}]
=
\frac{T_{M\to G}}{\eta}.
\label{eq:continuous_to_iteration_conversion_appendix}
\end{align}
Dividing equation~\eqref{eq:mfpt_continuous_compact_appendix} by $\eta$ gives
\begin{align}
\mathbb{E}_{S_0\in\Theta_S}[\tau_{M\to G}]
\approx
\frac{\bar s}{\eta\,\lambda}
\log\frac{\rho_M^{(0)}}{\rho_G^{(0)}}
+
\frac{c_\tau}{b\,\lambda}
+
\frac{c_\tau^{(2)}}{\lambda^2}
+
O\!\left(\frac{\eta}{\lambda^3}\right),
\end{align}
which proves equation~\eqref{eq:law_transition_time_restated}.

\paragraph{$\ell_2$ Regularization Coefficient Dependence.}
In the late-stage radial SDE, the diffusion scale is $\varepsilon=\eta$ and is independent
of batch size; the leading first-passage time is obtained by setting $\varepsilon=0$ and is
determined only by the deterministic radial drift $-\lambda r/\bar s$. Batch size enters the
leading term through the scalar preconditioning scale $\bar s = O(1/\sqrt{b})$, and the
corrections through the memorization asymptotic $\rho_M$ --- whose $O(\eta/b)$
correction becomes $O(1/b)$ after converting to iteration time --- and through
$c_\tau = O(1/\sqrt{b})$ and $c_\tau^{(2)} = O(1/b)$.
\end{proof}